\documentclass{article}
\usepackage{arxiv}

\usepackage{mathtools}
\usepackage[utf8]{inputenc} % allow utf-8 input
\usepackage[T1]{fontenc}    % use 8-bit T1 fonts
\usepackage{hyperref}       % hyperlinks
\usepackage{url}            % simple URL typesetting
\usepackage{booktabs}       % professional-quality tables
\usepackage{amsfonts}       % blackboard math symbols
\usepackage{nicefrac}       % compact symbols for 1/2, etc.
\usepackage{microtype}      % microtypography
\usepackage{xcolor}         % colors

\usepackage{bm}
\usepackage{bbm}
\usepackage{comment}         
\usepackage{multicol}
\usepackage{multirow}
\usepackage{caption}
\usepackage{braket}
\usepackage{enumitem}
\usepackage{soul}
\newcommand{\AT}[1]{\textcolor{blue}{#1}}

\mathchardef\mhyphen="2D
\DeclareMathOperator*{\argsup}{argsup}

\usepackage{amsmath,amssymb,amsthm}

\theoremstyle{plain}

\newtheorem{theorem}{Theorem}[section]
\newtheorem{lemma}[theorem]{Lemma}

\newtheorem{proposition}[theorem]{Proposition}

\newtheorem{definition}[theorem]{Definition}
\newtheorem{assumption}[theorem]{Assumption}

\newtheorem{remark}{Remark}

\usepackage{color}

\usepackage{graphicx}

\DeclareMathOperator*{\argmax}{arg\,max}
\DeclareMathOperator*{\argmin}{arg\,min}

\usepackage{times}

\renewcommand{\P}{\mathbb{P}}

\newcommand{\pa}{\mathrm{\pa}}

\newcommand{\RN}[1]{%
  \textup{\uppercase\expandafter{\romannumeral#1}}%
}

\usepackage{xcolor}
\newcount\Comments  % 0 suppresses notes to selves in text
\newcommand{\kibitz}[2]{\ifnum\Comments=1\textcolor{#1}{#2}\fi}

\usepackage{algorithm}
\usepackage{algorithmic}

\usepackage{authblk}
\allowdisplaybreaks

\title{Optimal Best Arm Identification with Fixed Confidence in Restless Bandits}

\renewcommand{\shorttitle}{\small Optimal Best Restless Markov Arm Identification with Fixed Confidence}

\author[1]{P. N. Karthik}
\author[2]{Vincent Y. F. Tan}
\author[3]{Arpan Mukherjee}
\author[3]{Ali Tajer}

\affil[1]{Indian Institute of Technology Hyderabad} \affil[2]{National University of Singapore}
\affil[3]{Rensselaer Polytechnic Institute

Emails: \href{mailto:pnkarthik@ai.iith.ac.in}{pnkarthik@ai.iith.ac.in}, 
\href{mailto:vtan@nus.edu.sg}{vtan@nus.edu.sg},
\href{mailto:mukhea5@rpi.edu}{mukhea5@rpi.edu}, \href{mailto:tajer@ecse.rpi.edu}{tajer@ecse.rpi.edu}}

\begin{document}

\maketitle

\begin{abstract}
We study best arm identification in a {\em restless} multi-armed bandit setting with finitely many arms. The discrete-time data generated by each arm forms a homogeneous Markov chain taking values in a common, finite state space. The state transitions in each arm are captured by an {\em ergodic} transition probability matrix (TPM) that is a member of a single-parameter exponential family of TPMs. The real-valued parameters of the arm TPMs are {\em unknown} and belong to a given space. Given a function~$f$ defined on the common state space of the arms, the goal is to identify the best arm---the arm with the largest average value of $f$ evaluated under the arm's stationary distribution---with the fewest number of samples, subject to an upper bound on the decision's error probability (i.e., the {\em fixed-confidence} regime). 
A lower bound on the growth rate of the expected stopping time is established in the asymptote of a vanishing error probability. Furthermore, a policy for best arm identification is proposed, and its expected stopping time is proved to have an asymptotic growth rate that matches the lower bound. It is demonstrated that tracking the long-term behavior of a certain Markov decision process and its state-action visitation proportions are the key ingredients in analyzing the converse and achievability bounds. It is shown that under every policy, the state-action visitation proportions satisfy a specific approximate flow conservation constraint and that these proportions match the optimal proportions dictated by the lower bound under any asymptotically optimal policy. 
The prior studies on best arm identification in restless bandits focus on {\em independent observations} from the arms, {\em rested} Markov arms, and restless Markov arms with {\em known} arm TPMs. In contrast, this work is the first to study best arm identification in restless bandits with unknown arm TPMs.

\end{abstract}

\iffalse
\newpage 

\tableofcontents

\newpage
\fi 

\section{\sc Introduction}
Multi-armed bandits constitute an effective probabilistic model for sequential decision-making under uncertainty. 
%\AT{I'm not sure if we need the following -- very well known to bandit experts: Introduced first by Thompson \cite{thompson1933likelihood} in the context of clinical trials, multi-armed bandits have now found much wider reach in a range of fields. These include communication systems, power systems, cognitive neuroscience, ad recommendation systems, etc., where they serve as a mathematical model for allocating a scarce resource to a pool of two or more options (metaphorically called {\em arms}) \cite{lattimore2020bandit}.} 
In the canonical multi-armed bandit models, each arm is assumed to yield random rewards generated by an unknown reward distribution. The arms are selected sequentially over time to optimize a pre-specified reward measure. The two common frameworks to formalize bandit algorithms are {\em regret minimization} and {\em pure exploration}. In regret minimization, the objective is to have an arm selection policy that minimizes the difference between the expected reward realized and the maximum reward achievable by an oracle that knows the true reward distributions. Minimizing such regret measures captures the inherent {\em exploration-exploitation} trade-off that specifies the balance between the desire to choose the arms with high expected rewards (exploitation) against the need to explore other arms to acquire better information discrimination (exploration). In this context, there exists a wide range of algorithms for different settings based on the notions of Upper Confidence Bound (UCB)  \cite{lai1985asymptotically, lai1987adaptive}
%, its variant UCB1 \cite{auer2002finite},
and Thompson Sampling \cite{thompson1933likelihood}
%are some of the popular and well-understood algorithms for regret minimisation. 
An in-depth analysis of these algorithms and a detailed survey of other studies on regret minimization can be found in \cite{lattimore2020bandit}.

The pure exploration framework, on the other hand, focuses on identifying one or a group of arms with specified properties using the fewest samples. Pure exploration disregards the reward regret incurred and is closely related to the literature on sequential hypothesis testing~\cite{chernoff1959sequential,albert1961sequential}. In pure exploration, algorithm design involves forming optimal data-adaptive sampling decisions and characterizing optimal stopping times. 

%A second popular theme in works on multi-armed bandits is that of {\em pure exploration} wherein, unlike the inherent exploration-exploitation dilemma in regret minimisation, the goal is to explore the arms with the eventual aim of validating the truth of one or more hypotheses without committing to any arm prematurely. Pure exploration bandit problems fall within the framework of {\em active} sequential hypothesis testing (ASHT) of Chernoff \cite{chernoff1959sequential} and Albert \cite{albert1961sequential}, and are instances of optimal stopping problems in decision theory. 
A practical instance of a pure exploration problem is {\em best arm identification} (BAI), which entails finding the best arm---the arm with the largest mean reward---as quickly and accurately as possible. Broadly, BAI is studied in two complementary regimes: the {\em fixed-budget} regime, in which the objective is to use a pre-specified number of arm sampling rounds to identify the best arm with minimal error probability, and the {\em fixed-confidence} regime, where the goal is to minimize the number of arm sampling rounds required to find the best arm with a pre-specified decision accuracy level. In this paper, we focus on BAI in a multi-armed bandit with restless Markov arms and focus on the fixed-confidence regime. In the rest of this section, we specify the problem framework and the technical contributions.

\subsection{\sc Problem Description}
We consider a {\em restless} multi-armed bandit setting with finitely many arms. In restless bandits, each arm has a finite number of states that evolve over time according to a homogeneous Markov chain taking values in a common, finite state space. We assume that the transition probability matrix (TPM) governing the state transitions in each arm belongs to a single-parameter exponential family of ergodic TPMs. Hence, a restless bandit setting with $K$ arms can be specified by $K$ TPMs. The real-valued parameters of the TPMs are unknown and belong to a given parameter space. The vector of TPM parameters specifies the problem instance. The TPM of each arm, being ergodic, is associated with a unique stationary distribution.
%\AT{perhaps a comment about stationary distributions? it will be referred to in the next paragraph.}

%We consider a multi-armed bandit with finitely many arms, in which each arm is a discrete-time, homogeneous Markov chain taking values in a common, finite state space. We assume that the transition probability matrix (TPM) governing the state transitions on each arm belongs to a single-parameter exponential family of ergodic TPMs, and is parameterized by a real-valued parameter belonging to a known parameter space. The vector of arm parameters specifies the problem instance.
%Every TPM in this family is parameterised by a real-valued parameter belonging to a known parameter space, is ergodic, and possesses a unique stationary distribution.
In this setting, we adopt a non-constant \emph{reward} function defined on the common state space of the arms. Accordingly, the {\em best arm} is defined as the arm with the largest average reward computed under the arm's stationary distribution. The learner is unaware of the underlying arm parameters and is faced with the task of identifying the best arm.
%Given a non-constant {\em reward} function defined on the common state space of the arms, a learner, who is unaware of the underlying arm parameters, is faced with the task of identifying the best arm---the arm with the largest average reward computed under the arm's stationary distribution. 
%\AT{stationary distribution undefined. Should we do it in the previous paragraph?} 
The learner selects the arms sequentially and one at a time \footnote{For simplicity in presentation, we assume that the learner selects only one arm at each time instant. The results of this paper can be easily extended to the case when the learner samples a subset of arms at each time instant.} Upon selecting an arm, the learner observes the current state of the arm. At the same time, the unobserved arms are \emph{restless} and {\em continue} undergoing state transitions.
%(hence, {\em restless} arms). 
Given a pre-specified confidence level $\delta \in (0,1)$, the learner's goal is to minimize the expected number of arm selections required to find the best arm while ensuring that the terminal error probability does not exceed $\delta$.
%\AT{I think we shouldn't say the following. Our interest and algorithm are for general $\delta$, but we provide asymptotic guarantees:} Our interest is in the asymptotics as $\delta \downarrow 0$.

%Given a non-constant `reward' function defined on the common state space of the arms, a learner who does not have knowledge of the underlying parameters of the arms is faced with the task of determining the best arm---the arm with the largest average reward computed under the arm's stationary distribution. The learner selects the arms sequentially, one at each time instant.\footnote{For simplicity in presentation, we assume that the learner selects only one arm at each time instant. The results of this paper may be easily extended to the case when the learner samples a subset of arms at each time instant.} Upon selecting an arm, the learner observes the current state of the arm, while the unobserved arms {\em continue} to undergo state transitions ({\em restless} arms). Given a confidence level $\delta \in (0,1)$, the learner's goal is to minimise the expected number of arm selections required to find the best arm, while ensuring that the terminal error probability is no more than $\delta$. Our interest is in the asymptotics as $\delta \downarrow 0$.

\subsection{\sc Key Analytical Challenges}
\label{subsec:key-analytical-challenges}
The continuous evolution of the unobserved arms necessitates that the learner, at each time instance,  maintains a record of (a) each arm's {\em delay}, which is defined as the time elapsed since an arm was last selected, and (b) each arm's {\em last observed state}, which is the state of each arm as observed at the last instance that it was selected. Keeping track of each arm's delay and the last observed state provides the learner with a historical perspective on how each arm performed or behaved during its previous selection. This information serves as a reference point for understanding an arm's characteristics or potential changes, helping the learner assess the arm's current state relative to its past behavior.
%\AT{a sentence about why tracking these two is necessary?} 
The existing studies on restless bandits establish that the arm delays and the last observed states collectively form a {\em controlled Markov chain}, with the arm selections serving as the controls and thereby influencing the overall behavior of the system (e.g., \cite[Section 5]{ortner2012regret}). In other words, we are in the setting of a {\em Markov decision process} (MDP) in which the state space is the space of all arm delays and last observed states, and the action space is the set of arms. We remark that such an MDP has potentially a {\em countably infinite} state space induced by the arm delays that may progressively increase with time. We write $\mathcal{M}$ as a shorthand representation for the above MDP.

%The continued evolution of the unobserved arms necessitates the learner to maintain at each time instant a record of (a) the time elapsed since each arm was previously selected (the arm's {\em delay}), and (b) the state of each arm as observed at its previous selection instant (the arm's {\em last observed state}). A preliminary examination of the prior works on restless arms (e.g., \cite[Section 5]{ortner2012regret}) reveals that the arm delays and the last observed states together form a {\em controlled Markov chain}, with the arm selections serving as the controls and thereby influencing the overall behaviour of the system. In other words, we are in the setting of a {\em Markov decision process} (MDP) in which the state space is the space of all arm delays and last observed states, and the action space is the set of arms. Importantly, we notice that the state space of this MDP (say $\mathcal{M}$) is potentially {\em countably infinite}, as the arm delays may progressively increase with time.  

\textbf{MDP ergodicity.\ }
Previous studies on restless arms, such as \cite{ortner2012regret, wang2020restless,karthik2021detecting}, have emphasized the importance of considering the ergodicity properties of the MDP $\mathcal{M}$ in their analysis. These studies typically establish some form of convergence of empirical functionals (e.g., reward, cost, and state-action visitations) to their respective true values, relying on the ergodicity/communication properties of $\mathcal{M}$. The task of proving such convergence is exacerbated when dealing with countable state MDPs (such as $\mathcal{M}$). Prior studies on countable-state MDPs reveal that guaranteeing the desired ergodicity properties relies on various regularity conditions.
%to facilitate analysis, various regularity conditions need to be imposed to guarantee the desired ergodicity properties. 
For example, \cite{borkar1982identification}~and~\cite{borkar1988control} assume that the countable-state MDPs therein are ergodic under {\em every} stationary control policy. This condition is met in \cite{karthik2021detecting,karthik2021learning} under a so-called ``trembling hand'' model. However, imposing similar conditions in our work has significant implications. It restricts the learner's choice of allowable policies to only those that make $\mathcal{M}$ ergodic; as such, $\mathcal{M}$ is merely communicating (see Lemma~\ref{lem:MDP-is-communicating}, a weaker property than ergodicity \cite[Section 8.3.1]{puterman2014markov}.
%\AT{this paragraph needs some edits -- we need to be more specific which statements are general for MDPs, which ones apply to only restless settings, and which ones apply to the specific restless settings of [8,9], in case they don't apply to us.} The findings in \cite{karthik2021detecting, karthik2021learning} underscore that analyzing MDP $\mathcal{M}$ critically hinges on evaluating its extent of ergodicity and communication properties. \AT{is the previous statement a universal MDP property? or only for restless? or only specific restless?} Specifically, these studies adopt a ``trembling hand'' model for randomized arm selection. Under this model, each policy primarily chooses the intended arm with high probability $1-\eta$, for a small $\eta>0$, and otherwise, with probability~$\eta$, it selects arms uniformly at random. Under this model, \cite{karthik2021detecting, karthik2021learning} demonstrate that the MDP $\mathcal{M}$ is {\em ergodic} under every {\em stationary} policy (\cite[Lemma 1]{karthik2021detecting}), a property that is pivotal to derive their results. Although ergodicity is a desirable property, the model mentioned above is overly restrictive for our setting, as it severely constrains the set of all learner policies. \AT{which aspect is severely limiting? ergodicity?} In fact, in the absence of the trembling hand model, we demonstrate that the MDP $\mathcal{M}$ is only {\em communicating} (see Lemma~\ref{lem:MDP-is-communicating}), a weaker property than ergodicity \cite[Section 8.3.1]{puterman2014markov}. 
One central challenge in this paper is devising a policy under which the MDP $\mathcal{M}$ has ``near-ergodicity'' properties and yet is amenable to analysis.

%The findings in \cite{karthik2021detecting, karthik2021learning} underscore the crucial role played by the ergodicity and communication properties of the MDP $\mathcal{M}$ in obtaining the relevant results. Notably, in \cite{karthik2021detecting, karthik2021learning}, the authors adopt a ``trembling hand'' model for arm selection, where each policy primarily chooses the intended arm with high probability, say $1-\eta$ for some small $\eta>0$, but with probability~$\eta$, it selects arms uniformly at random. Under this model, they demonstrate that the MDP $\mathcal{M}$ is {\em ergodic} under every {\em stationary} policy (cf. \cite[Lemma 1]{karthik2021detecting}), a property that is pivotal to deriving their results. Although ergodicity is a desirable property, the aforementioned model is overly restrictive for our work, as it severely constrains the set of all policies of the learner. In fact, in the absence of the trembling hand model, we demonstrate that the MDP $\mathcal{M}$ is only {\em communicating} (see Lemma~\ref{lem:MDP-is-communicating}), a weaker property than ergodicity \cite[Section 8.3.1]{puterman2014markov}. One of the central challenges we address in this paper is that of devising a policy under which the MDP $\mathcal{M}$ has ``near-ergodicity'' properties and is still amenable to analysis.

\textbf{Tracking the proportions of state-action visitations.\ }
Prior studies on BAI in the fixed-confidence regime have established problem-dependent lower bounds on the expected time required to find the best arm~\cite{garivier2016optimal, Kaufmann2016}. Characterizing these bounds involves solving sup-inf optimization problems, where the outer supremum is with respect to all probability distributions on the arms, while the inner infimum  accounts for alternative problem instances with varying best arm locations. The key to achieving such lower bounds is tracking the proportions of arm selections with time and ensuring that these proportions match the unique optimal (``sup''-attaining) proportion in the long run. These are the principles in the design of, for instance, ``C-tracking'' and ``D-tracking'' algorithms in \cite{garivier2016optimal}. In contrast to these known results, when dealing with restless Markov arms, the lower bounds are characterized not by the proportions of arm selections but rather by the proportions of {\em state-action visitations} of the MDP $\mathcal{M}$. Achieving such lower bounds necessitates ensuring that the proportion of visits to each state-action pair in the long term aligns with the optimal proportion specified by the lower bound. In particular, {\em merely matching the long-term proportions of action visitations (arm selections) with the optimal arm selection proportions may not lead to achieving the lower bound}. The primary challenge here is that while the learner can directly control the arm selections and the associated visitations, the learner lacks control over the state evolution of the MDP and, thereby, the state visitation proportions. Consequently, devising a policy that inherently guarantees the correct visitation proportion for each state-action pair is pivotal to achieving the lower bound. In this paper, we provide a comprehensive solution to this complex challenge.

\subsection{\sc Our Contributions}
We highlight the key contributions of the paper and how we address the challenges outlined in the previous section.
\begin{enumerate}[leftmargin=*]
    \item {\bf Maximum-delay constraint.} As mentioned earlier, the customary ergodicity assumptions of prior works, critical for analytical tractability, do not apply directly to our specific setting. As a solution to render the countable-state MDP $\mathcal{M}$ amenable to analysis, we constrain the {\em maximum delay} of each arm to be equal to a fixed and large positive integer denoted by $R$. This reduces the MDP's countably infinite state space to a finite state space. Despite this reduction, we show that the communication properties of the finite-state MDP with max-delay equal to $R$ (denoted $\mathcal{M}_R$) and the unconstrained MDP $\mathcal{M}$ are identical (Lemma~\ref{lem:MDP-R-is-communicating}), thereby not compromising our results significantly. We note that while it is computationally prohibitive to realize the countable-state MDP $\mathcal{M}$ on a machine with finite memory, the finite-state MDP $\mathcal{M}_R$ can indeed be realized on a machine with finite memory.
    %\AT{add a statement that this equivalence shows that the reduction is not (significantly) compromising the results?} 
    %It is worth noting here that the trembling-hand model of \cite{karthik2021detecting,karthik2021learning} alters the communication properties of the countable-state MDP, making it ergodic under every stationary policy. 
    
    \item {\bf Instance-dependent lower bound.} Given a problem instance specified by a vector of arm parameters, we establish a problem-dependent lower bound on the limiting growth rate of the expected number of arm selections (or simply the expected stopping time) required to find the best arm (Proposition~\ref{prop:lower-bound}). This growth rate is captured by the solution to a sup-inf optimization problem. In this problem, the outer supremum is over the {\em polytope} of all state-action distributions satisfying the flow constraint and the maximum delay constraint, and the inner infimum is over all alternative problem instances with the best arm distinct from the best arm in the true problem instance. Furthermore, the set over which the supremum is evaluated {\em depends} on the true problem instance. This is in contrast to the existing literature on BAI \cite{garivier2016optimal,Kaufmann2016,moulos2019optimal}. Consequently, it is unclear if this supremum is attained by a unique element in the set. Notably, the uniqueness of the sup-attaining solution in \cite{garivier2016optimal, Kaufmann2016, moulos2019optimal} significantly simplifies the subsequent analysis. 
    %in these studies.

    \item {\bf Sup-inf optimization.} 
    %By leveraging the fact that the arms are restless, 
    We show that when $R$ is the maximum delay of each arm, the objective function appearing in the sup-inf optimization of the lower bound contains Kullback--Leibler divergence terms that are functions of {\em powers} of TPMs up to order~$R$. The presence of second- and higher-order TPM powers further hinders simplifying the inner infimum, unlike in \cite{garivier2016optimal,Kaufmann2016,moulos2019optimal} where the inner infimum may be simplified further and cast as a minimum over finitely many non-best arms. Notwithstanding this, we employ a version of Berge's maximum theorem for non-compact sets~\cite[Theorem 1.2]{feinberg2014berges} to show that the inner infimum expression is a {\em continuous} function in its arguments despite the non-compactness of the set of alternative problem instances. We use this result to show that the potential {\em set} of sup-attaining solutions is convex, compact, and upper-hemicontinuous in the arm parameters. 
    %It is this upper-hemicontinuity property that we leverage to demonstrate the asymptotic optimality of a BAI policy that we propose in the paper. 

    \item {\bf Policy design.} We design a policy that selects the arms according to a certain time-dependent probability distribution on the arms, {\em conditional} on the current state of the MDP $\mathcal{M}_R$, while respecting the maximum delay constraint. This is in contrast to the explicit selection of arms under the C-tracking and D-tracking algorithms in \cite{garivier2016optimal}. 
    %Our policy respects the maximum delay constraint on the arms, and whenever the delay of any given arm reaches the prescribed maximum $R$, it is sampled w.p.1 in the following time instant. When the delay of each arm is below $R$, our policy selects each arm with a strictly positive probability. 
    We show that under this policy, the MDP $\mathcal{M}_R$ is ``near-ergodic'' in the following sense: if the probability distribution for selecting the arms at any given time $n$ were to be frozen and used to select the arms for all subsequent times $t \geq n$, then the MDP $\mathcal{M}_R$ becomes ergodic, admits a unique stationary distribution (on the space of state-action pairs), and consequently 
    %the state-action visitation proportions converge to this stationary distribution. However, the latter stationary distribution may not necessarily be an optimal (sup-attaining) state-action distribution governing the lower bound. Yet, using the near-ergodicity property, we show that under our policy, 
    every state-action pair is visited infinitely often (Lemma~\ref{lem:sufficient-exploration-of-state-actions}). 

    \item {\bf Convergence of state-action visitations and asymptotic optimality.} 
    %\AT{i'm not happy with this title, but I put it as a placeholder -- I think having the titles gives the reader an itemized list of challenges we're addressing.}
    %While most prior works on BAI employ the generalized likelihood ratio (GLLR) test statistic with simple closed-form expressions that are easy to evaluate, and track the value of this test statistic to decide when to stop further selection of arms, the GLLR does not admit a closed-form expression in our work because of the presence of powers of arm TPMs. Notwithstanding this, 
    We compute the empirical {\em state-action-state} transition probabilities and use this to design a test statistic that mimics the form of the inner infimum term in the lower bound expression (see~\eqref{eq:Z-of-n}). We employ this test statistic in conjunction with a random, time-dependent {\em threshold} that is a function of state-action visitations, and stop further selection of arms whenever the test statistic exceeds the threshold. We show that this leads to stopping in finite time almost surely and declaring the best arm correctly with the desired accuracy (Proposition~\ref{prop:stop-in-finite-time-and-error-prob-less-than-delta}). Furthermore, we show that the limiting growth rate of the expected stopping time satisfies an upper bound that matches the lower bound (Proposition~\ref{prop:upper-bound-on-expected-stopping-time}). 
    Our proof of the upper bound relies on showing the convergence of the empirical state-action visitation proportions to the {\em set} of sup-attaining state-action probability distributions governing the lower bound (Lemma~\ref{lem:concentration-of-state-action-visitations}).
    %, a result that we prove using a version of ergodic theorem for non-homogeneous Markov chains from \cite{al2021navigating}.   
\end{enumerate}

\subsection{\sc Overview of Prior Studies}
\textbf{Prior works on BAI. } BAI falls within the active sequential hypothesis testing framework of Chernoff~\cite{chernoff1959sequential} and Albert~\cite{albert1961sequential}, and has since been studied in a plethora of contexts. \cite{even2002pac} studies fixed-confidence BAI and provides a successive elimination algorithm for finding the best arm, proving an upper bound on its stopping time that only holds with high probability. For a similar setting as in \cite{even2002pac}, \cite{garivier2016optimal} presents (a) a sup-inf lower bound on the limiting growth rate of the expected stopping time using change-of-measure arguments, and (b) two algorithms for tracking the proportions of arm selections (C-tracking and D-tracking), along with upper bounds on stopping times that hold almost surely and in expectation for both algorithms. While the optimal solution to the lower bound in \cite{garivier2016optimal} was shown to be unique, \cite{degenne2019pure} investigates the case when the optimal solution is potentially non-unique and/or the set of all optimal solutions is non-convex. The paper \cite{jedra2020optimal} investigates fixed-confidence BAI in {\em linear} bandits with finitely/uncountably many arms and provides nearly-matching lower and upper bounds on the limiting growth rate of the expected stopping time. While the algorithms in the aforementioned studies explicitly compute the sup-attaining solution(s) at every time step for an empirical problem instance arising from empirical arm means, the recent study \cite{mukherjee2023best} proposes a computationally efficient policy that circumvents the computation of the sup-attaining solution(s). 
%The paper \cite{reddy2023almost} studies fixed-confidence BAI in a federated learning environment comprising a single server and multiple clients, with uplink (client to server) communication costs, and proposes an elimination-based algorithm for identifying the (locally) best arm at each client and the globally best arm (the arm with the largest average reward, averaged across the clients). This algorithm is shown to have near-zero communication costs. The recent paper \cite{agrawal2023optimal} investigates the performance of fixed-confidence BAI when the learner has prior access to offline data. 

In another direction, \cite{audibert2010best} investigates fixed-budget BAI, proposes a {\em successive-rejects} algorithm, and obtains an error probability upper bound for the same. While problem-dependent lower bounds are commonplace in the studies on fixed-confidence BAI, deriving such bounds for the fixed-budget regime is often challenging. Instead, the studies on fixed-budget BAI characterize {\em minimax lower bounds} on the error probability; such a lower bound dictates that there exists a problem under which every policy incurs an error probability with the minimum value given by the lower bound. In this space, the paper \cite{carpentier2016tight} obtains a minimax lower bound on the error probability of fixed-budget BAI, along with an upper bound that is order-wise tight in the exponent of the  error probability. Yang and Tan~\cite{yang2022minimax} investigate fixed-budget BAI in linear bandits and propose an algorithm based on the idea of G-optimal designs. They prove  a minimax lower bound on the error probability, similar to \cite{carpentier2016tight}, and obtain an upper bound on the error probability of their algorithm. 
%The recent paper \cite{barrier2023best} investigates fixed-budget BAI in non-parametric bandits. 

\textbf{Prior works on pure exploration in Markov bandits.} While Markov bandits have been extensively explored in the context of regret minimization \cite{Gittins1979,Whittle1988,agrawal1989asymptotically,ortner2012regret,wang2020restless,liu2012learning}, they have not been explored as well in the context of pure exploration. \cite{karthik2020learning} studies fixed-confidence odd arm identification in rested Markov bandits (where the unobserved arms do not exhibit state transitions, and the goal is to find the anomalous or odd arm). The studies in~\cite{karthik2021detecting,karthik2021learning} extend the results of \cite{karthik2020learning} to the setting of restless arms, using a trembling-hand model for arms selection inspired by cognitive neuroscience. \cite{moulos2019optimal} investigates BAI in rested Markov bandits under a parametric model for arm TPMs and {\em hidden} Markov observations from the arms, proposes a sup-inf lower bound on the limiting growth rate of the expected stopping time, and proposes a D-tracking rule similar to \cite{garivier2016optimal}. The setting of rested arms can be viewed as a special case of the setting of restless arms in which the arm delays are always equal to~$1$. Hence, the Kullback--Leibler divergence terms appearing in the lower bound of \cite{moulos2019optimal} are not functions of the second and higher order powers of TPMs. As a result, the inner infimum expression of the lower bound therein may be simplified further and cast as a minimum over finitely many non-best arms, as in \cite{garivier2016optimal}. This simplification may be exploited further to demonstrate the uniqueness of the optimal solution to the lower bound, thus greatly simplifying the achievability analysis. However, the presence of higher-order powers of TPMs in the Kullback--Leibler divergence terms in our setting do not permit further simplification of the inner infimum expression, thereby forcing us to work with a {\em set} of optimal solutions to the lower bound and its associated analytical challenges (e.g., upper-hemicontinuity instead of continuity). \cite{karthik2022best} investigates BAI in restless bandits when the arm TPMs are known up to a permutation. 
%To the best of our knowledge, 
Our work studies BAI in restless bandits with unknown TPMs. 

\textbf{Related works on MDPs.} The paper \cite{al2021adaptive} studies the problem of identifying the best policy in MDPs--- the one that maximizes the expected sum of discounted rewards over an infinite time horizon---in the fixed-confidence regime, when the learner has access to the next state {\em and} action at every time instant (generative model). In a follow-up work \cite{al2021navigating}, the results in \cite{al2021adaptive} are  extended to the case when the learner can access only the next action but not the next state (as in our work). Both studies~\cite{al2021adaptive,al2021navigating} present lower and upper bounds on the limiting growth rate of the time to identify the best policy. They propose a relaxation to their lower bounds by leveraging the structure of the MDP reward function, leading to a discrepancy of factor $2$ between the upper and lower bounds. However, a similar relaxation of the lower bound as in~\cite{al2021adaptive,al2021navigating} is not possible in our work. This is because the notion of MDP rewards is void in our work since the central problem we address is BAI and not reward maximization (or regret minimization), which is typical of MDPs. For an in-depth review of MDPs, see~\cite{puterman2014markov}. For more related works on MDPs, see~\cite{al2021navigating} and the references therein.

\subsection{\sc Paper Organisation}
%The rest of this paper is organized as follows. 
In Section~\ref{sec:notations}, we introduce the single-parameter exponential family of TPMs and the central objective of our paper. In Section~\ref{sec:delays-last-observed-states}, we introduce the countable-state MDP of arm delays and last observed states, outline its flow conservation property, and describe a reduction of its countably infinite state space to a finite state space via a constraint on the maximum delay of each arm. In Section~\ref{sec:lower-bound}, we present a lower bound on the asymptotic growth rate of the expected stopping time, the first main result of the paper. In Section~\ref{sec:achievability}, we present our policy for BAI. In Section~\ref{sec:results}, we present results on the performance of our policy, {\color{black} and demonstrate 
%in particular 
that our policy achieves the lower bound asymptotically as error probability vanishes.} {\color{black} In Section~\ref{sec:computational-efficient-variant}, we present preliminary ideas on making our proposed policy computationally feasible, albeit at the expense of asymptotic optimality.} In Section~\ref{sec:concluding-remarks-and-discussion}, we include a short discussion, provide concluding remarks, and outline future directions. The detailed proofs of all the results stated in the paper are presented in the appendices.

\section{\sc Preliminaries}
\label{sec:notations}

Let $\mathbb{N}\coloneqq \{1, 2, \ldots\}$ denote the set of positive integers. All vectors are column vectors unless stated otherwise. 
\iffalse
\AT{notation suggestions:
\begin{itemize}
    \item for consistency, we use $[]$ for vectors (now we sometimes have [] and sometimes ().
    \item for consistency, we use boldface for vectors (now it's mixed -- we have bold, regular, and underlined).
    \item We are using $P_\theta$ for TPM and measures. To avoid, we can use $\P_\theta$ for measures.
    \item I think it's better if we do not make the terminal rule and stopping time functions of (or indexed by) arm selection policy $\pi$. That'll create complications in analysis (and I think we are decoupling them anyways) 
    \item 
\end{itemize}
}
\fi

\subsection{\sc Restless Bandit Model}
We consider a {\em restless} multi-armed bandit setting with $K\geq 2$ arms in which each arm has a finite number of states that temporally evolve according to a discrete-time homogeneous Markov process taking values in a common, finite state space $\mathcal{S}=\{1, \ldots, |\mathcal{S}|\}$. 
%\AT{it seems later on $\cal S$ is used equivalent to $\{1,\dots, |\cal S|\}$. Should we specify it?} 
To formalize the transitions of the Markov processes of different arms on $\mathcal{S}$, we define a parameterized family of transition probability matrices (TPMs) as ${\cal P}(\Theta)\coloneqq \{P_\theta \;:\; \theta\in\Theta\}$, where $\Theta \subset \mathbb{R}$ is a fixed and known parameter space, and for each $\theta\in\Theta$, $P_\theta$ is a valid TPM. We denote the parameter of arm $a\in[K]\coloneqq\{1,\dots,K\}$ by $\theta_a$, and assume that the evolution of states on arm $a$ is governed by the TPM $P_{\theta_a}$. Accordingly, we define $\boldsymbol{\theta}\coloneqq[\theta_1, \ldots, \theta_K]^\top \in \Theta^K$ and refer to $\boldsymbol{\theta}$ as a {\em problem instance}. $\mathbb{P}_{\boldsymbol{\theta}}$ and $\mathbb{E}_{\boldsymbol{\theta}}$, respectively, denote the probability measure and the associated expectation induced by instance $\boldsymbol{\theta}$. We assume that each arm's temporal evolution is independent of the rest. 

\subsection{\sc Single-Parameter Exponential Family of TPMs}
We assume that the TPMs are generated according to a single-parameter exponential family studied in \cite{moulos2019optimal}. The model studied here is a generalization of the single-parameter exponential family model for independent observations from the arms studied in \cite{prabhu2020sequential}.
Fix an irreducible\footnote{This means that the whole state space $\mathcal{S}$ constitutes a single communicating class.} TPM $P$ on $\mathcal{S}$. We call $P$ the {\em generator} of the family. Let $f: \mathcal{S} \to \mathbb{R}$ be a known
%, non-constant\footnote{A non-constant function takes at least two distinct values.} 
function. 
%\AT{we don't need continuity, right?}. 
Given $P$ and $f$, define  $\Tilde{P}_\theta$ for any $\theta \in \Theta$  such that
\begin{equation}
    \Tilde{P}_\theta(j|i) = P(j|i)\ \exp({\theta\cdot  f(j)})\ , \qquad \forall i, j\in \mathcal{S}\ ,
    \label{eq:P-tilde-theta}
\end{equation}
where $P(j|i)$ and $\Tilde{P}_\theta(j|i)$ denote the $(i,j)$-th entry of $P$ and $\tilde{P}_\theta$, respectively.
The rows of $\tilde{P}_\theta$ do not necessarily sum up to 1. Hence,  $\tilde{P}_\theta$ is not necessarily a valid TPM. Nevertheless, we can normalize~\eqref{eq:P-tilde-theta} suitably to obtain a valid TPM in the following manner. For each $\theta \in \Theta$, let $\rho(\theta)$ be the Perron--Frobenius eigenvalue of $\Tilde{P}_\theta$. From the Perron-Frobenius theorem \cite[Theorem 8.8.4]{horn2013matrix}, we know that there exist unique left and right eigenvectors associated with the eigenvalue $\rho(\theta)$, say $\mathbf{u}_\theta=[\mathbf{u}_\theta(i):i \in \mathcal{S}]^\top$ and $\mathbf{v}_\theta=[\mathbf{v}_\theta(i): i \in \mathcal{S}]^\top$, respectively, such that $\mathbf{u}_\theta(i)>0, \mathbf{v}_\theta(i)>0$ for all $i\in \mathcal{S}$, and $\sum_{i\in \mathcal{S}} \mathbf{u}_\theta(i) \,  \mathbf{v}_\theta(i)=1$. Subsequently, the single-parameter exponential family with generator  $P$ is defined as ${\cal P}_\theta=\{P_\theta\;:\; \theta\in\Theta\}$, where for each $\theta\in\Theta$, $P_\theta$ 
is specified by
\begin{equation}
    P_\theta(j|i) = \frac{\mathbf{v}_\theta(j)}{\rho(\theta)\, \mathbf{v}_\theta(i)} \ \Tilde{P}_\theta(j|i)\ , \qquad  \ i, j\in \mathcal{S}\ .
    \label{eq:P-theta}
\end{equation} 
It can be readily verified that \eqref{eq:P-theta} specifies a valid TPM since
\begin{align}
    \sum_{j\in \mathcal{S}} P_\theta(j|i)= \frac{1}{\rho(\theta)\, \mathbf{v}_\theta(i)}\sum_{j\in \mathcal{S}} \mathbf{v}_\theta(j) \Tilde{P}_\theta(j|i) =1\ , \qquad \forall i \in \mathcal{S}, ~\forall\theta\in \Theta\  .
\end{align}
Furthermore, for each $\theta \in \Theta$, the matrix $P_\theta$ is irreducible and positive recurrent. Hence, $P_\theta$ has a unique stationary distribution, which we denote by $\mu_\theta=[\mu_\theta(i): i \in \mathcal{S}]^\top$. Note that $P_0=P$.
%\AM{define $\rho: \Theta\mapsto\R$?}
%It is easy to see that \eqref{eq:P-theta} defines a TPM, as $\sum_{j\in \mathcal{S}} \mathbf{v}_\theta(j) \, \Tilde{P}_\theta(j|i) = \rho(\theta)\, \mathbf{v}_\theta(i)$, and hence $\sum_{j\in \mathcal{S}} P_\theta(j\;|\; i)=1$ for all $i\in \mathcal{S}$. Furthermore, for each $\theta \in \Theta$, the matrix $P_\theta$ is irreducible, positive recurrent, and hence associated with a unique stationary distribution, say $\mu_\theta=[\mu_\theta(i): i \in \mathcal{S}]^\top$, satisfying $\mu_\theta^\top P_\theta = \mu_\theta^\top$. Notice that $P_0=P$.

Next, similar to~\cite{moulos2019optimal}, we impose mild assumptions on $P$. For this purpose, define $M_f=\max_{i\in \mathcal{S}} f(i)$ and $m_f=\min_{i\in \mathcal{S}} f(i)$. Accordingly, define the sets
\begin{align}
    \mathcal{S}_{M_f} = \{i\in \mathcal{S}: f(i)=M_f\} \ , \qquad \mbox{and} \qquad \mathcal{S}_{m_f} = \{i \in \mathcal{S}: f(i)=m_f\}\ .
\end{align}
%\AT{an alternative way of formalizing the assumptions is the following. But the original one is also perfectly fine. You can go with either.}
\begin{assumption}
    We assume that $P$ satisfies the following properties. 
    \begin{itemize}
        \item {\rm A}$_1$: The submatrix of $P$ with rows and columns in $S_{M_f}$ is irreducible.
        \item {\rm A}$_2$:  For every $i\in \mathcal{S} \setminus \mathcal{S}_{M_f}$, there exists $j\in \mathcal{S}_{M_f}$ such that $P(j|i)>0$.
        \item {\rm A}$_3$: The submatrix of $P$  with rows and columns in $S_{m_f}$ is irreducible. 
        \item {\rm A}$_4$: For every $i\in \mathcal{S} \setminus \mathcal{S}_{m_f}$, there exists $j\in \mathcal{S}_{m_f}$ such that $P(j|i)>0$. 
    \end{itemize}
\end{assumption}
\iffalse
\begin{align}
    & \text{The submatrix of }P\text{ with rows and columns in }S_{M_f}\text{ is irreducible}. \label{eq:property-1}\\
    & \text{For every }i\in \mathcal{S} \setminus \mathcal{S}_{M_f}, \text{ there exists }j\in \mathcal{S}_{M_f}\text{ such that }P(j|i)>0. \label{eq:property-2}\\
    & \text{The submatrix of }P\text{ with rows and columns in }S_{m_f}\text{ is irreducible}. \label{eq:property-3}\\
    & \text{For every }i\in \mathcal{S} \setminus \mathcal{S}_{m_f}, \text{ there exists }j\in \mathcal{S}_{m_f}\text{ such that }P(j|i)>0 \label{eq:property-4}. 
\end{align}
\fi
These assumptions, collectively, are mild and cover a wide range of models. For instance, when $P$ has strictly positive entries, it satisfies all of the above assumptions. In Remark~\ref{rem:need-for-parametric-model}, later in the paper, we elaborate on the crucial role of the above parametric model in our study.

%\AT{moved this part from the next subsection:} 
For any integer $d\geq 1$ and TPM $Q\in{\cal P}(\Theta)$, let $Q^{d}$ denote the matrix obtained by multiplying $Q$ with itself $d$ times. Also, for any $i,j\in\mathcal{S}$ and $d\geq 1$, let $Q^{d}(j|i)$ denote the $(i,j)$-th entry of $Q^{d}$, and let $Q^d(\cdot|i)$ denote the $i$-th row of $Q^d$.

\subsection{\sc Best Arm Identification}
%Given $\boldsymbol{\theta}=[\theta_1, \ldots, \theta_K]^\top \in \Theta^K$, we assume that the evolution of states on arm $a$ is governed by the TPM $P_{\theta_a}$, $a \in [K]$. In the sequel, we refer to $\boldsymbol{\theta}$ as a {\em problem instance}. 
%For each $a\in [K]$, because $P_{\theta_a}$ is ergodic, there exists a unique stationary distribution, say $\mu_{\theta_a}=(\mu_{\theta_a}(i):i\in \mathcal{S})$, satisfying $\mu_{\theta_a}^\top P_{\theta_a}= \mu_{\theta_a}^\top$.
%Corresponding to the reward function $f: \mathcal{S} \to \mathbb{R}$ and 

{\color{black} Throughout the paper, we assume that the generator $P$ and function $f$ are known beforehand.} Given $\boldsymbol{\theta}=[\theta_a: a \in [K]]^{\top}$ and the exponential family generated by $(P, f)$ via \eqref{eq:P-tilde-theta}, we define the {\em mean} of arm~$a$ as
\begin{equation}
\eta_{\theta_a} \coloneqq \sum_{i \in \mathcal{S}}\ f(i)\, \mu_{\theta_a}(i)\ , \qquad a \in [K]\ .
\label{eq:nu_a} 
\end{equation}
%\AT{in the above, $\mathcal{A}$ is not defined earlier. Should we replace it with $[K]$?.}
%as the {\em mean} of arm $a$ under the instance $\boldsymbol{\theta}$; here, $f$ is the same function appearing in \eqref{eq:P-tilde-theta}. Notice that \eqref{eq:nu_a} specifies the average value of $f$ under the stationary distribution of arm $a$. 
{\color{black} Let $\boldsymbol{\eta} \coloneqq [\eta_{\theta_a}: a \in [K]]^\top$.}  We define the {\em best arm} $a^\star(\boldsymbol{\theta})$ under the instance $\boldsymbol{\theta}$ as the arm with the largest mean, i.e.,
\begin{equation}
a^\star(\boldsymbol{\theta}) \coloneqq \arg \max _{a\in [K]}\ \eta_{\theta_a} = \arg \max _{a\in [K]}\ \sum_{i \in \mathcal{S}}\ f(i)\, \mu_{\theta_a}(i)\ .
\label{eq:best_arm}
\end{equation} 
We assume that $a^\star(\boldsymbol{\theta})$ is unique for all $\boldsymbol{\theta}\in \Theta^K$. In fixed-confidence BAI, a learner who does not have any prior knowledge of the instance $\boldsymbol{\theta}$, wishes to identify $a^\star(\boldsymbol{\theta})$ with the fewest number of arm selections (on the average) such that the decision error probability is confined below a pre-specified confidence level (a more formal specification of the problem objective is deferred until Section~\ref{sec:objective}). To distinguish the best arm from the rest, we write $\textsc{Alt}(\boldsymbol{\theta})$ to denote the set of all problem instances {\em alternative} to $\boldsymbol{\theta}$, i.e., those instances under which the best arm differs from the one under $\boldsymbol{\theta}$. Hence,\begin{equation}
    \textsc{Alt}(\boldsymbol{\theta})
    \coloneqq \{\boldsymbol{\lambda} \in \Theta^K: \exists\ a\neq a^\star(\boldsymbol{\theta}) \text{ such that }\eta_{\lambda_a} > \eta_{\lambda_{a^\star(\boldsymbol{\theta})}}\}\ .
    \label{eq:alt-theta}
\end{equation}
%For any integer $d\geq 1$ and TPM $Q$, let $Q^{d}$ denote the matrix obtained by multiplying $Q$ with itself $d$ times. Also, for any $i,j\in\mathcal{S}$ and $d\geq 1$, let $Q^{d}(j\;|\;i)$ denote the $(i,j)$-th entry of $Q^{d}$, and $Q^d(\cdot\; |\; i)$ denote the $i$-th row of $Q^d$. 
The Perron--Frobenius eigenvalue of $\tilde{P}_\theta$, $\rho(\theta)$, is pivotal in analyzing the properties of $ \textsc{Alt}(\boldsymbol{\theta})$. To formalize the connection, define $A(\theta)\coloneqq\log \rho(\theta)$, $\theta \in \Theta$. An important property of the family in \eqref{eq:P-theta} is that $A$ is differentiable, and $\Dot{A}=\frac{\mathrm{d}A}{\mathrm{d}\theta}$ is a strictly increasing and bijective map,
%between $\Theta$ and $(m, M)$, 
as noted in the following lemma (see \cite{moulos2019optimal} for a proof).
\begin{lemma}\cite[Lemma~2]{moulos2019optimal}
\label{lem:important-properties-of-family}
Let $P$ be an irreducible TPM on the finite state space $\mathcal{S}$ satisfying Assumptions {\rm A}$_1$-{\rm A}$_4$. Let $f:\mathcal{S}\to \mathbb{R}$ be a non-constant function. Consider the single-parameter exponential family of TPMs defined in~\eqref{eq:P-theta}, with $\tilde{P}_\theta$ as defined in~\eqref{eq:P-tilde-theta}. Let $A(\theta)=\log \rho(\theta)$ denote the log Perron--Frobenius eigenvalue of $\tilde{P}_\theta$. Then, the following properties hold.
\begin{enumerate}
    \item $\theta \mapsto A(\theta)$ is analytic.
    \item $P_\theta$ is irreducible and positive recurrent, and hence admits a unique stationary distribution, say $\mu_\theta$. 
    
    \item $\Dot{A}(\theta)=\eta_\theta = \sum_{i\in \mathcal{S}} f(i) \, \mu_\theta(i)$.
    
    \item The mapping $\theta \mapsto \Dot{A}(\theta)$ is strictly increasing.
    
    \item Let $\mathcal{M} \coloneqq \{\eta_{\theta}: \theta\in \Theta\}$. Then, 
    %$\mathcal{M}=(m, M)$, and 
    the mapping $\theta \mapsto \Dot{A}(\theta)$ is a bijection between $\Theta$ and $\mathcal{M}$.
\end{enumerate}
\end{lemma}
The fact that $\Dot{A}: \Theta \to \mathcal{M}$ is a strictly increasing bijection implies that
\begin{align}
    \textsc{Alt}(\boldsymbol{\theta})
    &= \{\boldsymbol{\lambda} \in \Theta^K: \exists\ a\neq a^\star(\boldsymbol{\theta}) \text{ such that }\eta_{\lambda_a} > \eta_{\lambda_{a^\star(\boldsymbol{\theta})}}\} \nonumber\\
    & = \{\boldsymbol{\lambda} \in \Theta^K: \exists\ a\neq a^\star(\boldsymbol{\theta}) \text{ such that }\Dot{A}(\lambda_a) > \Dot{A}(\lambda_{a^\star(\boldsymbol{\theta})})\} \nonumber\\ 
    &= \{\boldsymbol{\lambda} \in \Theta^K: \exists\ a\neq a^\star(\boldsymbol{\theta}) \text{ such that }\lambda_a > \lambda_{a^\star(\boldsymbol{\theta})}\}. \label{eq:Alt-theta-2}
\end{align} 

{\color{black}
\begin{remark}
    \label{rem:importance-of-exponential-family}
    The one-to-one correspondence between $\theta \in \Theta$ and $\eta_{\theta} \in \mathcal{M}$, as stated in Point~5 of Lemma~\ref{lem:important-properties-of-family}, leads to a one-to-one correspondence between the TPM of an arm and its associated mean, which in turn implies a convenient re-parameterisation of the arm TPMs via the arm means. This is akin to the parameterisation of arm distributions via the means in the settings of the prior works \cite{garivier2016optimal,prabhu2020sequential}. The re-parameterisation of TPMs in turn implies that for any given arm, the estimation {\color{black} of} its TPM may be accomplished via the estimation of its mean. 
    As we shall see, such a toggling between the estimation of mean and the estimation of the TPM plays a crucial role in the design of an asymptotically optimal policy, as also evidenced by the prior works \cite{garivier2016optimal,moulos2019optimal,mukherjee2023best}. The adept reader may readily recognise that our definition of the best arm in \eqref{eq:best_arm} involves the same function $f$ that appears in the specification of the exponential family of TPMs in \eqref{eq:P-tilde-theta}. The fact that the functions are the same in~\eqref{eq:P-tilde-theta} and~\eqref{eq:best_arm} results in the one-to-one correspondence stated in  Point~5 of Lemma~\ref{lem:important-properties-of-family}. Incorporating distinct reward functions in \eqref{eq:P-tilde-theta} and \eqref{eq:best_arm}, although of notable interest, is beyond the scope of the current paper, as the critical one-to-one correspondence between $\theta$ and $\eta_\theta$ would fail to hold in general. 
    %As we shall see, this one-to-one correspondence plays a pivotal role in the analysis. The model outlined above, even when the same function $f$ is employed in both \eqref{eq:P-tilde-theta} and \eqref{eq:best_arm}, presents significant analytical challenges and is non-trivial to study. It serves as a stepping stone towards analysing a more sophisticated model with distinct functions appearing in \eqref{eq:P-tilde-theta} and~\eqref{eq:best_arm}.
\end{remark}
}
%For $a\in [K]$, let $\mathcal{C}_a \subset \mathcal{C}$ denote the collection of all permutations in which $P_a=P_1$. Clearly, the collection $\{\mathcal{C}_a: \ a\in [K]\}$ is a partition of $\mathcal{C}$. 
%{\color{red} Let $\mathcal{H}_a$ denote the hypothesis that arm $a$ is the best arm.}
%A learner who has no prior knowledge of $\boldsymbol{\theta}$ wishes to find the best arm as quickly as possible, subject to an upper bound on the error probability. 

%\AM{I have a general suggestion about the notations. Later, we use $\underline{d}(t)$, $\underline{i}(t)$, etc, for denoting the vector of arm delays and last observed states. Since we use parenthesis $t$ for the value of any random variable at time $t$, is it not better to use $A(t)$ to denote the action at time $t$ and $\bar X(t)$ to denote the state observed at time $t$? The sequences can have subscripts with respect to $t$, such as $A_t\coloneqq \{A(s): s\in[t]\}$.}

\subsection{\sc Best Arm Identification Policy}

To find the best arm, the learner selects the arms sequentially, one at each time $n \in \mathbb{N} \cup \{0\}$.
Let $A_n\in[K]$ be the arm selected at time $n$, and let $\bar{X}_{n}\in{\cal S}$ be the state of arm $A_n$ observed by the learner. We assume that the arms are {\em restless}, i.e., the unobserved arms continue to undergo state transitions even though they are not selected. Let $(A_{0:n},\bar{X}_{0:n}) \coloneqq (A_0,\bar{X}_0,\ldots,A_{n}\,\bar{X}_{n})$ denote the history of all the arm selections and states observed up to time~$n$, generating the filtration  
\begin{equation}
\mathcal{F}_{n}\coloneqq \sigma(A_{0:n}, \bar{X}_{0:n})\ , \quad n\geq 0\ .
\label{eq:filtration}
\end{equation}
A BAI policy can be specified by three decision rules: (i) arm selection rule $\pi_n$ that is $\mathcal{F}_{n-1}$-measurable and specifies the arm to be selected at time $n$; (ii) a stopping rule adapted to $\{{\cal F}_n\;:\; n\geq 0\}$ that specifies the (random) time $\tau$ at which to terminate the arm selection process; and (iii) a terminal decision rule that is $\mathcal{F}_{\tau}$-measurable and specifies a candidate best arm $a\in[K]$ at the stopping time. 
%We define the policy $\pi$ as a collection of functions $\{\pi_{n}:n\geq 0\}$. Subsequently, BAI policy can be specified by the tuple $(\pi, \tau, a)$. 
Writing $\pi=\{\pi_n: n \geq 0\}$, we denote a generic BAI policy by the tuple $(\pi, \tau, a)$. Finally, for a pre-specified error tolerance level $\delta\in (0,1)$,  we define 
\begin{equation}
	\Pi(\delta)\coloneqq \{(\pi,\tau,a): \ \P_{\boldsymbol{\theta}}(\tau < +\infty)=1\ , \ \P_{\boldsymbol{\theta}}(a \neq a^\star(\boldsymbol{\theta})) \leq \delta\ , \quad \forall\, \boldsymbol{\theta}\in \Theta^K\}\ ,
\label{eq:Pi(epsilon)_sec2}
\end{equation}
as the collection of all policies that (a) stop in finite time almost surely, and (b) have an error probability no greater than the prescribed tolerance $\delta$ under {\em every} instance $\boldsymbol{\theta} \in \Theta^K$. The canonical BAI definition entails identifying a policy in $\Pi(\delta)$ that has the smallest average stopping time. We will show that in the restless bandit setting of interest, there needs to be an additional constraint, leading to a collection of policies that form a subset of $\Pi(\delta)$. We will discuss the necessary details in Section~\ref {sec:delays-last-observed-states} and provide the exact BAI formulation in Section~\ref{sec:objective}.

{\color{black}
\begin{remark}
    In order to be precise, it is essential to express  $\mathbb{P}_{\boldsymbol{\theta}}$ and $\mathbb{E}_{\boldsymbol{\theta}}$ more explicitly as $\mathbb{P}_{\boldsymbol{\theta}}^{\pi}$ and $\mathbb{E}_{\boldsymbol{\theta}}^{\pi}$ respectively under policy $\pi$, as these are contingent on the specific policy $\pi$. Nevertheless, for the sake of brevity, we omit the subscript $\pi$, and urge the reader to bear the dependence on $\pi$ in mind.
\end{remark}
}
\section{\sc Delays, Last Observed States, and a Markov Decision Process}
\label{sec:delays-last-observed-states}

The continued evolution of the unobserved arms necessitates the learner to maintain, at each time instance, a record of (a) each arm's {\em delay}, which is defined as the time elapsed since an arm was last selected, and (b) each arm's {\em last observed state}, which is the state of each arm at the last instance that it was selected. Keeping track of each arm's delay and the last observed state provides the learner with a historical perspective on how each arm performed or behaved during its previous selection. This information serves as a reference point for understanding an arm's characteristics or potential changes, helping the learner assess the arm's current state relative to its past behavior. The notion of arm delays is a key distinguishing feature of the setting of restless arms and is superfluous when each arm yields independent and identically distributed (i.i.d.) observations or when the unobserved arms do not evolve ({\em rested} arms).

%\AM{We are mixing up $t$ and $n$ to denote the time index. For example, in Section 2 (at the beginning), we use $t$. But, from Section 2.4, we use $n$. We should unify these into one notation.}
%\AT{is this a universal observation that applies to all types of decisions on restless bandits, or it's a property for specific decisions? Either way, let us be more explicit, and if it is the latter, we should also mention how this property shown for those decisions also applies to our problem:} The Prior studies \cite{karthik2022best,karthik2021learning, karthik2021detecting}  on restless bandits show that at each time~$n$, the decision entity must keep a record of (a) the \emph{delay} of each arm, which is defined as the time elapsed since the last time the arm was selected, and (b) the \emph{last observed state} of the arm, which is defined as the state of at last instance the arm was selected. As discussed in these works, the notion of arm delays is superfluous when the arms are {\em rested} or when each arm yields independent and identically distributed (i.i.d.) observations. Hence, the notion of arm delays is a key distinguishing feature of restless arms. To formalize these, 
Without loss of generality, 
we assume that every policy initially uses the first $K$ time slots to sequentially select and collect samples from arms $1$ through $K$, with $A_0=1, A_1=2, \ldots, A_{K-1}=K$.
This ensures that each arm is observed at least once. For $n\geq K$, let $d_a(n)$ and $i_a(n)$, respectively, denote the delay and the last observed state of arm $a$ at time $n$. Let $\mathbf{d}(n) \coloneqq (d_1(n),\ldots,d_K(n))$ and $\mathbf{i}(n) \coloneqq (i_1(n),\ldots,i_K(n))$ denote the vectors of arm delays and the last observed states at time $n$. We set $\mathbf{d}(K)=(K,K-1,\ldots,1)$, noting that with reference to $n=K$, arm $1$ was last observed $K$ time instants earlier (i.e., at $n=0$), arm $2$ was last observed $K-1$ time instants earlier  (i.e., at $n=1$), and so on. The following rule specifies how $d_a(n)$ and $i_a(n)$ can be updated recursively. When arm $a' \in [K]$ is selected at time $n$, i.e., $A_{n}=a'$, we have
\begingroup \allowdisplaybreaks\begin{align}
	{d}_a(n+1)=\begin{cases}
		d_a(n)+1, & a\neq a'\ ,\\
		1, & a=a'\ ,
	\end{cases} \qquad \mbox{and} \qquad 
	i_a(n+1)=\begin{cases}
		i_a(n), & a\neq a'\ ,\\
		\bar{X}_n, & a=a'\ .
	\end{cases}
	\label{eq:specific_transition_pattern}
\end{align}\endgroup
Note that $d_a(n)\geq 1$ for all $n\geq K$, with $d_a(n)=1$ if and only if $A_{n-1}=a$. Also note that $(A_{0:n-1}, \bar{X}_{0:n-1}) \equiv (A_{0:n-1}, \{(\mathbf{d}(s), \mathbf{i}(s))\}_{s=K}^{n})$. 
%At any time $n$, based on $\{(\mathbf{d}(s), \mathbf{i}(s))\}_{s=K}^{n}$, the learner pulls arm $A_{n}$, observes its current state $\bar{X}_{n}$,\footnote{Note that specifying $\{(\mathbf{d}(s), \mathbf{i}(s))\}_{s=K}^{n}$, $\{A_s\}_{s=0}^{n-1}$ is equivalent to specifying $(A_{0:n-1}, \bar{X}_{0:n-1})$ for all $n\geq K$.} and determines $(\mathbf{d}(n+1), \mathbf{i}(n+1)) $. This process is repeated until the stopping time, at which point arm $a_\pi$ is declared as the candidate best arm. 
It is clear that the process $\{(\mathbf{d}(n), \mathbf{i}(n))\}_{n=K}^{\infty}$ takes values in a subset $\mathbb{S}$ of the {\em countably infinite} set $\mathbb{N}^K\times\mathcal{S}^K$. The subset~$\mathbb{S}$ is formed based on the constraint that at any time $n\geq K$, exactly one component of $\mathbf{d}(n)$ is equal to $1$, and all the other components are strictly greater than $1$. Given $\boldsymbol{\theta}\in \Theta^K$, we note that
%\AT{we're using $P_\theta$ for TPMs. For the probability terms let's use $\P_\theta$. }
\begin{align}
	\P_{\boldsymbol{\theta}}  \Big(\mathbf{d}(n+1),\mathbf{i}(n+1)  \mid \{(\mathbf{d}(s),\mathbf{i}(s))\}_{s=K}^{n}, A_{0:n-1}, A_n\Big)=\P_{\boldsymbol{\theta}}\Big(\mathbf{d}(n+1),\mathbf{i}(n+1)\mid (\mathbf{d}(n),\mathbf{i}(n)),A_n\Big)\ , ~ \forall n \geq K.
    \label{eq:controlled_markov_chain}
\end{align}
This indicates that the evolution of the process $\{(\mathbf{d}(n), \mathbf{i}(n))\}_{n=K}^{\infty}$ is \emph{controlled} by the sequence $\{A_n\}_{n=K}^{\infty}$ of arm selections. Alternatively, $\{(\mathbf{d}(n), \mathbf{i}(n))\}_{n=K}^{\infty}$ is a {\em controlled Markov chain}, with $\{A_n\}_{n=K}^{\infty}$ being the sequence of controls.\footnote{The phrase ``controlled Markov chain'' is borrowed from \cite{borkar1988control}.} In other words, we are in the setting of a {\em Markov decision process} (MDP) whose state space, action space, and the associated transition probabilities can be specified as follows:
\begin{itemize}
    \item {\em State space:} The state space of the MDP is $\mathbb{S}$, with $(\mathbf{d}(n), \mathbf{i}(n))$ being the state at time $n$.
    \item {\em Action space}: The action space of the MDP is $[K]$, with action $A_n$ at time $n$ being $\mathcal{F}_{n-1}$-measurable.
    \item {\em Transition probabilities:} The transition probabilities under the instance $\boldsymbol{\theta}$ are given by
\begin{align}
	\P_{\boldsymbol{\theta}}(\mathbf{d}(n+1) & =\mathbf{d}',\mathbf{i}(n+1)=\mathbf{i}'\mid \mathbf{d}(n)=\mathbf{d},\mathbf{i}(n)=\mathbf{i}, A_n=a)\nonumber\\
%	&\hspace{3cm}
 &=\begin{cases}
		P_{\theta_a}^{d_a}(i_a'\mid i_a),&\text{if }d_a'=1\text{ and }d'_{\tilde{a}}=d_{\tilde{a}}+1 \quad \forall \tilde{a}\neq a,\\
		&i_{\tilde{a}}'=i_{\tilde{a}}\quad \forall \tilde{a}\neq a,\\
		0,&\text{otherwise},
  \end{cases}\label{eq:MDP_transition_probabilities}
\end{align}
\end{itemize}
Note that the right-hand side of \eqref{eq:MDP_transition_probabilities} is independent of $n$ and, therefore, it is stationary. Subsequently, we define
\begin{align}
    Q_{\boldsymbol{\theta}}(\mathbf{d}', \mathbf{i}'\mid \mathbf{d}, \mathbf{i},a) \coloneqq  \P_{\boldsymbol{\theta}}(\mathbf{d}(n+1) & =\mathbf{d}',\mathbf{i}(n+1)=\mathbf{i}'\mid \mathbf{d}(n)=\mathbf{d},\mathbf{i}(n)=\mathbf{i}, A_n=a)\ .
\end{align}
%We write $Q_{\boldsymbol{\theta}}(\mathbf{d}', \mathbf{i}'|\mathbf{d}, \mathbf{i},a)$ to denote the probability in \eqref{eq:MDP_transition_probabilities}. 
Let $\mathcal{M}_{\boldsymbol{\theta}}$ denote the MDP with state space $\mathbb{S}$, action space $[K]$, and transition probabilities given by $Q_{\boldsymbol{\theta}}$.

\subsection{\sc Reduction from Countable State Space to Finite State Space}
The existing studies on countable-state MDPs (and more generally controlled Markov chains) impose additional regularity conditions on the transition probabilities of the MDP to facilitate tractable analysis. One commonly used regularity condition is that ``under every stationary policy for choosing the actions, the MDP is ergodic''; see, for instance, \cite[Section II, pp. 58]{borkar1988control} and \cite[Assumption A4]{borkar1982identification}. Imposing a similar regularity condition in our setting in order to make the MDP $\mathcal{M}_{\boldsymbol{\theta}}$ ergodic implies restricting the space of all possible policies of the learner significantly. As such, the MDP $\mathcal{M}_{\boldsymbol{\theta}}$ is only {\em communicating} (a property much weaker than than ergodicity \cite[Section 8.3.1]{puterman2014markov}) as demonstrated in the below result.
\begin{lemma}
    \label{lem:MDP-is-communicating}
    For every $\boldsymbol{\theta} \in \Theta^K$, the MDP $\mathcal{M}_{\boldsymbol{\theta}}$ is communicating, i.e., for all $(\mathbf{d}, \mathbf{i}), (\mathbf{d}', \mathbf{i}') \in \mathbb{S}$, there exists $N \geq 1$ (possibly depending on $(\mathbf{d},\mathbf{i})$ and $(\mathbf{d}',\mathbf{i}')$) and a policy $\pi$ such that under the policy $\pi$,
    \begin{equation}
        \P_{\boldsymbol{\theta}}(\mathbf{d}(n+N)=\mathbf{d}', \mathbf{i}(n+N)=\mathbf{i}' \mid \mathbf{d}(n)=\mathbf{d}, \mathbf{i}(n)=\mathbf{i})>0\ , \qquad \forall n \geq K\ .
        \label{eq:communicating-MDP-definition}
    \end{equation}
\end{lemma}
As an alternative to imposing the customary regularity conditions, to facilitate further analysis in our work, we reduce the countable state space $\mathbb{S}$ of the MDP to a finite state space by constraining the delay of each arm to be no more than a finite and positive integer, say $R$. Under this constraint, once the delay of an arm reaches $R$ at any given time, this arm is forcefully selected at the next time instant. We refer to this constraint on arm delays as the {\em $R$-max-delay constraint}.

Let $\mathbb{S}_R \subset \mathbb{S}$ denote the subset of all arm delays and last observed states in which the delay of each arm is at most $R$. Furthermore, let $\mathbb{S}_{R, a} \subset \mathbb{S}_R$ denote the subset of all arm delays and last observed states in which the delay of arm $a$ is equal to $R$. The modified transition probabilities for the MDP $\mathcal{M}_{\boldsymbol{\theta}}$ under the $R$-max-delay constraint are as follows:
\begin{itemize}
    \item \textit{Case 1:} $(\mathbf{d},\mathbf{i})\notin \bigcup_{a=1}^{K}\ \mathbb{S}_{R,a}$. In this case, the transition probabilities are as in \eqref{eq:MDP_transition_probabilities}.
%    \begin{align}
%	&P_C^\pi(\mathbf{d}(t+1)=\mathbf{d}',\mathbf{i}(t+1)=\mathbf{i}'\mid \mathbf{d}(t)=\mathbf{d},\mathbf{i}(t)=\mathbf{i}, A_t=a)\nonumber\\
%	&=\begin{cases}
%		(P_C^a)^{d_a}(i_a'|i_a),&\text{if }d_a'=1\text{ and }d'_{\tilde{a}}=d_{\tilde{a}}+1\text{ for all }\tilde{a}\neq a,\\
%		&i_{\tilde{a}}'=i_{\tilde{a}}\text{ for all }\tilde{a}\neq a,\\
%		0,&\text{otherwise},
%	\end{cases}\label{eq:modified_MDP_transition_probabilities1}
%\end{align}
\item \textit{Case 2:} $(\mathbf{d},\mathbf{i})\in \mathbb{S}_{R,a}$ for some $a\in \mathcal{A}$. In this case, when $A_{n}=a$,
    \begin{align}
	\P_{\boldsymbol{\theta}}(\mathbf{d}(n+1) & =\mathbf{d}',\mathbf{i}(n+1)=\mathbf{i}'\mid \mathbf{d}(n)=\mathbf{d},\mathbf{i}(n)=\mathbf{i}, A_n=a)\nonumber\\
	&=\begin{cases}
		{\color{black} P_{\theta_a}^{R}(i_a'\mid i_a)}, &\text{if }d_a'=1\text{ and }d'_{\tilde{a}}=d_{\tilde{a}}+1\text{ for all }\tilde{a}\neq a,\\
		&i_{\tilde{a}}'=i_{\tilde{a}}\text{ for all }\tilde{a}\neq a,\\
		0,&\text{otherwise},
	\end{cases}
\label{eq:modified_MDP_transition_probabilities2}
\end{align}
and when $A_{n} \neq a$, the transition probabilities are undefined. Noting that the right-hand side of \eqref{eq:modified_MDP_transition_probabilities2} is independent of $n$, we define 
\begin{align}
Q_{\boldsymbol{\theta},R}(\mathbf{d}', \mathbf{i}'\mid \mathbf{d}, \mathbf{i},a) \coloneqq \P_{\boldsymbol{\theta}}(\mathbf{d}(n+1) & =\mathbf{d}',\mathbf{i}(n+1)=\mathbf{i}'\mid \mathbf{d}(n)=\mathbf{d},\mathbf{i}(n)=\mathbf{i}, A_n=a)\ .
\end{align}

%to denote the transition probabilities in \eqref{eq:modified_MDP_transition_probabilities2}. 
\end{itemize}
Going forward, we write $\mathcal{M}_{\boldsymbol{\theta}, R}$ to denote the finite-state MDP with state space $\mathbb{S}_R$, action space $[K]$, and transition probabilities $Q_{\boldsymbol{\theta}, R}$. The following analogue of Lemma~\ref{lem:MDP-is-communicating} shows that despite the finite-state space reduction described above, the MDP $\mathcal{M}_{\boldsymbol{\theta}, R}$ is still communicating. A proof of this follows along the same lines as the proof of Lemma~\ref{lem:MDP-is-communicating} and is omitted for brevity.
\begin{lemma}
    \label{lem:MDP-R-is-communicating}
    Fix $R \geq K$. For every $\boldsymbol{\theta} \in \Theta^K$, the MDP $\mathcal{M}_{\boldsymbol{\theta}, R}$ is communicating.
\end{lemma}
%Lemma~\ref{lem:MDP-R-is-communicating} shows that our technique of state space reduction while making the problem amenable to analysis, retains the MDP communication properties.

\subsection{\sc MDP Transition Kernel}
It is convenient to view a policy as a (randomized) rule for mapping any given $(\mathbf{d}, \mathbf{i}) \in \mathbb{S}_R$ to an action $a\in{\mathcal A}$. 
Given a policy $\pi=\{\pi(a\mid \mathbf{d}, \mathbf{i}): (\mathbf{d}, \mathbf{i}, a) \in \mathbb{S}_R \times [K]\}$ and $\boldsymbol{\theta} \in \Theta^K$, where $\pi(a \mid \mathbf{d}, \mathbf{i})$ is the probability of choosing action~$a$ when the MDP $\mathcal{M}_{\boldsymbol{\theta}, R}$ is in state $(\mathbf{d}, \mathbf{i})$, we define  $Q_{\boldsymbol{\theta}, \pi}$ as the {\em transition kernel} of the MDP $\mathcal{M}_{\boldsymbol{\theta}, R}$ under $\pi$. Formally,
\begin{equation}
    Q_{\boldsymbol{\theta}, \pi}(\mathbf{d}', \mathbf{i}', a'\mid \mathbf{d}, \mathbf{i}, a) \coloneqq Q_{\boldsymbol{\theta}, R}(\mathbf{d}', \mathbf{i}' \mid \mathbf{d}, \mathbf{i}, a) \cdot \pi(a' \mid  \mathbf{d}', \mathbf{i}') \ ,\quad \forall (\mathbf{d}, \mathbf{i}, a), (\mathbf{d}', \mathbf{i}', a') \in \mathbb{S}_R \times [K].
    \label{eq:MDP-transition-kernel}
\end{equation}
For any $r \in \mathbb{N}$, we write $Q_{\boldsymbol{\theta}, \pi}^r$ to denote the $r$-fold self-product of $Q_{\boldsymbol{\theta}, \pi}$. Note that \eqref{eq:MDP-transition-kernel} represents the probability of transitioning from the state-action $(\mathbf{d}, \mathbf{i}, a)$ to the state-action $(\mathbf{d}', \mathbf{i}', a')$ in a single time step under $\pi$ and under the instance $\boldsymbol{\theta}$. Also, when there is no ambiguity, we write $Q_{\boldsymbol{\theta}, \pi}(\mathbf{d}', \mathbf{i}'\mid \mathbf{d}, \mathbf{i})$ to denote the probability of transitioning from the state $(\mathbf{d}, \mathbf{i})$ to the state $(\mathbf{d}', \mathbf{i}')$ in a single time step under $\pi$ and under the instance $\boldsymbol{\theta}$. We mask the dependence of $Q_{\boldsymbol{\theta}, \pi}$ on $R$ for notational clarity and ask the reader to bear this dependence in mind.

\subsection{\sc A Uniform Arm Selection Policy and Ergodicity of the Transition Kernel}
For later use, we record here a uniform arm selection policy that, while respecting the $R$-max-delay constraint, selects the arms uniformly at random at every time instant. We denote this policy by $\pi^{\text{unif}}$. Formally, for all $(\mathbf{d}, \mathbf{i}, a) \in \mathbb{S}_R \times [K]$, 
\begin{equation}
    \pi^{\text{unif}}(a\mid \mathbf{d}, \mathbf{i}) = 
    \begin{cases}
        \frac{1}{K}, & (\mathbf{d}, \mathbf{i}) {\color{black} \notin} \bigcup_{a'=1}^{K} \mathbb{S}_{R, a'}, \\
        1, & (\mathbf{d}, \mathbf{i}) \in \mathbb{S}_{R,a}, \\
        0, & (\mathbf{d}, \mathbf{i}) \in \bigcup_{a' \neq a} \mathbb{S}_{R, a'}.
    \end{cases}
    \label{eq:uniform-arm-selection-rule}
\end{equation}
Note that $\pi^{\text{unif}}$ is a stationary policy, i.e., the probabilities in \eqref{eq:uniform-arm-selection-rule} do not depend on time. The following lemma demonstrates that under $\pi^{\text{unif}}$, the MDP transition kernel is ergodic for every $\boldsymbol{\theta} \in \Theta^K$.
\begin{lemma}
    \label{lem:ergodicity-of-MDP-under-unif-policy}
    Fix $R \geq K$. The transition kernel $Q_{\boldsymbol{\theta}, \pi^{\text{\rm unif}}}$ is ergodic for all $\boldsymbol{\theta} \in \Theta^K$.
\end{lemma}
While the above ergodicity property naturally emerges within the framework of our paper, it is pragmatically {\em assumed} to hold in \cite{al2021navigating}; see, for instance, \cite[Assumption~2, p.9]{al2021navigating}. As we shall see, the ergodicity property of Lemma~\ref{lem:ergodicity-of-MDP-under-unif-policy} shall play an important role in the analysis of the BAI policy that we propose later in the paper.

\subsection{\sc State-Action Visitations and Flow Conservation}
\label{sec:state-action-visitations}

Given $n \geq K$ and $(\mathbf{d}, \mathbf{i}, a) \in \mathbb{S}_R \times [K]$, let
\begin{align}
N(n, \mathbf{d}, \mathbf{i}, a) \coloneqq \sum_{t=K}^{n} \mathbf{1}_{\{\mathbf{d}(t)=\mathbf{d}, \, \mathbf{i}(t)=\mathbf{i}, \, A_t=a\}}\ , \quad \mbox{and} \quad N(n, \mathbf{d}, \mathbf{i}) \coloneqq \sum_{a=1}^{K}\ N(n, \mathbf{d}, \mathbf{i}, a)\ ,
\label{eq:N(n,d,i,a)-and-N(n,d,i)}
\end{align}
denote, respectively, the number of times the state-action pair $(\mathbf{d}, \mathbf{i}, a)$ and state $(\mathbf{d}, \mathbf{i})$ are visited up to time~$n$. We refer to these as the {\em state-action visitations} and {\em state visitations} up to time $n$.
\iffalse
Given $\boldsymbol{\theta}, \boldsymbol{\lambda}\in \Theta^K$, let
\begin{align}
Z_{\boldsymbol{\theta}, \boldsymbol{\lambda}}(n) \coloneqq Z_{\boldsymbol{\theta}}(n) - Z_{\boldsymbol{\lambda}}^{\pi}(n)=\log \frac{P_{\boldsymbol{\theta}}(A_{0:n},\Bar{X}_{0:n})}{P_{\boldsymbol{\lambda}}(A_{0:n},\Bar{X}_{0:n})}
\label{eq:LLR}
\end{align}
denote the log-likelihood ratio (LLR) of the arm selections and observations seen under the policy $\pi$ and under the instance  $\boldsymbol{\theta}$, with respect to that under the instance $\boldsymbol{\lambda}$. Using \eqref{eq:log-likelihood-final} and the fact that~\eqref{eq:Z_{C}_2} does not depend on the actual instance, we have
\begin{align}
\label{eq:LLR-final}
Z_{\boldsymbol{\theta}, \boldsymbol{\lambda}}(n)=\sum_{a=1}^{K} \log \frac{P_{\boldsymbol{\theta}}(X_{a-1}^{a})}{P_{\boldsymbol{\lambda}}(X_{a-1}^{a})} + \sum_{(\mathbf{d}, \mathbf{i})\in \mathbb{S}}\ \sum_{a=1}^{K}\ \sum_{j\in \mathcal{S}}\ N(n, \mathbf{d}, \mathbf{i}, a, j) \ \log \frac{P_{\theta_a}^{d_{a}}(j|i_{a})}{P_{\lambda_a}^{d_{a}}(j|i_{a})}.
\end{align}
\fi 
The next result shows that the expected values of these visitations satisfy an approximate {\em flow-conservation} property.
\begin{lemma}[Flow conservation]
\label{lem:flow-conservation-property}
Fix $R \geq K$, $\boldsymbol{\theta} \in \Theta^K$, and $(\mathbf{d}', \mathbf{i}', a) \in \mathbb{S}_R \times [K]$. Under every policy $\pi$,
\begin{equation}
    \left\lvert \mathbb{E}_{\boldsymbol{\theta}}[N(n, \mathbf{d}', \mathbf{i}')] - \sum_{(\mathbf{d}, \mathbf{i}) \in \mathbb{S}_R} \ \sum_{a=1}^{K} \mathbb{E}_{\boldsymbol{\theta}}[N(n, \mathbf{d}, \mathbf{i}, a)] \, Q_{\boldsymbol{\theta}, R}(\mathbf{d}', \mathbf{i}'|\mathbf{d}, \mathbf{i}, a) \right\rvert \leq 1 \ ,\qquad \forall n \geq K\ .
    \label{eq:flow-conservation-property}
\end{equation}
\end{lemma}
In \eqref{eq:flow-conservation-property}, the first term on the left-hand side may be interpreted as the total {\em outward flow} from the state $(\mathbf{d}', \mathbf{i}')$ at time $n$, whereas the second term may be interpreted as the total {\em inward flow} into state $(\mathbf{d}', \mathbf{i}')$ at time $n$. Then, \eqref{eq:flow-conservation-property} dictates that the outward flow for $(\mathbf{d}', \mathbf{i}')$ almost matches its inward flow for all times and for all $(\mathbf{d}', \mathbf{i}') \in \mathbb{S}_R$. In this sense, \eqref{eq:flow-conservation-property} may be regarded as an approximate flow conservation property for the process $\{(\mathbf{d}(n), \mathbf{i}(n)): n \geq K\}$.

We note here that the $R$-max-delay constraint may be expressed in terms of the state-action visitations and the state visitations as follows. For all $(\mathbf{d}, \mathbf{i}, a) \in \mathbb{S}_R \times [K]$ and $n \geq K$,
\begin{equation}
    N(n, \mathbf{d}, \mathbf{i}, a)=
    \begin{cases}
        N(n, \mathbf{d}, \mathbf{i}), & (\mathbf{d}, \mathbf{i}) \in \mathbb{S}_{R, a}, \\
        0, & (\mathbf{d}, \mathbf{i}) \in \bigcup_{a' \neq a} \mathbb{S}_{R, a'}, \\
        \text{unaltered}, & (\mathbf{d}, \mathbf{i}) \notin \bigcup_{a'=1}^{K} \mathbb{S}_{R, a'}.
    \end{cases}
    \label{eq:R-max-delay-in-terms-of-state-action-visitations}
\end{equation}
In \eqref{eq:R-max-delay-in-terms-of-state-action-visitations}, the first line on the right-hand side depicts the scenario when $(\mathbf{d}, \mathbf{i}) \in \mathbb{S}_{R, a}$, i.e., $d_a=R$. In this case, because arm $a$ is forcefully selected following every occurrence of $(\mathbf{d}, \mathbf{i})$, it follows that $N(n, \mathbf{d}, \mathbf{i}, a) = N(n, \mathbf{d}, \mathbf{i})$ for all $n \geq K$. On the other hand, if $(\mathbf{d}, \mathbf{i}) \in \bigcup_{a' \neq a} \mathbb{S}_{R, a'}$, then there exists $a' \neq a$ such that $d_{a'}=R$, and therefore arm $a'$ is forcefully selected following every occurrence of $(\mathbf{d}, \mathbf{i})$, thereby implying that $N(n, \mathbf{d}, \mathbf{i}, a)=0$ for all $n \geq K$. The last line on the right-hand side of \eqref{eq:R-max-delay-in-terms-of-state-action-visitations} depicts the scenario when $d_a < R$ for all $a \in [K]$.

\subsection{\sc \texorpdfstring{$R$}{R}-max-constrained BAI}
\label{sec:objective}
\iffalse
Given an error probability threshold $\delta \in (0,1)$ and $R \geq K$, let 
\begin{equation}
	\Pi_R(\delta)\coloneqq \{\pi: \ P_{\boldsymbol{\theta}}(\tau_\pi < +\infty)=1, \ P_{\boldsymbol{\theta}}(a_\pi \neq a^\star(\boldsymbol{\theta})) \leq \delta \quad \forall\, \boldsymbol{\theta}\in \Theta^K\}
\label{eq:Pi(epsilon)}
\end{equation}
%\AM{$\delta$-correct policies should almost surely stop?}
denote the collection of all policies that stop in finite time almost surely, satisfy an error probability that is no greater than $\delta$ under {\em every} instance $\boldsymbol{\theta} \in \Theta^K$, and respect the $R$-max-delay constraint. Here, $P_{\boldsymbol{\theta}}(\cdot)$ denotes the probability measure under the instance $\boldsymbol{\theta}$ (the dependence of this probability measure on $\pi$ is suppressed for notational convenience). 
\fi
%If we adopt the changes in section 2.4, we can also replace the above as follows:}
Given an error probability threshold $\delta \in (0,1)$ and $R \geq K$, based on the definition of $\Pi(\delta)$ in~\eqref{eq:Pi(epsilon)_sec2} we define
\begin{equation}
	\Pi_R(\delta)\coloneqq \{(\pi,\tau,a)\in \Pi(\delta) \; :\; (\pi,\tau,a)\; \mbox{satisfies $R$-max-delay constraint}\}\ ,
\label{eq:Pi(epsilon)}
\end{equation}
%\AM{$\delta$-correct policies should almost surely stop?}
which is the collection of all policies that stop in finite time almost surely, satisfy an error probability that is no greater than $\delta$ under {\em every} instance $\boldsymbol{\theta} \in \Theta^K$, and respect the $R$-max-delay constraint.
%We emphasise that policies in $\Pi(\delta)$ work for all possible best arm locations. 
%Writing $\mathbb{E}_{\boldsymbol{\theta}}[\cdot]$ to denote expectations under the measure $P_{\boldsymbol{\theta}}$, 
We anticipate from similar results in the literature that $\inf_{\pi \in \Pi_R(\delta)}\mathbb{E}_{\boldsymbol{\theta}}[\tau_\pi] \sim \Omega(\log(1/\delta))$, where the asymptotics is as $\delta \downarrow 0$. Our objective in this paper is to precisely characterize the value of
\begin{equation}
    \lim_{\delta \downarrow 0} \ \inf_{\pi \in \Pi_R(\delta)} \ \frac{\mathbb{E}_{\boldsymbol{\theta}}[\tau_\pi]}{\log(1/\delta)}
    \label{eq:main_quantity}
\end{equation}
in terms of $\boldsymbol{\theta}$ and $R$. For the remainder of the paper, we fix $R \geq K$. 
\section{\sc Lower Bound}
\label{sec:lower-bound}

In this section, we present an instance-dependent lower bound for \eqref{eq:main_quantity}. Throughout the analysis, given two probability mass functions $p$ and $q$ with identical support, we define $D_{\text{KL}}(p \| q)$ as the Kullback--Leibler (KL) divergence between $p$ and $q$. Given $\boldsymbol{\theta} \in \Theta^K$, let $\Sigma_R(\boldsymbol{\theta})$ denote the space of all probability mass functions $\nu$ satisfying
\begin{align}
    &(\text{Flow conservation}) \quad \sum_{a=1}^{K} \nu(\mathbf{d}', \mathbf{i}', a) = \sum_{(\mathbf{d}, \mathbf{i}) \in \mathbb{S}_R} \ \sum_{a=1}^{K} \nu(\mathbf{d}, \mathbf{i}, a) \, Q_{\boldsymbol{\theta}, R}(\mathbf{d}', \mathbf{i}'\mid \mathbf{d}, \mathbf{i}, a), & \forall (\mathbf{d}', \mathbf{i}') \in \mathbb{S}_R, 
    \label{eq:flow-constraint} \\
    &(R\text{-max-delay constraint}) \quad \nu(\mathbf{d}, \mathbf{i}, a) = \sum_{a'=1}^{K} \nu(\mathbf{d}, \mathbf{i}, a'), \quad \forall (\mathbf{d}, \mathbf{i}) \in \mathbb{S}_{R, a}, \ a \in [K]\ .
    \label{eq:max-delay-constraint}
\end{align}
Let $Q_{\boldsymbol{\theta}, R}(\cdot \mid\mathbf{d}, \mathbf{i}, a) \coloneqq [Q_{\boldsymbol{\theta}, R}(\mathbf{d}', \mathbf{i}'\mid \mathbf{d}, \mathbf{i}, a): (\mathbf{d}', \mathbf{i}') \in \mathbb{S}_R]^\top$. The following proposition gives a lower bound on~\eqref{eq:main_quantity}.
%\AT{to make the following proposition a bit lighter and easier to read, should we specify $\Sigma_R(\boldsymbol{\theta})$ before the proposition and explain what it means before using it?}
\begin{proposition}
\label{prop:lower-bound}
For any $\boldsymbol{\theta}\in \Theta^K$,
\begin{equation}
    \liminf_{\delta \downarrow 0} \ \inf_{\pi \in \Pi_R(\delta)}\ \frac{\mathbb{E}_{\boldsymbol{\theta}}[\tau_\pi]}{\log(1/\delta)} \geq \frac{1}{T_R^\star(\boldsymbol{\theta})},
    \label{eq:lower-bound}
\end{equation}
where $T_R^\star(\boldsymbol{\theta})$ in \eqref{eq:lower-bound} is given by 
%\AT{in the following it's better to specify all the arguments and avoid replacing them by $\cdot$}
\begin{equation}
    T_R^\star(\boldsymbol{\theta}) = \sup_{\nu \in \Sigma_R(\boldsymbol{\theta})}\ \inf_{\boldsymbol{\lambda} \in \textsc{Alt}(\boldsymbol{\theta})}  \ \sum_{(\mathbf{d}, \mathbf{i}) \in \mathbb{S}_R}\  \sum_{a=1}^{K} \nu(\mathbf{d}, \mathbf{i}, a) \,  D_{\text{\rm KL}}(Q_{\boldsymbol{\theta}, R}(\cdot \mid\mathbf{d}, \mathbf{i}, a) \| Q_{\boldsymbol{\lambda}, R}(\cdot \mid\mathbf{d}, \mathbf{i}, a))\ .
    \label{eq:T-R-star-C}
\end{equation}
In \eqref{eq:T-R-star-C}, the KL divergence is computed on the vectorized forms of the distributions $Q_{\boldsymbol{\theta}, R}(\cdot \mid\mathbf{d}, \mathbf{i}, a)$ and $ Q_{\boldsymbol{\lambda}, R}(\cdot \mid\mathbf{d}, \mathbf{i}, a)$ viewed as conditional probability distributions on $\mathbb{S}_R$, conditioned on $(\mathbf{d}, \mathbf{i}, a)$.
\end{proposition}
Recalling \eqref{eq:modified_MDP_transition_probabilities2}, we note  that $Q_{\boldsymbol{\theta}, R}(\cdot \mid\mathbf{d}, \mathbf{i}, a)$ and  $Q_{\boldsymbol{\lambda}, R}(\cdot \mid\mathbf{d}, \mathbf{i}, a)$ are functions of $P_{\theta_a}^{d_a}$ and $ P_{\lambda_a}^{d_a}$, respectively, where $1 \leq d_a \leq R$. That is, the KL divergence in \eqref{eq:T-R-star-C} is a function of {\em powers} of arm TPMs of order up to $R$. Because of the presence of TPM powers, the inner infimum expression in \eqref{eq:T-R-star-C} cannot be simplified any further. This is in contrast to the inner infimum expressions appearing in prior works on BAI dealing with i.i.d. observations from the arms \cite{garivier2016optimal} (where the arm delays are inconsequential because of the i.i.d. nature of observations) or rested Markov arms \cite{moulos2019optimal} (where $d_a \equiv 1$ for all $a$). Furthermore, the supremum in \eqref{eq:T-R-star-C} is over the {\em instance-dependent} set $\Sigma_R(\boldsymbol{\theta})$, which is in contrast to the prior works on BAI \cite{garivier2016optimal,moulos2019optimal} in which the supremum is over the {\em instance-independent} simplex of arm distributions. The constant $T_R^\star(\boldsymbol{\theta})$ measures the ``hardness'' of problem instance $\boldsymbol{\theta}$ in the following sense: the closer the arm TPMs are to one another in the KL divergence sense, the smaller the value of $T_R^\star(\boldsymbol{\theta})$, and therefore the larger the stopping time.

{\color{black} We note here that $\Sigma_R(\boldsymbol{\theta})$ is non-empty, and hence the supremum in \eqref{eq:T-R-star-C} is well defined. Indeed, we have $\pi^{\rm {unif}} \in \Sigma_R(\boldsymbol{\theta})$. To see this, recall that $\pi^{\rm unif}$ respects the $R$-max-delay constraint. Furthermore, from Lemma~\ref{lem:ergodicity-of-MDP-under-unif-policy}, we know that the controlled Markov chain $\{(\mathbf{d}(n), \mathbf{i}(n))\}_{n=K}^{\infty}$ is, in fact, an ergodic Markov chain under the policy $\pi^{\text{\rm unif}}$. Let $\mu_{\boldsymbol{\theta}}^{\text{\rm unif}}=[\{\mu_{\boldsymbol{\theta}}^{\text{\rm unif}}(\mathbf{d}, \mathbf{i}): (\mathbf{d}, \mathbf{i}) \in \mathbb{S}_R]^\top$ denote the corresponding stationary distribution when the underlying instance is $\boldsymbol{\theta}$. Further, let 
\begin{equation}
    \nu_{\boldsymbol{\theta}}^{{\rm unif}}(\mathbf{d}, \mathbf{i}, a) \coloneqq \mu_{\boldsymbol{\theta}}^{{\rm unif}}(\mathbf{d}, \mathbf{i}) 
    \cdot \pi^{{\rm unif}}(a | \mathbf{d}, \mathbf{i}), \quad \forall (\mathbf{d}, \mathbf{i}, a) \in \mathbb{S}_R \times [K]\ .
    \label{eq:ergodic-state-action-occupancy-under-theta}
\end{equation}
We now claim that $\nu = \nu_{\boldsymbol{\theta}}^{\rm unif}$ satisfies \eqref{eq:flow-constraint}. Indeed, it is straightforward to see that for $\nu = \nu_{\boldsymbol{\theta}}^{\rm unif}$, \eqref{eq:flow-constraint} reduces to
$$
    \mu_{\boldsymbol{\theta}}^{\rm unif}(\mathbf{d}^\prime, \mathbf{i}^\prime) = \sum_{(\mathbf{d}, \mathbf{i}) \in \mathbb{S}_R} \mu_{\boldsymbol{\theta}}^{\rm unif}(\mathbf{d}, \mathbf{i}) \cdot Q_{\boldsymbol{\theta}, R}(\mathbf{d}^\prime, \mathbf{i}^\prime | \mathbf{d}, \mathbf{i}) \quad \forall (\mathbf{d}^\prime, \mathbf{i}^\prime) \in \mathbb{S}_R,
$$
the equation defining the stationary distribution $\mu_{\boldsymbol{\theta}}^{\rm unif}$. More generally, for any probability distribution on the arms, say $\boldsymbol{\zeta} = [\zeta_1, \ldots, \zeta_K]^\top$, such that $\zeta_a > 0$ for all $a \in [K]$, the stationary arm selection policy which selects arms according to $\boldsymbol{\zeta}$, while respecting the $R$-max delay constraint, is an element of $\Sigma_R(\boldsymbol{\theta})$. A formal proof of this follows along the same lines as that of Lemma~\ref{lem:ergodicity-of-MDP-under-unif-policy}.
}

The next result shows that the supremum in \eqref{eq:T-R-star-C} can be replaced by a maximum, i.e., the supremum in \eqref{eq:T-R-star-C} is attained for some $\nu\in \Sigma_R(\boldsymbol{\theta})$.
\begin{lemma}
\label{lem:attainment-of-supremum}
Let 
\begin{equation}
    \psi(\nu, \boldsymbol{\theta}) = \inf_{\boldsymbol{\lambda} \in \textsc{Alt}(\boldsymbol{\theta})}  \ \sum_{(\mathbf{d}, \mathbf{i}) \in \mathbb{S}_R}\  \sum_{a=1}^{K} \ \nu(\mathbf{d}, \mathbf{i}, a) \,  D_{\text{KL}}(Q_{\boldsymbol{\theta}, R}(\cdot \mid\mathbf{d}, \mathbf{i}, a) \| Q_{\boldsymbol{\lambda}, R}(\cdot \mid\mathbf{d}, \mathbf{i}, a)), \quad \nu \in \Sigma_R(\boldsymbol{\theta}), \ \boldsymbol{\theta} \in \Theta^K\ .
    \label{eq:psi}
\end{equation}
Then, $\psi$ is continuous under the topology induced by the sup-norm metric on $\Sigma_R(\boldsymbol{\theta}) \times \mathbb{R}^K$. Consequently, the supremum in \eqref{eq:T-R-star-C} may be replaced by a maximum. Furthermore, the mapping $\boldsymbol{\theta} \mapsto T_R^\star(\boldsymbol{\theta})$ is continuous, and the set-valued mapping $\boldsymbol{\theta} \mapsto \mathcal{W}^\star(\boldsymbol{\theta})$, with
\begin{equation}
   \mathcal{W}^\star(\boldsymbol{\theta}) \coloneqq \left\lbrace \nu \in \Sigma_R(\boldsymbol{\theta}): \psi(\nu, \boldsymbol{\theta}) = T_R^\star(\boldsymbol{\theta}) \right \rbrace,
   \label{eq:optimiser-set}
\end{equation}
is upper-hemicontinuous and compact-valued.
\end{lemma}
From \eqref{eq:Alt-theta-2}, it is evident that $\textsc{Alt}(\boldsymbol{\theta})$ is non-compact for each $\boldsymbol{\theta} \in \Theta^K$. To establish the continuity of $\psi$, we rely on a version of Berge's maximum theorem \cite[Theorem 1.2]{feinberg2014berges} for non-compact sets. Our proof of Lemma~\ref{lem:attainment-of-supremum} is an adaptation of the proof of \cite[Theorem 4]{degenne2019pure}, taking into account the dependence of $\Sigma_R(\boldsymbol{\theta})$ on the problem instance $\boldsymbol{\theta}$. In \cite{degenne2019pure}, the counterpart of $\Sigma_R(\boldsymbol{\theta})$ is the simplex of all probability distributions on the arms---an instance-independent set.
\begin{remark}
    \label{rem:lim-T-R-star}
    Although we keep $R$ fixed throughout the paper, we note here the following monotonicity property: $T_R^\star(\boldsymbol{\theta}) \leq T_{R+1}^\star(\boldsymbol{\theta})$ for all $R$. Indeed, writing $\psi$ and $\mathcal{W}^\star$ more explicitly as $\psi_R$ and $\mathcal{W}_R^\star$ to emphasize their dependence on $R$, it is straightforward to see that (a) the larger the value of $R$, the larger the cardinality of $\mathbb{S}_R$, and (b) for any $\nu \in \mathcal{W}_R^\star(\boldsymbol{\theta})$, defining $\tilde{\nu}$ via $\tilde{\nu}(\mathbf{d}, \mathbf{i}, a) = \nu(\mathbf{d}, \mathbf{i}, a) \, \mathbf{1}_{\{(\mathbf{d}, \mathbf{i}) \in \mathbb{S}_R\}}$ for all $(\mathbf{d}, \mathbf{i}, a) \in \mathbb{S}_{R+1} \times [K]$, we have $\tilde{\nu} \in \Sigma_{R+1}(\boldsymbol{\theta})$. Therefore, it follows that
    \begin{align}
        T_{R+1}^\star(\boldsymbol{\theta})
        \geq \psi_{R+1}(\tilde{\nu}, \boldsymbol{\theta}) = \psi_R(\nu, \boldsymbol{\theta}) = T_R^\star(\boldsymbol{\theta}), \qquad \forall R\in\mathbb{N}\ .
    \end{align}
    Hence, $\lim_{R \to \infty} T_R^\star(\boldsymbol{\theta})$ exists. See Section~\ref{sec:concluding-remarks-and-discussion} for further discussions.
\end{remark}

From \eqref{eq:lower-bound} and \eqref{eq:psi}, it is evident that to achieve the lower bound in \eqref{eq:lower-bound}, it is critical to control the values of the empirical state-action visitation proportions $\{N(n, \mathbf{d}, \mathbf{i}, a)/n: (\mathbf{d}, \mathbf{i}, a) \in \mathbb{S}_R \times [K]\}$, and ensure that these long-term fractions converge to the set $\mathcal{W}^\star(\boldsymbol{\theta})$ under the instance $\boldsymbol{\theta}$. In particular, merely ensuring that the empirical arm selection proportions converge to their respective optimal proportions given by the lower bound {\em does not} suffice for achievability. 
%In this regard, the following remark is in order.
{\color{black}
\begin{remark}
    \label{rem:need-for-parametric-model}
    The single-parameter exponential family of TPMs outlined in Section~\ref{sec:notations} serves a specific and critical purpose in our paper. Given unknown TPMs $\{P_k: k\in[K]\}$ with no structural constraints on their entries as in~\eqref{eq:P-theta}, suppose that %\ATC{in earlier definitions the argument of $T^*_R$ is $\btheta$:}
    $T_R^\star(P_1, \ldots, P_K)$ ({\color{black} the analogue of $T_R^\star(\boldsymbol{\theta})$ in the absence of the parametric model}) is the constant appearing in the corresponding lower bound expression. To achieve this lower bound, as outlined above, it is critical to ensure that the long-term state-action visitation proportions converge to $\mathcal{W}^\star(P_1, \ldots, P_K)$ (the analogue of $\mathcal{W}^\star(\boldsymbol{\theta})$ in \eqref{eq:optimiser-set} in the absence of the parametric model). However, because the TPMs $\{P_k: k\in[K]\}$ are not known beforehand, and they must be estimated along the way using arm observations characterized by {\em delays}. This is a fundamentally challenging task. It is noteworthy that the estimated matrices are not guaranteed to be ergodic. Furthermore, even after the TPM estimates are obtained, it is the estimates of the arm means that ultimately enable identifying the best arm. Consequently, a critical need arises for a continual alternation between estimating arm means and estimating the arm TPMs. This alternation is facilitated by the adoption of the parametric model in our study, by virtue of the one-to-one correspondence between the arm means and the arm TPMs (see item 5 under Lemma~\ref{lem:important-properties-of-family} and Remark~\ref{rem:importance-of-exponential-family}). A similar alternation is facilitated by the parametric models adopted in \cite{garivier2016optimal,moulos2019optimal,mukherjee2023best}. Estimating $\boldsymbol{\theta}=[\theta_1, \ldots, \theta_K]^\top$ allows us to estimate the TPMs and the arm means {\em simultaneously}.
    %Under the parametric model of Section~\ref{sec:notations}, the unknown parameters of the arms bear a one-one correspondence with the arm means. Therefore, the estimation of arm means is akin to the estimation of the unknown parameters, and thereby an estimation of the TPMs corresponding to these parameters.
\end{remark}
}
\section{\sc Achievability: A Policy for Best Arm Identification}
\label{sec:achievability}

In this section, we propose a policy for BAI that works with the {\em set} of optimal solutions \eqref{eq:optimiser-set} at each time, and ensures that the long-term state-action visitation proportions converge to the ``correct'' set of optimal proportions. 
%This is in contrast to  prior works wherein the policies work with the unique optimal solution at each time instant.

\subsection{\sc Parameter Estimates}
We start by forming estimates for the unknown parameters of the arms.
%, and provide a sufficient condition for the convergence of these estimates.
Noting the one-to-one correspondence between $\theta \in \Theta$ and $\eta_\theta \in (m_f, M_f)$ from Lemma \ref{lem:important-properties-of-family}, it suffices to estimate $\eta_{\theta_a}$ for each $a\in [K]$. For all $n$ and $a\in [K]$, let $N_a(n)=\sum_{(\mathbf{d}, \mathbf{i}) \in \mathbb{S}_R} N(n, \mathbf{d}, \mathbf{i}, a)$ denote the number of times arm $a$ is selected up to time $n$, where $N(n, \mathbf{d}, \mathbf{i}, a)$ is as defined in \eqref{eq:N(n,d,i,a)-and-N(n,d,i)}.
\iffalse
For $a \in [K]$, let 
\begin{equation}
    \mathbb{V}_{R, a} \coloneqq \mathbb{S}_R \setminus \bigcup_{a' \neq a} \mathbb{S}_{R, a'}
    \label{eq:V-r-a}
\end{equation}
denote the set of all $(\mathbf{d}, \mathbf{i})$ pairs in $\mathbb{S}_R$ such that $(\mathbf{d}, \mathbf{i}, a)$ is a {\em valid} tuple, i.e., following the occurrence of any $(\mathbf{d}, \mathbf{i}) \in \mathbb{V}_{R, a}$, the selection of arm $a$ is possible. Because of the $R$-max-delay constraint, if $(\mathbf{d}, \mathbf{i})$ is such that the delay of arm $a'$ is equal to $R$ for some $a' \neq a$, then following the occurrence of such $(\mathbf{d}, \mathbf{i})$, arm $a'$ forcefully selected and therefore the selection of arm $a$ is not possible. It is precisely all such $(\mathbf{d}, \mathbf{i})$ pairs that are excluded in \eqref{eq:V-r-a}. Notice that $N_a(n) = \sum_{(\mathbf{d}, \mathbf{i}) \in \mathbb{V}_{R, a}} N(n, \mathbf{d}, \mathbf{i}, a)$.

With the above notations in place,
\fi
Subsequently, our estimates $\widehat{\boldsymbol{\eta}}(n) \coloneqq [\widehat{\eta}^1(n), \ldots, \widehat{\eta}^K(n)]^\top$ at time $n$ are given by
\begin{equation}
\widehat{\eta}^a(n) = 
\begin{cases}
    0, & N_a(n)=0\ , \\
    \frac{1}{N_a(n)} \sum_{t=0}^{n} \mathbf{1}_{\{A_t=a\}} \, f(\bar{X}_t), & N_a(n) > 0\ .
\end{cases}
%= \frac{1}{N_a(n)} \sum_{(\mathbf{d}, \mathbf{i}) \in \mathbb{V}_{R, a}} \, \sum_{t=K}^{n} \mathbf{1}_{\{\mathbf{d}(t)=\mathbf{d}, \ \mathbf{i}(t)=\mathbf{i}, \ A_t=a\}} \, f(\bar{X}_t), \quad a\in [K].
    \label{eq:parameter-estimates}
\end{equation}
\iffalse
The following lemma provides a sufficient condition for the convergence of the estimates in \eqref{eq:parameter-estimates}.

\begin{lemma}
\label{lem:convergence-of-ML-estimates}
Fix $\boldsymbol{\theta} \in \Theta^K$. Suppose that for some $\alpha > 0$,
\begin{equation}
    \liminf_{n\to\infty} \ \frac{N(n, \mathbf{d}, \mathbf{i}, a)}{n^\alpha}>0 \quad \text{\rm a.s.} \quad \forall \ (\mathbf{d}, \mathbf{i}) \in \mathbb{V}_{R, a}, \ a\in [K].
    \label{eq:sufficient-condition-on-no-of-arm-pulls}
\end{equation}
Then, for all $a \in [K]$,
\begin{equation}
    |\widehat{\eta}^a(n) - \eta_{\theta_a}| \stackrel{n \to \infty}{\longrightarrow} 0 \quad  \text{\rm a.s.}.
    \label{eq:convergence-of-estimates}
\end{equation}
\end{lemma}
\fi 
The next step in the design of our policy, a crucial step, is the construction of an arms selection rule under which almost surely, (a) the above estimates converge to their true values, and (b) the state-action visitation proportions inherently converge to the correct set of optimal proportions.

\subsection{\sc Arms Selection Rule}
Recall the uniform arm selection policy $\pi^{\text{\rm unif}}$ defined in \eqref{eq:uniform-arm-selection-rule}. 
Fix $\eta \in (0,1)$. Let $\widehat{\theta}_a(n) = \dot{A}^{-1}(\widehat{\eta}^a(n))$ for each $a \in [K]$, {\color{black} where $\dot{A}^{-1}$ denotes the inverse of the mapping $\theta \mapsto \dot{A}(\theta) = \frac{d}{d\theta} \log \rho(\theta)$}. Let ${\boldsymbol{\widehat\theta}}(n) = [\widehat{\theta}_a(n): a \in [K]]^\top$. and let $\widehat{\boldsymbol{\theta}}(n)=[\widehat{\theta}_1(n), \ldots, \widehat{\theta}_a(n)]^\top
$ denote the vector of estimated arm parameters at time $n$. Choose an arbitrary $\nu_n^\star \in \mathcal{W}^\star(\widehat{\boldsymbol{\theta}}(n))$, and let
\iffalse
\begin{equation}
    \nu_n(\mathbf{d}, \mathbf{i}, a) = \eta\, \nu_{\widehat{\boldsymbol{\theta}}(n)}^{\text{\rm unif}}(\mathbf{d}, \mathbf{i}, a) + (1-\eta)\, \nu_n^\star(\mathbf{d}, \mathbf{i}, a) \quad \forall (\mathbf{d}, \mathbf{i}, a).
    \label{eq:certainty-equivalence-nu-n}
\end{equation}
Define
\fi 
\begin{equation}
   \pi_{\widehat{\boldsymbol{\theta}}(n)}^\eta(a|\mathbf{d}, \mathbf{i}) \coloneqq \frac{\eta\, \nu_{\widehat{\boldsymbol{\theta}}(n)}^{\text{\rm unif}}(\mathbf{d}, \mathbf{i}, a) + (1-\eta)\, \nu_n^\star(\mathbf{d}, \mathbf{i}, a)}{\eta\, \mu_{\widehat{\boldsymbol{\theta}}(n)}^{\text{\rm unif}}(\mathbf{d}, \mathbf{i}) + (1-\eta)\, \sum_{a'=1}^{K} \nu_n^\star(\mathbf{d}, \mathbf{i}, a')}\ , \quad (\mathbf{d}, \mathbf{i}, a) \in \mathbb{S}_R \times [K]\ .
   \label{eq:lambda-n-definition}
\end{equation}
Let $\{\varepsilon_n\}_{n=1}^{\infty}$ be a sequence such that $\varepsilon_n > 0$ for all $n$ and $\varepsilon_n \to 0$ as $n \to \infty$. Let
\begin{equation}
    \pi_{n} = \varepsilon_{n} \pi^{\text{\rm unif}} + (1-\varepsilon_{n}) \, \pi_{\widehat{\boldsymbol{\theta}}(n-1)}^\eta\ , \quad \forall n \geq K\ .
    \label{eq:pi-n-definition}
\end{equation}
Then, for all $n \geq K$, our arms selection rule is as follows: 
%\AT{(1) should we change $\Pr$ to $\P_\theta$? (2) in the definitions of $\pi_n$ we didn't mention their probabilistic nature. This has appeared before too. I believe we need to specify these in the previous sections, where we have all the ingredients for it.}
\begin{equation}
    \Pr (A_{n}=a|A_{0:n-1}, \bar{X}_{0:n-1}) = \pi_{n}(a|\mathbf{d}(n), \mathbf{i}(n))\ , \quad a \in [K]\ .
    \label{eq:arms-selection-rule}
\end{equation}
Note that \eqref{eq:arms-selection-rule} defines a {\em conditional} probability distribution on the arms, conditional on the arm delays and last observed states. Our recipe for selecting the arms, based on using a {\em mixture} with uniform policy as in \eqref{eq:pi-n-definition}, is inspired by \cite{albert1961sequential,al2021navigating} and plays a critical role in proving that the MDP $\mathcal{M}_{\boldsymbol{\theta}, R}$ has ``near-ergodicity'' properties under the rule in \eqref{eq:arms-selection-rule} for every $\boldsymbol{\theta} \in \Theta^K$. As we shall shortly see, the latter near-ergodicity property hinges on the fact that $\pi_n(a|\mathbf{d}, \mathbf{i}) \geq \varepsilon_n \, \pi^{\text{\rm unif}}(a|\mathbf{d}, \mathbf{i}) = \varepsilon_n/K > 0$ whenever the arm delays are all strictly smaller than $R$. 

\begin{remark}[$\eta$-mixture]
    It is unclear whether $ \sum_{a=1}^{K} \nu_n^\star(\mathbf{d}, \mathbf{i}, a) > 0$ for all $(\mathbf{d}, \mathbf{i}) \in \mathbb{S}_R$. 
If the preceding property indeed holds, we may simply use $\pi_{\widehat{\boldsymbol{\theta}}(n)}^\eta(a|\mathbf{d}, \mathbf{i}) = \nu_n^\star(\mathbf{d}, \mathbf{i}, a) / \sum_{a=1}^{K} \nu_n^\star(\mathbf{d}, \mathbf{i}, a)$. Recognizing that this property may not potentially hold true, we design an ``$\eta$-mixture'' of $\nu_n^\star$ with $\nu_{\widehat{\boldsymbol{\theta}}(n)}^{\text{\rm unif}}$, and normalize this mixture to arrive at \eqref{eq:lambda-n-definition}. Observe that the denominator of the right-hand side of \eqref{eq:lambda-n-definition} is strictly positive for every $(\mathbf{d}, \mathbf{i}) \in \mathbb{S}_R$ and hence well defined.
\end{remark}

\subsection{\sc Test Statistic, Stopping Rule, and Recommendation Rule}
Let $S_R = |\mathbb{S}_R|$ denote the cardinality of the set $\mathbb{S}_R$. Recall from  \eqref{eq:parameter-estimates} that $\widehat{\boldsymbol{\eta}}(n) = (\widehat{\eta}^a(n): a \in [K])$ denotes the estimates of the arm means at time~$n$.
For all $n \geq K$ and $(\mathbf{d}, \mathbf{i}, a) \in \mathbb{S}_R \times [K]$, let 
%\AT{if the elements in $\mathbf{1}_{\{\}}$ are hard to see, we can instead use $\mathbf{1}\{\}$.}
\begin{equation}
    \widehat{Q}_n(\mathbf{d}', \mathbf{i}'|\mathbf{d}, \mathbf{i}, a) \coloneqq 
    \begin{cases}
        \frac{1}{N(n, \mathbf{d}, \mathbf{i}, a)} \sum_{t=K}^{n} \mathbf{1}_{\{(\mathbf{d}(t), \mathbf{i}(t))=(\mathbf{d}, \mathbf{i}), \, A_t=a, \, (\mathbf{d}(t+1), \mathbf{i}(t+1))=(\mathbf{d}', \mathbf{i}')\}}, & N(n, \mathbf{d}, \mathbf{i}, a) > 0\ , \\
        \frac{1}{S_R}, & N(n, \mathbf{d}, \mathbf{i}, a) = 0\ .
    \end{cases}
    \label{eq:categorical-distribution}
\end{equation}
Note that $\sum_{(\mathbf{d}', \mathbf{i}') \in \mathbb{S}_R} \widehat{Q}_n(\mathbf{d}', \mathbf{i}'|\mathbf{d}, \mathbf{i}, a)=1$, and hence \eqref{eq:categorical-distribution} defines a probability mass function on $\mathbb{S}_R$. 
Our test statistic at time $n$, denoted by $Z(n)$, is then given by
\begin{equation}
    Z(n) \coloneqq \inf_{\boldsymbol{\lambda} \in \textsc{Alt}(\boldsymbol{\widehat{\theta}}(n))} \sum_{(\mathbf{d}, \mathbf{i}) \in \mathbb{S}_R} \ \sum_{a = 1}^{K} N(n, \mathbf{d}, \mathbf{i}, a) \, D_{\text{\rm KL}}(\widehat{Q}_n(\cdot \mid\mathbf{d}, \mathbf{i}, a) \| Q_{\boldsymbol{\lambda}, R}(\cdot \mid\mathbf{d}, \mathbf{i}, a))\ ,
    \label{eq:Z-of-n}
\end{equation}
where $\widehat{Q}_n$ is as defined in \eqref{eq:categorical-distribution}. Furthermore, let
\begin{equation}
    \zeta(n, \delta) \coloneqq \log\left(\frac{1}{\delta}\right) + (S_R - 1)\, \sum_{(\mathbf{d}, \mathbf{i}) \in \mathbb{S}_R} \ \sum_{a=1}^{K} \, \log\left( e \left[ 1 + \frac{N(n, \mathbf{d}, \mathbf{i}, a)}{S_R-1}\right] \right)\ .
    \label{eq:threshold-zeta}
\end{equation}
Combining the test statistic in \eqref{eq:Z-of-n} and the threshold in \eqref{eq:threshold-zeta}, we define our stopping rule as follows: %\AT{should we change $\tau_\delta$ to $\tau$?}
\begin{equation}
    \tau \coloneqq \inf\{n \geq K: Z(n) \geq \zeta(n, \delta)\}\ .
    \label{eq:stopping-time}
\end{equation}
At the stopping time, we output the arm with the largest empirical mean value, i.e., $\argmax_{a \in [K]} \widehat{\eta}_a(\tau)$.

In summary, our policy, which we call {\em restless D-tracking} or \textsc{Rstl-Dtrack} in short, takes the following parameters as its inputs: $R \in \mathbb{N}$, $K \in \mathbb{N}$, $\eta \in (0,1)$, and $\delta \in (0, 1)$. To start, the policy selects arm $1$ at time $n=0$, arm $2$ at time $n=1$, and so on until arm $K$ at time $n=K-1$. For all $n \geq K$, it checks for the validity of the condition $Z(n) \geq \zeta(n, \delta)$ (defined in \eqref{eq:Z-of-n}). If this condition holds, the policy stops and outputs $\argmax_{a} \widehat{\eta}_a(n)$. If $Z(n) < \zeta(n, \delta)$, then the policy continues and selects arm $A_{n+1}$ according to the rule in \eqref{eq:arms-selection-rule} while respecting the $R$-max-delay constraint. We write $\pi^{\textsc{Rstl-Dtrack}}$ to symbolically denote the policy \textsc{Rstl-Dtrack}. The pseudocode for $\pi^{\textsc{Rstl-Dtrack}}$ is given in Algorithm~\ref{alg:rstl-dtracking}. In Section~\ref{sec:concluding-remarks-and-discussion} later in the paper, we make some remarks on the computational aspects of our policy (specifically on evaluating the infimum expression in \eqref{eq:Z-of-n} at every time step).

\begin{remark}
    The definition of $Z(n)$ in \eqref{eq:Z-of-n} resembles \eqref{eq:psi} albeit with (a) $\boldsymbol{\theta}$ replaced with $\widehat{\boldsymbol{\theta}}(n)$, and (b) $Q_{\boldsymbol{\theta}, R}$ replaced with $\widehat{Q}_{n}$. In the settings of the prior works \cite{garivier2016optimal,moulos2019optimal,mukherjee2023best}, \eqref{eq:Z-of-n} specializes to the classical generalized likelihood ratio (GLR) test statistic having simple closed-form expressions. However, \eqref{eq:Z-of-n} does not admit a simple closed-form expression because of the presence of arm delays of order $2$ or higher (which are absent from \cite{garivier2016optimal,moulos2019optimal,mukherjee2023best}). In \cite{al2021navigating}, a simplification to \eqref{eq:Z-of-n} is proposed by relaxing the infimum to a larger set than $\textsc{Alt}(\widehat{\boldsymbol{\theta}}(n))$ by leveraging the specific structure of rewards therein. A similar simplification is not possible in our setting because the notion of rewards is absent in our work. See Section~\ref{sec:concluding-remarks-and-discussion} for a further discussion.
\end{remark}

\begin{algorithm}[t]
    \caption{D-Tracking for BAI in Restless Multi-Armed Bandits (\textsc{Rstl-Dtrack})}
    \begin{algorithmic}[1]
        \REQUIRE ~~\\
        $K \in \mathbb{N}$: number of arms.\\
        $R \in \mathbb{N}$: maximum tolerable arm delay. \\
        $\eta \in (0,1):$ mixture parameter. \\
        $\delta\in (0,1)$: confidence level.\\

        \ENSURE $a_{\pi^{\textsc{Rstl-Dtrack}}}$: the best arm. 
        \STATE Initialise: 
        $n=0$, 
        $N_a(n)=0$, $\widehat{\eta}_a(n)=0$ for all $a \in [K]$, \\
        $N(n, \mathbf{d}, \mathbf{i}, a)=0$ for all $(\mathbf{d}, \mathbf{i}, a) \in \mathbb{S}_R \times [K]$, $\texttt{stop}=0$. 
        \FOR{$n < K$}
            \STATE Select arm $A_n = n+1$.
        \ENDFOR
        \WHILE{$\texttt{stop} == 0$}
            \STATE Update $(\mathbf{d}(n), \mathbf{i}(n))$. Update $\widehat{\eta}_a(n)$ for each $a \in [K]$. 
            \STATE Set ${\color{black} \widehat{\theta}_a(n)} = \Dot{A}^{-1}(\widehat{\eta}_a(n))$ for each $a \in [K]$. Set $\widehat{\boldsymbol{\theta}}(n)=[\widehat{\theta}_1(n), \ldots, \widehat{\theta}_a(n)]^\top$.
            \STATE Evaluate $Z(n)$ according to \eqref{eq:Z-of-n}.
            \IF{$Z(n) \geq \zeta(n, \delta)$}
                \STATE $\texttt{stop}=1$.
                \STATE $\widehat{a} = \argmax_{a} \widehat{\eta}_a(n)$. Resolve ties at random.
            \ELSE
                \STATE Select $A_{n} \sim \pi_n(\cdot | \mathbf{d}(n), \mathbf{i}(n))$, where $\pi_n$ is as defined in \eqref{eq:pi-n-definition}.
                \STATE $n \leftarrow n+1$.
            \ENDIF
        \ENDWHILE

        \RETURN $\widehat{a}$.
\end{algorithmic}
\label{alg:rstl-dtracking}
\end{algorithm}

\section{\sc Theoretical Guarantees}
\label{sec:results}

In this section, we provide theoretical guarantees for the proposed $\textsc{Rstl-Dtrack}$ policy. 
{\color{black}
We first present the key lemmas pertaining to the arms selection rule in \eqref{eq:arms-selection-rule} that are pivotal in proving the asymptotic optimality of the {\sc Rstl-Dtrack} policy. This section is organized as follows. In Lemma~\ref{lem:sufficient-exploration-of-state-actions}, we show that each state-action tuple $(\mathbf{d}, \mathbf{i}, a) \in \mathbb{S}_R \times [K]$ is visited infinitely often almost surely. Furthermore, given any $\alpha \in (0,1)$, we show in Lemma~\ref{lem:sufficient-exploration-of-state-actions} that each state-action tuple is visited at a rate of $O_{\alpha}(n^{1/4})$ with high probability greater than $1-\alpha$; here, $O_{\alpha}(\cdot)$ captures the dependence on $\alpha$. Next, in Lemma~\ref{lem:concentration-of-empirical-arm-means}, we establish the concentration of the empirical arm means $\widehat{\boldsymbol{\eta}}(n)$ around their true values $\boldsymbol{\eta}$ with high probability. Under this concentration event, we show in Lemma~\ref{lem:concentration-of-state-action-visitations} that under the instance $\boldsymbol{\theta} \in \Theta^K$, the empirical state-action visitations concentrate around the set of all probability distributions on state-actions that are an $\eta$-mixture of $\nu_{\boldsymbol{\theta}}^{\rm unif}$ and elements of $\mathcal{W}^\star(\boldsymbol{\theta})$. In Proposition~\ref{prop:stop-in-finite-time-and-error-prob-less-than-delta}, we show that for any given $\delta$, the probability of the joint event that {\sc Rstl-Dtrack} stops in finite time and outputs the best arm incorrectly is upper bounded by $\delta$. In Proposition~\ref{prop:almost-sure-upper-bound}, we derive an almost sure asymptotic upper bound on the stopping time of {\sc Rstl-Dtrack}, in the asymptotic regime as $\delta \downarrow 0$. Combining Propositions~\ref{prop:stop-in-finite-time-and-error-prob-less-than-delta},\ref{prop:almost-sure-upper-bound}, we conclude that $\pi^{\textsc{Rstl-Dtrack}} \in \Pi(\delta)$ for each $\delta \in (0,1)$. Finally, in Proposition~\ref{prop:upper-bound-on-expected-stopping-time}, we derive an upper bound on the expected stopping time of {\sc Rstl-Dtrack} that matches with the almost sure upper bound of Proposition~\ref{prop:almost-sure-upper-bound}. 
}

{\color{black} \subsection{Key Lemmas Pertaining to the Arms Selection Rule \texorpdfstring{\eqref{eq:arms-selection-rule}}{(equation-arms-selection-rule}}}

Let $\mathbb{V}$ denote the set of all {\em valid} $(\mathbf{d}, \mathbf{i}, a)$ tuples, i.e., those tuples for which the selection of arm $a$ in state $(\mathbf{d}, \mathbf{i})$ is permissible under the $R$-max-delay constraint. That is, for any $(\mathbf{d}, \mathbf{i}, a) \notin \mathbb{V}$, we have $N(n, \mathbf{d}, \mathbf{i}, a)=0$ almost surely for all $n \geq K$.

The first result below shows that under the proposed arms selection rule in \eqref{eq:arms-selection-rule}, every valid $(\mathbf{d}, \mathbf{i}, a)$ tuple is visited infinitely often and at a rate of $\Omega(n^{1/4})$ with high probability.
\begin{lemma}
    \label{lem:sufficient-exploration-of-state-actions}
    Fix $\boldsymbol{\theta} \in \Theta^K$. Let $S_R = |\mathbb{S}_{R}|$. 
    \begin{enumerate}
        \item The proposed arms selection rule in \eqref{eq:arms-selection-rule} with $\varepsilon_n = n^{-\frac{1}{2(1+S_R)}}$ satisfies
    \begin{equation}
        \P_{\boldsymbol{\theta}}\bigg( \forall(\mathbf{d}, \mathbf{i}, a) \in \mathbb{V}, \quad \lim_{n \to \infty} N(n, \mathbf{d}, \mathbf{i}, a)=+\infty \bigg) = 1\ .
        \label{eq:almost-sure-divergence-of-state-action-visitations}
    \end{equation}

    \item Under the above arms selection rule, for every $\alpha \in (0,1)$,
    \begin{equation}
        \P_{\boldsymbol{\theta}}\left(\forall(\mathbf{d}, \mathbf{i}, a) \in \mathbb{V}, ~\forall n \geq K, \quad N(n, \mathbf{d}, \mathbf{i}, a) \geq \left(\frac{n}{\lambda_\alpha(\boldsymbol{\theta})}\right)^{1/4} -1 \right) \geq 1-\alpha\ ,
        \label{eq:sufficient-exploration-of-states-and-actions}
    \end{equation}
    where $\lambda_\alpha(\boldsymbol{\theta})=\frac{(1+S_R)^2}{\sigma_{\boldsymbol{\theta}}^2} \, \log^2(1+\frac{K\,S_R}{\alpha})$. Here, $\sigma_{\boldsymbol{\theta}}>0$ is a constant that depends only on $\boldsymbol{\theta}$.
    \end{enumerate}
\end{lemma}

\begin{remark}[Choice of $\varepsilon_n$]
    Our proof of Lemma~\ref{lem:sufficient-exploration-of-state-actions} is an adaptation of a similar proof in \cite{al2021navigating}. Notice that the ``$\varepsilon_n$-mixture'' rule in \eqref{eq:pi-n-definition} satisfies the following decomposition property for the transition kernels that facilitates analysis:
    \begin{equation}
        Q_{\boldsymbol{\theta}, \pi_n} = \varepsilon_n \, Q_{\boldsymbol{\theta}, \pi^{\text{\rm unif}}} + (1-\varepsilon_n) \, Q_{\boldsymbol{\theta}, \pi_{\widehat{\boldsymbol{\theta}}(n-1)}^\eta}.
        \label{eq:property-of-epsilon-n-mixture}
    \end{equation}
    Choosing $\varepsilon_n=n^{-\beta}$ where $\beta < \frac{1}{1+S_R}$ leads to a convenient closed-form expression for $\lambda_{\alpha}(\boldsymbol{\theta})$. We use $\beta = \frac{1}{2(1+S_R)}$, and hence $\varepsilon_n = n^{-\frac{1}{2(1+S_R)}}$. For additional details, we refer the reader to the proof of Lemma~\ref{lem:sufficient-exploration-of-state-actions} in the appendix.
    %for all $n \in \mathbb{N}$. 
    %See the discussion following Lemma~\ref{lem:sufficient-exploration-of-state-actions} below for additional details.
\end{remark}
An immediate consequence of Lemma~\ref{lem:sufficient-exploration-of-state-actions} is that under the proposed arms selection rule in \eqref{eq:arms-selection-rule}, each arm $a \in [K]$ is explored at a rate $\Omega(n^{1/4})$ with high probability (w.h.p.), thereby ensuring that w.h.p., we have ${\color{black} \widehat{\boldsymbol{\eta}}(n)} \to \boldsymbol{\eta}$, {\color{black} where $\boldsymbol{\eta}=[\eta_{\theta_a}: a \in [K]]^\top$}. This is formalized in the following lemma.
%The next result shows that under the sampling rule in \eqref{eq:arms-selection-rule}, the vector of empirical state-action proportions approaches the correct set of optimal allocations asymptotically.
\iffalse
\begin{lemma}
    \label{lem:convergence-to-the-correct-set-of-optimal-allocations}
    Fix $\eta \in (0,1)$. For every $\boldsymbol{\theta} \in \Theta^K$, under the sampling rule in \eqref{eq:arms-selection-rule} with $\varepsilon_n = n^{-\frac{1}{2(1+S_R)}}$ for all $n \ge 1$, we have
    \begin{equation}
        \lim_{n \to \infty} d_{\infty}\left(\left\lbrace \frac{N(n, \mathbf{d}, \mathbf{i}, a)}{n-K+1} \right\rbrace_{(\mathbf{d}, \mathbf{i}, a)},  \ \mathcal{W}_\eta^\star(\boldsymbol{\theta})\right)=0\ ,
    \end{equation}
    where $\mathcal{W}_\eta^\star(\boldsymbol{\theta}) \coloneqq \{\eta \nu_{\boldsymbol{\theta}}^{\text{\rm unif}} + (1-\eta) \nu: \nu \in \mathcal{W}^\star(\boldsymbol{\theta})\}$.
\end{lemma}
\fi

\begin{lemma}
    \label{lem:concentration-of-empirical-arm-means}
    Given $\xi>0$ and a positive integer $N \geq K$, let 
    \begin{equation}
        C_N^2(\xi) \coloneqq \bigcap_{n=N^5}^{N^6} \bigg\lbrace \|\widehat{\boldsymbol{\eta}}(n) - \boldsymbol{\eta}\|_2 \leq \xi \bigg\rbrace\ .
        \label{eq:C-N-2-of-xi-event}
    \end{equation}
    Consider the non-stopping version of policy $\pi^{\textsc{Rstl-Dtrack}}$ (with the same parameters as that of $\pi^{\textsc{Rstl-Dtrack}}$). Under this policy, for all $\xi>0$ and $N \geq K$,
    \begin{equation}
        \P_{\boldsymbol{\theta}}\left(\overline{C_N^2(\xi)}\right) \leq \frac{1}{N^2} + \frac{2^{K/2 + 2} \, K^{K/4}}{\sigma_{\boldsymbol{\theta}}^{K/4}}\, N^{9K/4 + 7} \, \exp\left(-\frac{\sqrt{\sigma_{\boldsymbol{\theta}}}\, \xi^2 \, N^{1/4}}{8\, \sqrt{K} \, (2 \, M_f)}\right)\ ,
        \label{eq:concentration-of-empirical-arm-means}
    \end{equation}
    where $\sigma_{\boldsymbol{\theta}}$ is the constant from Lemma~\ref{lem:sufficient-exploration-of-state-actions}, and $M_f = \max_{i \in \mathcal{S}} f(i)$.
\end{lemma}
Combining Lemma~\ref{lem:concentration-of-empirical-arm-means} along with the upper-hemicontinuity property of the mapping $\boldsymbol{\lambda} \to \mathcal{W}^\star(\boldsymbol{\lambda})$ from Lemma~\ref{lem:attainment-of-supremum}, we establish a concentration result for the empirical state-action visitation proportions under $C_N^2(\xi)$.
\begin{lemma}
    \label{lem:concentration-of-state-action-visitations}
     Fix $\boldsymbol{\theta} \in \Theta^K$, $\nu \in \mathcal{W}^\star(\boldsymbol{\theta})$, and $\eta \in (0,1)$. Let $\omega_{\boldsymbol{\theta}, \nu}^\star = \eta\, \nu_{\boldsymbol{\theta}}^{\text{\rm unif}} + (1-\eta) \, \nu$. Consider the non-stopping version of policy $\pi^{\textsc{Rstl-Dtrack}}$ (with the same parameters as that of $\pi^{\textsc{Rstl-Dtrack}}$). Under this policy, for all $\xi>0$, there exists a time $N_{\xi}>0$ such that for all $N \geq N_{\xi}$ and all $n \geq \sqrt{N} + 1$, 
        \begin{equation}
            \P_{\boldsymbol{\theta}}\left(\exists (\mathbf{d}, \mathbf{i}, a): \left\lvert  \frac{N(n, \mathbf{d}, \mathbf{i}, a)}{n-K+1} - \omega_{\boldsymbol{\theta}, \nu}^\star(\mathbf{d}, \mathbf{i}, a)\right\rvert > K_{\xi} (\boldsymbol{\theta}, \nu) \, \xi \, \bigg| \, C_N^2(\xi)\right) = O\bigg(\exp \left(-n \xi^2\right)\bigg)\ ,
            \label{eq:concentration-of-state-action-visitations}
        \end{equation}
        where $K_{\xi}(\boldsymbol{\theta}, \nu)$ is a constant that depends on $\xi$, $\boldsymbol{\theta}$ and $\nu$, and satisfies
        \begin{equation}
            \limsup_{\xi \downarrow 0} K_{\xi}(\boldsymbol{\theta}, \nu) < +\infty \quad \forall \nu \in \mathcal{W}^\star(\boldsymbol{\theta}), \ \boldsymbol{\theta} \in \Theta^K\ .
            \label{eq:limsup-K-xi-is-finite}
        \end{equation}
\end{lemma}
Lemma~\ref{lem:concentration-of-state-action-visitations} is one of the important results of this paper. It establishes that under any instance $\boldsymbol{\theta} \in \boldsymbol{\Theta}^K$, the empirical state-action visitation proportions converge w.h.p. to $\omega_{\boldsymbol{\theta}, \nu}^\star$ for every $\nu \in \mathcal{W}^\star(\boldsymbol{\theta})$. Disregarding the scaling factor $\eta$ in the expression for $\omega_{\boldsymbol{\theta}, \nu}^\star$, the above result implies that under the instance $\boldsymbol{\theta}$, the empirical state-action visitation proportions converge to the desired set $\mathcal{W}^\star(\boldsymbol{\theta})$. This, as we shall soon see, is pivotal to establishing asymptotic optimality of the policy $\textsc{Rstl-Dtrack}$. In the proof, we show that under the policy $\textsc{Rstl-Dtrack}$, the MDP $\mathcal{M}_{\boldsymbol{\theta}, R}$ possesses a ``near-ergodicity'' property in the following sense: for any fixed $n$, if $\pi=\pi_n$ is used for selecting the arms at all times, then by virtue of Lemma~\ref{lem:ergodicity-of-MDP-under-unif-policy}, the corresponding transition kernel $Q_{\boldsymbol{\theta}, \pi}$ is ergodic; let its stationary distribution under the instance $\boldsymbol{\theta}$ be $\omega_{\boldsymbol{\theta}, n}^\star$. We find a bound on $\|\omega_{\boldsymbol{\theta}, n}^\star-\omega_{\boldsymbol{\theta}, \nu}^\star\|_{\infty}$ to arrive at the exponential bound in \eqref{eq:concentration-of-state-action-visitations}.

{\color{black} \subsection{Key Results on the Performance of \texorpdfstring{\sc Rstl-Dtrack}{Rstl-Dtrack}}}

In this section, we present the key results on the performance of {\sc Rstl-Dtrack} policy. The first result below demonstrates that any arbitrary arms selection rule, in conjunction with the stopping rule in \eqref{eq:stopping-time} and the threshold in \eqref{eq:threshold-zeta}, satisfies the desired error probability constraint.
\begin{proposition}
\label{prop:stop-in-finite-time-and-error-prob-less-than-delta}
Fix $\boldsymbol{\theta} \in 
\Theta^K$. For all $\delta \in (0,1)$,
\begin{equation}
    \P_{\boldsymbol{\theta}}\left(\exists n \geq K: \ \ \sum_{(\mathbf{d}, \mathbf{i}) \in \mathbb{S}_R} \ \sum_{a=1}^{K} \ N(n, \mathbf{d}, \mathbf{i}, a)\, D_{\text{\rm KL}}(\widehat{Q}_n(\cdot \mid\mathbf{d}, \mathbf{i}, a) \| Q_{\boldsymbol{\theta}, R}(\cdot \mid\mathbf{d}, \mathbf{i}, a)) > \zeta(n, \delta)\right) \leq \delta\ .
\end{equation}
Consequently, for any algorithm with an arbitrary sampling rule, stopping time $\tau$ given by \eqref{eq:stopping-time} (with the threshold as in \eqref{eq:threshold-zeta}), and best arm recommendation $\widehat{a} = \arg\max_{a} \widehat{\eta}_a(\tau)$ we have 
\begin{equation}
    \P_{\boldsymbol{\theta}}(\tau < \infty, \ \eta_{\widehat{a}} < \eta_{a^\star(\boldsymbol{\theta})}) \leq \delta\ .
    \label{eq:stop-in-finite-time-and-error-prob-less-than-delta}
\end{equation}
\end{proposition}
In particular, we note that \eqref{eq:stop-in-finite-time-and-error-prob-less-than-delta} holds for the proposed arms selection rule in \eqref{eq:arms-selection-rule}. The next result below shows that the stopping time of policy $\textsc{Rstl-Dtrack}$ is finite almost surely, and satisfies an almost-sure asymptotic upper bound that nearly matches with the lower bound in \eqref{eq:lower-bound}.
\begin{proposition}
    \label{prop:almost-sure-upper-bound}
    Fix $\eta \in (0,1)$. For all $\delta \in (0,1)$, the stopping time $\tau$ of policy  $\pi^{\textsc{Rstl-Dtrack}}$ is finite almost surely, and hence $\pi^{\textsc{Rstl-Dtrack}} \in \Pi(\delta)$. Furthermore,
    \begin{equation}
        \P_{\boldsymbol{\theta}}\left( \limsup_{\delta \downarrow 0} \frac{\tau}{\log(1/\delta)} \leq \frac{1}{\eta\, T_{\text{\rm unif}}^\star(\boldsymbol{
        \theta
        }) + (1-\eta)\, T_R^\star(\boldsymbol{\theta})}\right)=1\ ,
        \label{eq:almost-sure-upper-bound}
    \end{equation}
    where $T_{\text{\rm unif}}^\star(\boldsymbol{\theta})$ in \eqref{eq:almost-sure-upper-bound} is defined as
    \begin{equation}
        T_{\text{\rm unif}}^\star(\boldsymbol{\theta}) \coloneqq \inf_{\boldsymbol{\lambda} \in \textsc{Alt}(\boldsymbol{\theta})} \sum_{(\mathbf{d}, \mathbf{i}) \in \mathbb{S}_R} \ \sum_{a=1}^{K} \nu_{\boldsymbol{\theta}}^{\text{\rm unif}}(\mathbf{d}, \mathbf{i}, a) \, D_{\text{\rm KL}}(Q_{\boldsymbol{\theta}, R}(\cdot \mid\mathbf{d}, \mathbf{i}, a) \| Q_{\boldsymbol{\lambda}, R}(\cdot \mid\mathbf{d}, \mathbf{i}, a))\ .
    \end{equation}
\end{proposition}
Having established that {\sc Rstl-Dtrack} satisfies the desired error probability constraint, we present below the main result of this section, an upper bound on the growth rate of the expected stopping time of {\sc Rstl-Dtrack}.
\begin{proposition}
    \label{prop:upper-bound-on-expected-stopping-time}
    Fix $\eta \in (0,1)$. For all $\delta \in (0,1)$, the expected stopping time $\mathbb{E}_{\boldsymbol{\theta}}[\tau]$ of policy $\pi^{\textsc{Rstl-Dtrack}}$ is finite. Furthermore,
    \begin{equation}
        \limsup_{\delta \downarrow 0} \frac{\mathbb{E}_{\boldsymbol{\theta}}[\tau]}{\log(1/\delta)} \leq \frac{1}{\eta\, T_{\text{\rm unif}}^\star(\boldsymbol{
        \theta
        }) + (1-\eta)\, T_R^\star(\boldsymbol{\theta})}\ .
        \label{eq:upper-bound-on-expected-stopping-time-with-eta}
    \end{equation}
    Consequently, letting $\eta \downarrow 0$, we have
    \begin{equation}
        \limsup_{\eta\downarrow 0} \ \limsup_{\delta \downarrow 0} \frac{\mathbb{E}_{\boldsymbol{\theta}}[\tau]}{\log(1/\delta)} \leq \frac{1}{ T_R^\star(\boldsymbol{\theta})}\ .
        \label{eq:upper-bound-on-expected-stopping-time}
    \end{equation}
\end{proposition}
{\color{black} Combining Proposition~\ref{prop:upper-bound-on-expected-stopping-time} with Proposition~\ref{prop:lower-bound}, we see that $1/T_R^\star(\boldsymbol{\theta})$ captures the optimal growth rate of the expected stopping time for BAI in restless bandits with problem instance $\boldsymbol{\theta} \in \Theta^K$, i.e.,
\begin{equation}
    \frac{1}{T_R^\star(\boldsymbol{\theta})} \leq \liminf_{\delta \downarrow 0} \, \inf_{\pi \in \Pi_R(\delta)} \frac{\mathbb{E}_{\boldsymbol{\theta}}[\tau_\pi]}{\log(1/\delta)} \leq \limsup_{\eta\downarrow 0}\ \limsup_{\delta \downarrow 0} \frac{\mathbb{E}_{\boldsymbol{\theta}}[\tau_{\pi^{\textsc{Rstl-Dtrack}}}]}{\log(1/\delta)} \leq \frac{1}{ T_R^\star(\boldsymbol{\theta})}\ .
    \label{eq:main-result}
\end{equation}}

{
\color{black}

\section{Computational Feasibility of \textsc{Rstl-Dtrack}}

%\section{On Rendering \textsc{Rstl-Dtrack} Computational Feasible}
\label{sec:computational-efficient-variant}

In this section, we present an approach to render {\sc Rstl-Dtrack} computationally feasible. 
%policy is of notable concern, 
In particular, we discuss the computation of the infimum in \eqref{eq:Z-of-n} at each time step, which can be quite resource-intensive. Due to the presence of arm delays, analytically simplifying this infimum any further is a formidable challenge, as discussed in Section~\ref{sec:achievability}. In this section, we propose an %potential 
approach to alleviate the computational burden of evaluating the infimum in \eqref{eq:Z-of-n}, 
%albeit 
at the expense of trading off asymptotic optimality. Treating the function $f$ that is used to define the single-parameter exponential family of TPMs (see \eqref{eq:P-tilde-theta}) as a heuristic ``reward'' function for the finite-state MDP $\mathcal{M}_{\boldsymbol{\theta}, R}$, and emulating the techniques in \cite{al2021navigating}, we propose a {\em proxy} for the inner infimum term in the expression \eqref{eq:T-R-star-C} for $T_R^\star(\boldsymbol{\theta})$, one that may be easily computed in closed form. For all $\boldsymbol{\theta} \in \Theta^K$ and $R \in \mathbb{N}$, we show that the proxy term, say $U_R^\star(\boldsymbol{\theta})$, satisfies $U_R^\star(\boldsymbol{\theta}) \leq T_R^\star(\boldsymbol{\theta})$. Furthermore, following the template of Algorithm~\ref{alg:rstl-dtracking}, and replacing the test statistic $Z(n)$ in \eqref{eq:Z-of-n} with $U_R^\star(\widehat{\boldsymbol{\theta}}(n))$ at any given time $n$, we show that an asymptotic upper bound of $1/U_R^\star(\boldsymbol{\theta})$ may be achieved, thus leading to mismatched lower and upper bounds and, thereby, asymptotic sub-optimality.

\subsection{Some Notations}
We now introduce some notations. Fix $R \in \mathbb{N}$ and $\boldsymbol{\theta} \in \Theta^K$. Let $\mathcal{M}_{\boldsymbol{\theta}, R}$ denote the MDP with state space $\mathbb{S}_R$ and transition probabilities $Q_{\boldsymbol{\theta}, R}$ defined in~\eqref{eq:modified_MDP_transition_probabilities2}. Fix $\gamma \in (0,1)$ and time instant $N > K$. For any policy $\pi$, let
\begin{equation}
    V_{\boldsymbol{\theta}, R}^{\pi} (\mathbf{d}, \mathbf{i}) \coloneqq \mathbb{E}_{\boldsymbol{\theta}}\left[\sum_{n=N}^{\infty} \gamma^n \, f(\bar{X}_n) \, \bigg| \, (\mathbf{d}(N), \mathbf{i}(N)) = (\mathbf{d}, \mathbf{i})\right], \quad (\mathbf{d}, \mathbf{i}) \in \mathbb{S}_R\ ,
    \label{eq:V-theta-R-definition}
\end{equation}
where $f: \mathcal{S} \to \mathbb{R}$ is the same function appearing in \eqref{eq:P-tilde-theta} and is used to define the single-parameter exponential family of TPMs. In writing \eqref{eq:V-theta-R-definition}, we treat $f(\bar{X}_n)$ as the heuristic ``reward'' of the MDP $\mathcal{M}_{\boldsymbol{\theta}, R}$ at time $n$. Then, \eqref{eq:V-theta-R-definition} defines the {\em value function} associated with policy $\pi$. Furthermore, we note that the expectation in \eqref{eq:V-theta-R-definition} is over the randomness induced by policy $\pi$. Because the MDP $\mathcal{M}_{\boldsymbol{\theta}, R}$ is defined over the finite state space $\mathbb{S}_R$, there exists a {\em deterministic} stationary policy, say $\pi_{\boldsymbol{\theta}, R}^\star = [\pi_{\boldsymbol{\theta}, R}^\star(\mathbf{d}, \mathbf{i}): (\mathbf{d}, \mathbf{i}) \in \mathbb{S}_R]^\top$, such that (see, for instance, \cite[Theorem 6.2.10]{puterman2014markov})
\begin{equation}
    \pi_{\boldsymbol{\theta}, R}^\star \in \argmax_{\pi} V_{\boldsymbol{\theta}, R}^{\pi} \ .
    \label{eq:pi-theta-R-star-definition}
\end{equation}
%Here, $\pi_{\boldsymbol{\theta}, R}^\star(\mathbf{d}, \mathbf{i})$ denotes the arm selected by policy $\pi_{\boldsymbol{\theta}, R}^\star$ in state $(\mathbf{d}, \mathbf{i})$. 
Let $V_{\boldsymbol{\theta}, R}^\star$ be the value function of policy $\pi_{\boldsymbol{\theta}, R}^\star$. Furthermore, denote the optimal {\em $Q$-value function} corresponding to policy $\pi_{\boldsymbol{\theta}, R}^\star$ by
\begin{equation}
    {\rm QV}_{\boldsymbol{\theta}, R}^\star(\mathbf{d}, \mathbf{i}, a) \coloneqq \mathbb{E}_{\boldsymbol{\theta}}\left[\sum_{n=N}^{\infty} \gamma^n \, f(\bar{X}_n) \, \bigg| \, (\mathbf{d}(N), \mathbf{i}(N), A_N) = (\mathbf{d}, \mathbf{i}, a)\right], \quad (\mathbf{d}, \mathbf{i}, a) \in \mathbb{S}_R \times [K]\ .
    \label{eq:QV-theta-R-star-definition}
\end{equation}
We note that the expectation in \eqref{eq:QV-theta-R-star-definition} is over the randomness induced by policy $\pi_{\boldsymbol{\theta}, R}^\star$. Let 
\begin{align}
    \Delta_{\boldsymbol{\theta}, R}(\mathbf{d}, \mathbf{i}, a) \coloneqq V_{\boldsymbol{\theta}, R}^{\star}(\mathbf{d}, \mathbf{i}) - {\rm QV}_{\boldsymbol{\theta}, R}^\star(\mathbf{d}, \mathbf{i}, a)\ ,
\end{align}
%$\Delta_{\boldsymbol{\theta}, R}(\mathbf{d}, \mathbf{i}, a) \coloneqq V_{\boldsymbol{\theta}, R}^{\star}(\mathbf{d}, \mathbf{i}) - {\rm QV}_{\boldsymbol{\theta}, R}^\star(\mathbf{d}, \mathbf{i}, a)$, 
and accordingly, define
\begin{equation}
    \Delta_{\boldsymbol{\theta}, R}^{\rm min} \coloneqq \min_{\substack{(\mathbf{d}, \mathbf{i}, a): \\ a \neq \pi_{\boldsymbol{\theta}, R}^\star(\mathbf{d}, \mathbf{i})}} \Delta_{\boldsymbol{\theta}, R}(\mathbf{d}, \mathbf{i}, a)\ .
    \label{eq:Delta-min-definition}
\end{equation}
Subsequently, denote the {\em span} of $V_{\boldsymbol{\theta}, R}^\star$ by
\begin{align}
    {\rm sp}(V_{\boldsymbol{\theta}, R}^\star) = \max_{(\mathbf{d}, \mathbf{i}), (\mathbf{d}', \mathbf{i}') \in \mathbb{S}_R} |V_{\boldsymbol{\theta}, R}^\star(\mathbf{d}, \mathbf{i}) - V_{\boldsymbol{\theta}, R}^\star(\mathbf{d}', \mathbf{i}')|\ ,
\end{align}
%Let ${\rm sp}(V_{\boldsymbol{\theta}, R}^\star) = \max_{(\mathbf{d}, \mathbf{i}), (\mathbf{d}', \mathbf{i}') \in \mathbb{S}_R} |V_{\boldsymbol{\theta}, R}^\star(\mathbf{d}, \mathbf{i}) - V_{\boldsymbol{\theta}, R}^\star(\mathbf{d}', \mathbf{i}')|$ 
Finally, let 
\begin{align}
    {\rm Var}(V_{\boldsymbol{\theta}, R}^\star | \mathbf{d}, \mathbf{i}, a) 
    &\coloneqq {\rm Variance}_{(\mathbf{d}', \mathbf{i}') \sim Q_{\boldsymbol{\theta}, R}(\cdot| \mathbf{d}, \mathbf{i}, a)}(V_{\boldsymbol{\theta}, R}^\star(\mathbf{d}', \mathbf{i}')), \quad (\mathbf{d}\ , \mathbf{i}, a) \in \mathbb{S}_R \times [K]\ ,
    \label{eq:var-v-theta-R-definition} \\
    {\rm Var}_{\rm max}(V_{\boldsymbol{\theta}, R}^\star) 
    &\coloneqq \max_{(\mathbf{d}, \mathbf{i}) \in \mathbb{S}_R} {\rm Var}(V_{\boldsymbol{\theta}, R}^\star | \mathbf{d}, \mathbf{i}, \pi_{\boldsymbol{\theta}, R}^\star(\mathbf{d}, \mathbf{i}))\ .
    \label{eq:var-max-v-theta-R-definition}
\end{align}

\subsection{A Proxy for the Inner Infimum in \texorpdfstring{\eqref{eq:T-R-star-C}}{eq:T-R-star-C}}

Following \cite{al2021navigating}, we propose a proxy for the inner infimum term in \eqref{eq:T-R-star-C} that is computationally easy to solve, and subsequently use it as a substitute for the test statistic $Z(n)$ at any given time $n$.  Given any $\boldsymbol{\theta} \in \Theta^K$ and $\nu \in \Sigma_R(\boldsymbol{\theta})$, this proxy is defined as
\begin{align}
    \chi(\nu,\boldsymbol\theta)
    &\coloneqq \left ( \max\limits_{\substack{(\mathbf{d},\mathbf{i},a): \\ a \neq \pi_{\boldsymbol{\theta}, R}^\star(\mathbf{d},\mathbf{i})}} \frac{H_{\boldsymbol\theta, R}(\mathbf{d},\mathbf{i},a)}{\nu(\mathbf{d},\mathbf{i},a)} + \frac{H^\star_{\boldsymbol\theta, R}}{\min\limits_{(\mathbf{d},\mathbf{i})\in\mathbb{S}_R} \nu(\mathbf{d},\mathbf{i},\pi^\star_{\boldsymbol\theta}(\mathbf{d},\mathbf{i}))}\right)^{-1},
    \label{eq:chi-nu-theta-definition}
\end{align}
where the terms $H_{\boldsymbol\theta, R}(\mathbf{d},\mathbf{i},a)$ and $H^\star_{\boldsymbol\theta, R}$ are defined as
\begin{align}
    &H_{\boldsymbol\theta, R}(\mathbf{d},\mathbf{i},a) \coloneqq \frac{2}{\Delta_{\boldsymbol{\theta}, R}^2(\mathbf{d},\mathbf{i},a)} + \max\left \{ \frac{16 \, {\rm Var}(V_{\boldsymbol{\theta}, R}^\star| \mathbf{d},\mathbf{i}, a)}{\Delta_{\boldsymbol{\theta}, R}^2(\mathbf{d},\mathbf{i},a)}, ~ 6 \, \left(\frac{{\rm sp}(V_{\boldsymbol{\theta}, R}^\star)}{\Delta_{\boldsymbol{\theta}, R}^2(\mathbf{d},\mathbf{i},a)} \right )^{4/3}\right \},
    \label{eq:H-theta-R-dia-definition}\\
    &H^\star_{\boldsymbol\theta, R} \coloneqq \frac{2}{(1-\gamma)^2 \, (\Delta_{\boldsymbol{\theta}, R}^{\rm min})^2} + \min\left\lbrace \frac{27}{(\Delta_{\boldsymbol{\theta}, R}^{\rm min})^2 \, (1-\gamma)^3}, ~  \max\left \{ \frac{16 \, {\rm Var}_{\rm max}(V_{\boldsymbol{\theta}, R}^\star)}{(\Delta_{\boldsymbol{\theta}, R}^{\rm min})^2 \, (1-\gamma)^2}, ~ 6 \, \left(\frac{{\rm sp}(V_{\boldsymbol{\theta}, R}^\star)}{\Delta_{\boldsymbol{\theta}, R}^{\rm min} \, (1-\gamma)} \right )^{4/3}\right \}\right\rbrace. 
    \label{eq:H-theta-R-star-definition}
\end{align}
The next result, which is an adaptation of \cite[Lemma 1]{al2021navigating}, shows that the proxy is a {\em lower bound} on the inner infimum term in\eqref{eq:T-R-star-C}.
\begin{lemma}(Adaptation of \cite[Lemma 1]{al2021navigating})
    \label{lemma:proxy-lower-bound-on-inner-infimum}
    For all $\boldsymbol{\theta} \in \Theta^K$ and $\nu \in \Sigma_R(\boldsymbol{\theta})$, we have $\chi(\nu, \boldsymbol{\theta}) \leq \psi(\nu, \boldsymbol{\theta})$, where $\psi(\nu, \boldsymbol{\theta})$ is as defined in \eqref{eq:psi}.
\end{lemma}
The proof of Lemma~\ref{lemma:proxy-lower-bound-on-inner-infimum} follows along similar lines as the proof of \cite[Lemma~1]{al2021navigating}, and is omitted for brevity.

\subsection{A Computationally Efficient Alternative to \textsc{Rstl-Dtrack}}

Using the proxy defined in the previous section, we propose an alternative policy to \textsc{Rstl-Dtrack}, which is computationally efficient, by implementing two key modifications to \textsc{Rstl-Dtrack}. Before we describe these modifications, we introduce some notations. Given any $\boldsymbol{\theta} \in \Theta^K$, let
\begin{equation}
    \mathcal{X}^\star(\boldsymbol{\theta}) \coloneqq \argsup_{\nu \in \Sigma_R(\boldsymbol{\theta})} \chi(\nu, \boldsymbol{\theta})
    \label{eq:x-star-theta}
\end{equation}
denote the set of all probability distributions over $\mathbb{S}_R \times [K]$ that maximise $\chi$ over the set $\Sigma_R(\boldsymbol{\theta})$. For any $n \geq K$, denote the vector of empirical state-action frequencies at time $n$ by
\begin{align}
    \frac{\mathbf{N}(n)}{n-K+1} = \left[\frac{N(n, \mathbf{d}, \mathbf{i}, a)}{n-K+1}: (\mathbf{d}, \mathbf{i}, a) \in \mathbb{S}_R \times [K]\right]^\top\ .
\end{align}
%$\frac{\mathbf{N}(n)}{n-K+1} = \left[\frac{N(n, \mathbf{d}, \mathbf{i}, a)}{n-K+1}: (\mathbf{d}, \mathbf{i}, a) \in \mathbb{S}_R \times [K]\right]^\top$ 
Notice that the optimisation problem in \eqref{eq:x-star-theta} has a convex objective function (the mapping $\nu \mapsto \chi(\nu, \boldsymbol{\theta})$ is convex) with a convex set of constraints $\Sigma_R(\boldsymbol{\theta})$, and therefore $\mathcal{X}^\star(\boldsymbol{\theta})$ may easily be computed using an algorithm such as projected sub-gradient. 

The modifications to \textsc{Rstl-Dtrack} that lead to a computationally efficient policy are as follows.
\begin{enumerate}
    \item In lines 8,9 of Algorithm~\ref{alg:rstl-dtracking}, we replace $Z(n)$ with $\chi\left(\frac{\mathbf{N}(n)}{n-K+1}, \widehat{\boldsymbol{\theta}}(n)\right)$.

    \item To implement $\pi_n$ in line 13, we pick an arbitrary $\nu_n^\star \in \mathcal{X}^\star(\widehat{\boldsymbol{\theta}}(n-1))$ \big(instead of picking $\nu_n^\star \in \mathcal{W}^\star(\widehat{\boldsymbol{\theta}}(n-1))$ in {\sc Rstl-Dtrack}\big), and use the chosen $\nu_n^\star$ to design $\pi_n$ according to \eqref{eq:pi-n-definition}.  
\end{enumerate}
Except for the above modifications, we retain the other lines of Algorithm~\ref{alg:rstl-dtracking} as is. We call the modified policy $\textsc{Rstl-Dtrack-Eff}$ and use $\pi^{\textsc{Rstl-Dtrack-Eff}}$ as its shorthand representation.

\subsection{Theoretical Guarantees for {\sc Rstl-Dtrack-Eff}}

In this section, we provide the theoretical guarantees for the proposed {\sc Rstl-Dtrack-Eff} policy. To start, we note that the guarantees of Lemma~\ref{lem:sufficient-exploration-of-state-actions} and Lemma~\ref{lem:concentration-of-empirical-arm-means} hold as is for {\sc Rstl-Dtrack-Eff}. The key element in the proofs of the above results is an exploitation of the following property of $\pi_n$: for any state-action tuple $(\mathbf{d}, \mathbf{i}, a) \in \mathbb{S}_R \times [K]$,
$$
    \pi_n(a | \mathbf{d}, \mathbf{i}) \geq \varepsilon_n \, \pi^{\rm unif}(a | \mathbf{d}, \mathbf{i})\ ,
$$
which holds under {\sc Rstl-Dtrack-Eff}. Furthermore, recall that {\sc Rstl-Dtrack-Eff} picks an arbitrary $\nu_n^\star \in \mathcal{X}^\star(\widehat{\boldsymbol{\theta}}(n-1))$ at each time step $n$, and uses this to select arm $A_n \sim \pi_n(\cdot|\mathbf{d}(n), \mathbf{i}(n))$. In contrast, {\sc Rstl-Dtrack} uses $\mathcal{W}^\star(\widehat{\boldsymbol{\theta}}(n))$ to replace$\mathcal{X}^\star(\widehat{\boldsymbol{\theta}}(n))$. Hence, an analogue of Lemma~\ref{lem:concentration-of-state-action-visitations} holds for {\sc Rstl-Dtrack-Eff}, with $\mathcal{W}^\star(\boldsymbol{\theta})$ replaced with $\mathcal{X}^\star(\boldsymbol{\theta})$.

Next, we note that the guarantee in Proposition~\ref{prop:stop-in-finite-time-and-error-prob-less-than-delta} holds as is for {\sc Rstl-Dtrack}. Indeed, from Lemma~\ref{lemma:proxy-lower-bound-on-inner-infimum}, we note that $\chi\left(\frac{\mathbf{N}(n)}{n-K+1}, \widehat{\boldsymbol{\theta}}(n)\right) \leq \psi\left(\frac{\mathbf{N}(n)}{n-K+1}, \widehat{\boldsymbol{\theta}}(n)\right) = Z(n)$ for all $n \geq K$, and therefore
\begin{align}
    \mathbb{P}_{\boldsymbol{\theta}}\left(\exists n \geq K: \chi\left(\frac{\mathbf{N}(n)}{n-K+1}, \widehat{\boldsymbol{\theta}}(n)\right) \geq \zeta(n, \delta)\right) \leq \mathbb{P}_{\boldsymbol{\theta}}\left(\exists n \geq K: Z(n) \geq \zeta(n, \delta)\right)\ .
    \label{eq:err-prob-proof-goes-through-for-rstl-dtrack-eff}
\end{align}
The proof technique in Appendix~\ref{appndx:stop-in-finite-time-and-error-prob-less-than-delta}can be followed to demonstrate that the right-hand of \eqref{eq:err-prob-proof-goes-through-for-rstl-dtrack-eff} is upper bounded by $\delta$, thus establishing the counterpart of Proposition~\ref{prop:stop-in-finite-time-and-error-prob-less-than-delta} for {\sc Rstl-Dtrack-Eff}. It then remains to establish the analogues of Proposition~\ref{prop:almost-sure-upper-bound} and Proposition~\ref{prop:upper-bound-on-expected-stopping-time} for {\sc Rstl-Dtrack-Eff}.

Given any $\boldsymbol{\theta} \in \Theta^K$, define 
\begin{equation}
    U_R^\star(\boldsymbol{\theta}) \coloneqq \sup_{\nu \in \Sigma_R(\boldsymbol{\theta})} \chi(\nu, \boldsymbol{\theta})\ .
    \label{eq:U-R-star-definition}
\end{equation}
The next result presents the asymptotic growth rate of the stopping time of {\sc Rstl-Dtrack-Eff}.
\begin{proposition}
    \label{prop:almost-sure-and-expected-upper-bounds-for-Rstl-Dtrack-Eff}
    Fix $\eta \in (0,1)$.
    \begin{enumerate}
        \item For all $\delta \in (0,1)$, the stopping time $\tau$ of policy  $\pi^{\textsc{Rstl-Dtrack-Eff}}$ is finite almost surely, and hence $\pi^{\textsc{Rstl-Dtrack-Eff}} \in \Pi(\delta)$. Furthermore,
    \begin{equation}
        \P_{\boldsymbol{\theta}}\left( \limsup_{\delta \downarrow 0} \frac{\tau}{\log(1/\delta)} \leq \frac{1}{\eta\, T_{\text{\rm unif}}^\star(\boldsymbol{
        \theta
        }) + (1-\eta)\, U_R^\star(\boldsymbol{\theta})}\right)=1\ .
        \label{eq:almost-sure-upper-bound-Rstl-Dtrack-Eff}
    \end{equation}

    \item For all $\delta \in (0,1)$, the expected stopping time $\mathbb{E}_{\boldsymbol{\theta}}[\tau]$ of policy $\pi^{\textsc{Rstl-Dtrack-Eff}}$ is finite. Furthermore,
    \begin{equation}
        \limsup_{\delta \downarrow 0} \frac{\mathbb{E}_{\boldsymbol{\theta}}[\tau]}{\log(1/\delta)} \leq \frac{1}{\eta\, T_{\text{\rm unif}}^\star(\boldsymbol{
        \theta
        }) + (1-\eta)\, U_R^\star(\boldsymbol{\theta})}\ .
        \label{eq:upper-bound-on-expected-stopping-time-with-eta-Rstl-Dtrack-Eff}
    \end{equation}
    Consequently, in the limit of $\eta \downarrow 0$, we have
    \begin{equation}
        \limsup_{\eta\downarrow 0} \ \limsup_{\delta \downarrow 0} \frac{\mathbb{E}_{\boldsymbol{\theta}}[\tau]}{\log(1/\delta)} \leq \frac{1}{ U_R^\star(\boldsymbol{\theta})}\ .
        \label{eq:upper-bound-on-expected-stopping-time-Rstl-Dtrack-Eff}
    \end{equation}
    \end{enumerate}
\end{proposition}
Recall that the lower bound in \eqref{eq:lower-bound} is for {\em all} policies in $\Pi(\delta)$. Combining Proposition~\ref{prop:almost-sure-and-expected-upper-bounds-for-Rstl-Dtrack-Eff} with Proposition~\ref{prop:lower-bound}, we get
\begin{equation}
    \frac{1}{T_R^\star(\boldsymbol{\theta})} \leq \liminf_{\delta \downarrow 0} \, \inf_{\pi \in \Pi_R(\delta)} \frac{\mathbb{E}_{\boldsymbol{\theta}}[\tau_\pi]}{\log(1/\delta)} \leq \limsup_{\eta\downarrow 0}\ \limsup_{\delta \downarrow 0} \frac{\mathbb{E}_{\boldsymbol{\theta}}[\tau_{\pi^{\textsc{Rstl-Dtrack-Eff}}}]}{\log(1/\delta)} \leq \frac{1}{ U_R^\star(\boldsymbol{\theta})}\ .
    \label{eq:main-result-Rstl-Dtrack-Eff}
\end{equation}
The proof of Proposition~\ref{prop:almost-sure-and-expected-upper-bounds-for-Rstl-Dtrack-Eff} follows along the exact same lines as those of Propositions~\ref{prop:almost-sure-upper-bound}~and~\ref{prop:upper-bound-on-expected-stopping-time}, by replacing $\mathcal{W}^\star(\boldsymbol{\theta})$ with $\mathcal{X}^\star(\boldsymbol{\theta})$ and $T_R^\star(\boldsymbol{\theta})$ with $U_R^\star(\boldsymbol{\theta})$. The full proof is omitted for brevity.

Thus, while the computational feasibility of {\sc Rstl-Dtrack-Eff} is notable, it suffers from asymptotic sub-optimality, as evidenced by \eqref{eq:main-result-Rstl-Dtrack-Eff}. In contrast, {\sc Rstl-Dtrack} is asymptotically optimal (as evidenced by \eqref{eq:main-result}), albeit computationally infeasible. This underscores the fundamental trade-off between computational feasibility and asymptotic optimality, a crucial consideration in pure exploration problems such as BAI.

}
\section{\sc Concluding Remarks and Future Directions}
\label{sec:concluding-remarks-and-discussion}
In this paper, we have studied BAI in restless multi-armed bandits under the fixed-confidence regime, when the TPM of each arm belongs to a single-parameter exponential family of TPMs and the arm parameters are unknown. We have shown that the restless nature of the arms gives rise to the notion of arm delays and last observed states, the combination of which constitutes an MDP with a countable state space and a finite action space. By constraining the delay of each arm to be at most $R$ for some fixed, positive integer $R$, we have reduced the countable state space to a finite set, making the problem amenable to tractable analysis. Under the above $R$-max-delay constraint, we have obtained a problem instance-dependent lower bound on the limiting growth rate of the expected stopping time (time required to find the best arm) subject to an upper bound on the error probability, in the limit as the error probability vanishes. We have showed that the lower bound is characterized by the solution to a max-min optimization problem in which the outer `max' is over the set of all state-action occupancy measures satisfying (a) the $R$-max-delay constraint, and (b) a natural flow-conservation constraint. The inner `min' is over the set of alternative problem instances. We have devised a policy ({\sc Rstl-Dtrack}) for BAI, based on the idea of D-tracking \cite{garivier2016optimal}, that first estimates the unknown parameters of the arms, and then samples an arm at any given time according to a conditional probability distribution on the arms, conditioned on the values of arm delays and last observed states at that time. As for the stopping rule, we have devised a test statistic whose form is akin to that of the inner `min' expression of the lower bound, but with the true MDP state-action-state transition probabilities replaced with its empirical counterpart. In conjunction with a {\em random} threshold that is a function of the desired error probability, we have designed a rule for stopping further selection of arms whenever the test statistic exceeds the threshold. We have shown that our policy stops in finite time almost surely, satisfies the desired error probability, and is asymptotically optimal.

The computational complexity of the \textsc{Rstl-Dtrack} policy is a notable concern, particularly regarding the computation of the infimum in \eqref{eq:Z-of-n} at each time step, which can be quite resource-intensive. Due to the presence of arm delays, simplifying this infimum any further is a formidable challenge, as discussed in Section~\ref{sec:achievability}. One potential approach to alleviating this computational burden is to adopt a technique proposed in \cite{mukherjee2023best}. Their method involves expressing the inner infimum in the lower bound using ``projection measures,'' which is computationally more tractable, especially for single-parameter exponential families. However, unlike in \cite{mukherjee2023best}, the projection measures in our specific context will depend on $\nu$, the variable of optimization in the outer `sup' expression of the lower bound; this in turn may be attributed to the arm delays in our setting. Resolving this challenge remains an open issue and an intriguing avenue for further research. Additionally, recent studies, such as \cite{lee2023thompson}, have demonstrated the promise of Thomson sampling-based policies in reducing computational complexity of BAI. It could be worthwhile to explore the extensions to restless settings, which might offer further computational efficiencies and improve the performance of our policy.

\textbf{Future directions:}
While we keep $R$ fixed throughout the paper, it is interesting to note that $T_R^\star(\boldsymbol{\theta})$, the constant appearing in the lower bound, is monotone increasing in $R$ and therefore admits a limit as $R \to \infty$; see Remark~\ref{rem:lim-T-R-star}. It is natural to expect that $\lim_{R \to \infty} T_R^\star(\boldsymbol{\theta})=T^\star(\boldsymbol{\theta})$, where $T^\star(\boldsymbol{\theta})$ is the constant governing the lower bound without the maximum delay constraint. A cursory examination of the analysis in \cite[Section XI]{karthik2022best} reveals that the above relation indeed holds in the special case when the observations from each arm are i.i.d. However, in the general setting of restless arms, it is unclear whether the above relation holds, and a formal justification of this could be an interesting future direction. While it is natural to expect that $T_R^\star(\boldsymbol{\theta})$ ought to depend on the {\em mixing times} of the arms, our analysis does not bring out this dependence explicitly. Considering a simple $2$-armed restless bandit problem with $\mathcal{S}=\{0,1\}$ in which one arm yields i.i.d. observations according to $\text{\rm Ber}(1/2)$ distribution while the other arm is a slowly mixing Markov process, characterizing $T_R^\star(\boldsymbol{\theta})$ explicitly in terms of the mixing time of the second arm could be an interesting direction to explore. Furthermore, considering a dataset of offline observations from arms with inherent delays, an investigation into how incorporating this offline data affects the overall sample complexity of BAI along the lines of \cite{agrawal2023optimal} would be insightful. Finally, we note that the extensions to {\em hidden Markov} observations from the arms, wherein at each time $n$, the learner observes $\bar{Y}_n=u(\bar{X}_n)$ for some known/unknown function $u$, may be of interest. Here, the technical key challenge is that while successive observations $\bar{X}_{n}$ and $\bar{X}_{n+1}$ from any given arm possess a Markov dependence, the same may not be said about $\bar{Y}_{n}$ and $\bar{Y}_{n+1}$. While \cite{moulos2019optimal} considers hidden Markov observations in rested bandits, the lower bound therein does not capture the ``hidden'' aspect of the observations. It may therefore be worthwhile to first establish a lower bound for BAI with hidden Markov observations in rested bandits and subsequently undertake a formal study of restless hidden Markov bandits.

\section*{Acknowledgements}

The primary author, P.~N.~Karthik, wishes to extend deep gratitude to Prof.\ Shie Mannor (Technion Israel Institute of Technology, Haifa, Israel), Aymen Al Marjani (ENS Lyon, France), and Dr.\ Karthikeyan Shanmugam (Google Research India, Bengaluru, India) for the invaluable and enlightening discussions. The author also wishes to express sincere gratitude to Dr. Vrettos Moulos (Google Research, New York) for generously sharing some portions of unpublished research from his doctoral studies and for engaging in extensive discussions. A portion of this research was conducted during the author's tenure as a Visiting Researcher at the Technion.

The work of Arpan Mukherjee and Ali Tajer was supported in part by the RPI-IBM Artificial Intelligence Research Collaboration and in part by the U.S. National Science Foundation award ECCS-193310.
\bibliographystyle{IEEEtran}
\bibliography{references.bib}

\newpage
\appendix

\section{Notations}
{
\renewcommand{\arraystretch}{1.5}
\begin{table*}[!ht]
    \centering
    \textcolor{black}{
    \begin{tabular}{|r|l|}
        \hline 
        $K$ & Number of arms \\
        \hline
        $[K]$ & Shorthand for $\{1, \ldots, K\}$, the set of arms \\
        \hline
        $\mathcal{S}$ & State space of each arm \\
        \hline
        $P$ & Generator matrix defining the parametric family of TPMs\\
        \hline
        $f: \mathcal{S} \to \mathbb{R}$ & Reward function defining the parametric family of TPMs \\
        \hline
        $\Theta$ & Space of parameters \\
        \hline
        $\theta$ & Generic parameter belonging to $\Theta$ \\
        \hline
        $P_\theta$ & TPM associated with parameter $\theta \in \Theta$ (defined in \eqref{eq:P-theta}) \\
        \hline
        $P_\theta^d$, ~$d \in \mathbb{N}$ & $d$-fold product of $P_\theta$ with itself \\
        \hline 
        $\theta_a$ & Parameter of arm $a \in [K]$ \\
        \hline
        $\eta_{\theta_a}$ & Stationary mean of arm $a$ (defined in \eqref{eq:nu_a}) \\
        \hline
        $\boldsymbol{\theta}$ & Shorthand for $[\theta_a: a \in [K]]^\top \in \Theta^K$, a problem instance \\
        \hline
        $a^\star(\boldsymbol{\theta})$ & Best arm under the instance $\boldsymbol{\theta} \in \Theta^K$ (defined in \eqref{eq:best_arm}) \\
        \hline
        $\textsc{Alt}(\boldsymbol{\theta})$ & Set of instances alternative to $\boldsymbol{\theta}$ (defined in \eqref{eq:alt-theta}) \\
        \hline
        $\tilde{P}_\theta$ & Exponential tilt of generator $P$ (defined in \eqref{eq:P-tilde-theta}) \\
        \hline
        $\rho(\theta)$, ~$\theta \in \Theta$ & Perron--Frobenius eigenvalue of $\tilde{P}_\theta$ \\
        \hline
        $A: \Theta \to \mathbb{R}$ & Shorthand for $A(\theta) = \log \rho(\theta)$, ~$\theta \in \Theta$ \\
        \hline 
        $\Dot{A}: \Theta \to \mathbb{R}$ & Derivative of $A$ \\
        \hline
        $\Dot{A}^{-1}: \mathbb{R} \to \Theta$ & Inverse of $\Dot{A}$ \\
        \hline
        $A_n \in [K]$ & Arm selected at time $n$ \\
        \hline
        $\bar{X}_n \in \mathcal{S}$ & State of arm $A_n$ \\
        \hline
        $\mathbf{d}$ & Generic vector of arm delays $[d_a: a \in [K]]^\top \in \mathbb{N}^K$ \\
        \hline
        $\mathbf{i}$ & Generic vector of last observed states of the arms $[i_a: a \in [K]]^\top \in \mathcal{S}^K$ \\
        \hline
        $\mathbf{d}(n)$ & Vector of arm delays at time $n$ \\
        \hline
        $\mathbf{i}(n)$ & Vector of last observed states of the arms at time $n$ \\
        \hline
        %$Q_{\boldsymbol{\theta}}$ & Transition matrix of the countable-state MDP of arm delays and last observed states \\
        %\hline
        $R$ & Maximum delay parameter \\
        \hline
        %$Q_{\boldsymbol{\theta}, R}$ & Transition matrix of the finite-state MDP
        $\pi^{\rm unif}$ & Uniform arm selection policy (defined in \eqref{eq:uniform-arm-selection-rule}) \\
        \hline
        %$T_R^\star(\boldsymbol{\theta})$ & Constant governing the lower bound under the instance $\boldsymbol{\theta}$ \\
        %\hline
        $\widehat{\eta}^a(n) \in \mathbb{R}$ & Estimate of mean of arm $a$ at time $n$ (defined in \eqref{eq:parameter-estimates}) \\
        \hline
        $\boldsymbol{\eta}(n) \in \mathbb{R}^K$ & Shorthand for the vector $[\widehat{\eta}^a(n): a \in [K]]^\top$\\
        \hline
        $\widehat{\theta}_a(n) \in \Theta$ & Shorthand for $\Dot{A}^{-1}(\widehat{\eta}^a(n))$ \\
        \hline
        $\widehat{\boldsymbol{\theta}}(n) \in \Theta^K$ & Shorthand for the vector $[\widehat{\theta}_a(n): a \in [K]]^\top$ \\
        \hline
        $\eta \in (0,1)$ & Mixture parameter (appearing in \eqref{eq:lambda-n-definition})\\
        \hline
        $\boldsymbol{\eta} \in \mathbb{R}^K$ & Shorthand for the vector $[\eta_{\theta_a}: a \in [K]]^\top$ \\
        \hline
    \end{tabular}}
    \vspace{0.3cm}
    \caption{\textcolor{black}{List of important notations appearing in the paper.}}
    \label{tab:my_label}
\end{table*}
}

% ----------------------------------------------------------------------------------------------

\section{Proof of Lemma~\ref{lem:MDP-is-communicating}}
Fix $\boldsymbol{\theta} \in \Theta^K$ and $(\mathbf{d}, \mathbf{i}), (\mathbf{d}', \mathbf{i}') \in \mathbb{S}$. Let
\begin{equation}
    M \coloneqq \min\{d \geq 1: P^d(j \mid i) > 0 \quad \forall i, j \in \mathcal{S}\},
    \label{eq:M-definition}
\end{equation}
where $P$ is the generator of the single-parameter exponential family of TPMs defined in \eqref{eq:P-theta}. From \cite[Proposition 1.7]{levin2017markov}, we know that $M < +\infty$. Fix $n \geq K$, and suppose that $(\mathbf{d}(n), \mathbf{i}(n)) = (\mathbf{d}, \mathbf{i})$. Let $\pi$ denote the uniform arm selection policy that selects the arms uniformly at random at each time instant. We now demonstrate that there exists an integer $N$ (possibly depending on $(\mathbf{d},\mathbf{i})$ and $(\mathbf{d}',\mathbf{i}')$) such that \eqref{eq:communicating-MDP-definition} holds under $\pi$.
Without loss of generality, let the components of $\mathbf{d}'$ satisfy the ordering $d_1' > d_2' > \cdots > d_K^\prime=1$.
Order the components of $\mathbf{d}$ in decreasing order. 
Consider the following sequence of arm selections and observations: for a total of $M$ time instants, from $n$ to $n+M-1$, select the arms in a round-robin fashion in the decreasing order of their component $\mathbf{d}$ values. At time $n+M$, select arm $1$ and observe state $i_{1}'$ on it. Thereafter, select arms $2, \ldots, K$ in a round-robin fashion in the decreasing order of their component $\mathbf{d}$ values until time $n+M+d_{1}'-d_{2}'-1$. At time $n+M+d_{1}'-d_{2}'$, select arm $2$ and observe state $i_{2}'$ on it. Continue the round-robin sampling on arms $3, \ldots, K$ till time $n+M+d_{1}'-d_{3}'-1$. At time $n+M+d_{1}'-d_{3}'$, select arm $3$ and observe state $i_{3}'$ on it. Continue this process till arm $K$ is selected at time $n+M+d_{1}'-1$ and the state $i_{K}'$ is observed on it.

Clearly, the above sequence of arm selections and observations results in $(\mathbf{d}(n+N), \mathbf{i}(n+N))=(\mathbf{d}',\mathbf{i}')$ for $N = M+d_1'$. Furthermore, using the preceding value for $N$, the probability in \eqref{eq:communicating-MDP-definition} may be lower bounded by
\begin{align}
	&\bigg(\frac{1}{K}\bigg)^{N} \cdot  \left[\prod_{a=1}^{K}P_{\theta_a}^{M+d_{a}+d_1'-d_a'}(i_a'|i_a)\right] > 0,
	\label{eq:irreducibility-1}
\end{align}
where in \eqref{eq:irreducibility-1}, $(1/K)^N$ is from the uniform selection of arms, and the strict positivity of the term within square braces follows by noting from \eqref{eq:P-theta} that for each $\theta \in \Theta$,
$$
P^d(j \mid i)>0 \implies P_{\theta}^d(j \mid i) > 0 \quad \forall d \geq M, \ i, j \in \mathcal{S}.
$$
This completes the proof.

%-------------------------------------------------------------------

\section{Proof of Lemma~\ref{lem:ergodicity-of-MDP-under-unif-policy}}
From \cite[Lemma 1]{karthik2022best}, we know that  $Q_{\boldsymbol{\theta}, \pi^{\text{\rm unif}}}=\{Q_{\boldsymbol{\theta}, \pi^{\text{\rm unif}}}(\mathbf{d}', \mathbf{i}'|\mathbf{d}, \mathbf{i}): (\mathbf{d}, \mathbf{i}), (\mathbf{d}', \mathbf{i}') \in \mathbb{S}_R\}$ is ergodic (albeit with one modification: the quantity $M$ appearing in the proof of \cite[Lemma 1]{karthik2022best} must be replaced with $M$ as defined in \eqref{eq:M-definition}). Using the preceding result together with the fact that $\pi^{\text{\rm unif}}(a|\mathbf{d}, \mathbf{i})>0$ for all {\em valid} $(\mathbf{d}, \mathbf{i}, a)$ tuples proves the lemma.

%-------------------------------------------------------------------

\section{Proof of Lemma \ref{lem:flow-conservation-property}}

    First, we note that
    \begin{align}
        N(n, \mathbf{d}', \mathbf{i}') 
        &= \sum_{t=K}^{n} \mathbf{1}_{\{\mathbf{d}(t)=\mathbf{d}', \mathbf{i}(t)=\mathbf{i}'\}} \nonumber\\
        &= \mathbf{1}_{\{\mathbf{d}(K)=\mathbf{d}', \mathbf{i}(K)=\mathbf{i}'\}} + \sum_{(\mathbf{d}, \mathbf{i}) \in \mathbb{S}_R} \ \sum_{a=1}^{K} \ \sum_{t=K+1}^{n} \mathbf{1}_{\{\mathbf{d}(t-1)=\mathbf{d}, \mathbf{i}(t-1)=\mathbf{i}, A_{t-1}=a, \mathbf{d}(t)=\mathbf{d}', \mathbf{i}(t)=\mathbf{i}'\}}
        \nonumber\\
        &= \mathbf{1}_{\{\mathbf{d}(K)=\mathbf{d}', \mathbf{i}(K)=\mathbf{i}'\}} + \sum_{(\mathbf{d}, \mathbf{i}) \in \mathbb{S}_R} \ \sum_{a=1}^{K} \ \sum_{u=1}^{N(n-1, \mathbf{d}, \mathbf{i}, a)} \mathbf{1}_{\{W_u(\mathbf{d}, \mathbf{i}, a) = (\mathbf{d}', \mathbf{i}')\}},
        \label{eq:proof-of-flow-conservation-1}
    \end{align}
    where $W_u(\mathbf{d}, \mathbf{i}, a)$ denotes the next state of the MDP when, for the $u$-th time, the state $(\mathbf{d}, \mathbf{i})$ appeared and arm $a$ was chosen subsequently. We then note that for all $(\mathbf{d}, \mathbf{i}, a) \in \mathbb{S}_R \times [K]$, the events $\{N(n-1, \mathbf{d}, \mathbf{i}, a) \geq u\}$ and $\{W_u(\mathbf{d}, \mathbf{i}, a) = (\mathbf{d}', \mathbf{i}')\}$ are independent of one another. Indeed, let $\tau_u(\mathbf{d},  \mathbf{i}, a)$ denote the instant at which, for the $u$-th time, the state $(\mathbf{d}, \mathbf{i})$ appeared and arm $a$ was chosen subsequently. Then, noting that $\tau_u(\mathbf{d},  \mathbf{i}, a) \leq n-1$ almost surely under the event $\{N(n-1, \mathbf{d}, \mathbf{i}, a) \geq u\}$, we have
    \begin{align}
        &\P_{\boldsymbol{\theta}}(W_u(\mathbf{d}, \mathbf{i}, a)=(\mathbf{d}', \mathbf{i}') \mid N(n-1, \mathbf{d}, \mathbf{i}, a) \geq u) \nonumber\\
        &= \sum_{t=K}^{n-1} \P_{\boldsymbol{\theta}}(W_u(\mathbf{d}, \mathbf{i}, a)=(\mathbf{d}', \mathbf{i}'), \tau_u(\mathbf{d}, \  \mathbf{i}, a)=t \mid N(n-1, \mathbf{d}, \mathbf{i}, a) \geq u) \nonumber\\
        &= \sum_{t=K}^{n-1} \P_{\boldsymbol{\theta}}(W_u(\mathbf{d}, \mathbf{i}, a)=(\mathbf{d}', \mathbf{i}') \mid  \tau_u(\mathbf{d}, \  \mathbf{i}, a)=t) \cdot \P_{\boldsymbol{\theta}}(\tau_u(\mathbf{d}, \  \mathbf{i}, a)=t \mid N(n-1, \mathbf{d}, \mathbf{i}, a) \geq u) \nonumber\\
        &= \sum_{t=K}^{n-1} \P_{\boldsymbol{\theta}}((\mathbf{d}(t+1), \mathbf{i}(t+1))=(\mathbf{d}', \mathbf{i}') \mid  (\mathbf{d}(n), \mathbf{i}(t), A_t)=(\mathbf{d}, \mathbf{i}, a)) \cdot \P_{\boldsymbol{\theta}}(\tau_u(\mathbf{d}, \  \mathbf{i}, a)=t \mid N(n-1, \mathbf{d}, \mathbf{i}, a) \geq u) \nonumber\\ 
        &= \sum_{t=K}^{n-1} \P_{\boldsymbol{\theta}}((\mathbf{d}(t+1), \mathbf{i}(t+1))=(\mathbf{d}', \mathbf{i}') \mid  (\mathbf{d}(n), \mathbf{i}(t), A_t)=(\mathbf{d}, \mathbf{i}, a)) \cdot \P_{\boldsymbol{\theta}}(\tau_u(\mathbf{d}, \  \mathbf{i}, a)=t \mid N(n-1, \mathbf{d}, \mathbf{i}, a) \geq u) \nonumber\\
        &= \sum_{t=K}^{n-1} Q_{\boldsymbol{\theta}, R}(\mathbf{d}', \mathbf{i}'|\mathbf{d}, \mathbf{i}, a) \cdot \P_{\boldsymbol{\theta}}(\tau_u(\mathbf{d}, \  \mathbf{i}, a)=t \mid N(n-1, \mathbf{d}, \mathbf{i}, a) \geq u) \nonumber\\
        &= Q_{\boldsymbol{\theta}, R}(\mathbf{d}', \mathbf{i}'|\mathbf{d}, \mathbf{i}, a),
        \label{eq:proof-of-flow-conservation-2}
    \end{align}
    where the penultimate line follows from \eqref{eq:MDP_transition_probabilities}. On the other hand, it is straightforward to see that $\P_{\boldsymbol{\theta}}(W_u(\mathbf{d}, \mathbf{i}, a)=(\mathbf{d}', \mathbf{i}'))=Q_{\boldsymbol{\theta}, R}(\mathbf{d}', \mathbf{i}'|\mathbf{d}, \mathbf{i}, a)$, which is identical to the right-hand side of \eqref{eq:proof-of-flow-conservation-2}. 
    
    Writing \eqref{eq:proof-of-flow-conservation-1} as
    \begin{align}
        N(n, \mathbf{d}', \mathbf{i}') = \mathbf{1}_{\{\mathbf{d}(K)=\mathbf{d}', \mathbf{i}(K)=\mathbf{i}'\}} + \sum_{(\mathbf{d}, \mathbf{i}) \in \mathbb{S}_R} \ \sum_{a=1}^{K} \ \sum_{u=1}^{\infty} \mathbf{1}_{\{W_u(\mathbf{d}, \mathbf{i}, a) = (\mathbf{d}', \mathbf{i}')\}} \, \mathbf{1}_{\{N(n-1, \mathbf{d}, \mathbf{i}, a) \geq u\}},
        \label{eq:proof-of-flow-conservation-3}
    \end{align}
    taking expectations on both sides of \eqref{eq:proof-of-flow-conservation-3}, and using the monotone convergence theorem, we get
    \begin{align}
        \mathbb{E}_{\boldsymbol{\theta}}[N(n, \mathbf{d}', \mathbf{i}')] 
        &= \P_{\boldsymbol{\theta}}(\mathbf{d}(K)=\mathbf{d}', \mathbf{i}(K)=\mathbf{i}') + \sum_{(\mathbf{d}, \mathbf{i}) \in \mathbb{S}_R} \ \sum_{a=1}^{K} \ \sum_{u=1}^{\infty} Q_{\boldsymbol{\theta}, R}(\mathbf{d}', \mathbf{i}'|\mathbf{d}, \mathbf{i}, a) \, \P_{\boldsymbol{\theta}}(N(n-1, \mathbf{d}, \mathbf{i}, a) \geq u) \nonumber\\
        &= \P_{\boldsymbol{\theta}}(\mathbf{d}(K)=\mathbf{d}', \mathbf{i}(K)=\mathbf{i}') + \sum_{(\mathbf{d}, \mathbf{i}) \in \mathbb{S}_R} \ \sum_{a=1}^{K} \ Q_{\boldsymbol{\theta}, R}(\mathbf{d}', \mathbf{i}'|\mathbf{d}, \mathbf{i}, a) \, \mathbb{E}_{\boldsymbol{\theta}}[N(n-1, \mathbf{d}, \mathbf{i}, a)].
        \label{eq:proof-of-flow-conservation-4}
    \end{align}
    The desired follows from \eqref{eq:proof-of-flow-conservation-4} by simply noting that $\P_{\boldsymbol{\theta}}(\mathbf{d}(K)=\mathbf{d}', \mathbf{i}(K)=\mathbf{i}') \leq 1$ and 
    $$
    \mathbb{E}_{\boldsymbol{\theta}}[N(n-1, \mathbf{d}, \mathbf{i}, a)] \leq \mathbb{E}_{\boldsymbol{\theta}}[N(n, \mathbf{d}, \mathbf{i}, a)] \leq \mathbb{E}_{\boldsymbol{\theta}}[N(n-1, \mathbf{d}, \mathbf{i}, a)] + 1
    $$
    for all $(\mathbf{d}, \mathbf{i}, a) \in \mathbb{S}_R \times [K]$.

%-------------------------------------------------------------------

\section{Proof of Proposition~\ref{prop:lower-bound}}
We do not provide all the detailed steps here, as most of them follow straightforwardly from \cite{karthik2022best}.
Let $\boldsymbol{\theta} \in \Theta^K$ be fixed. Recall the assumption that each of the arms is selected once at the beginning, from time $n=0$ to $n=K-1$. For all $n \geq K$ and $\boldsymbol{\lambda} \in \Theta^K$, let $Z_{\boldsymbol{\theta}, \boldsymbol{\lambda}}(n)$ denote the log-likelihood ratio (LLR) of all the arm selections and observations from the arms up to time $n$ under the instance $\boldsymbol{\theta}$ versus that under the instance $\boldsymbol{\lambda}$. That is,
\begin{equation}
    Z_{\boldsymbol{\theta}, \boldsymbol{\lambda}}(n) \coloneqq \log \frac{\P_{\boldsymbol{\theta}}(A_{0:n}, \bar{X}_{0:n})}{P_{\boldsymbol{\lambda}}(A_{0:n}, \bar{X}_{0:n})}, \quad n \geq K.
\end{equation}
It can be easily shown that
\begin{align}
    Z_{\boldsymbol{\theta}, \boldsymbol{\lambda}}(n)=\sum_{a=1}^{K} \log \frac{\P_{\boldsymbol{\theta}}(X_{a-1}^{a})}{P_{\boldsymbol{\lambda}}(X_{a-1}^{a})} + \sum_{(\mathbf{d}, \mathbf{i})\in \mathbb{S}_R}\ \sum_{a=1}^{K}\ \sum_{j\in \mathcal{S}}\ N(n, \mathbf{d}, \mathbf{i}, a, j) \ \log \frac{P_{\theta_a}^{d_{a}}(j \mid i_{a})}{P_{\lambda_a}^{d_{a}}(j \mid i_{a})}.
    \label{eq:LLR-final}
\end{align}
In \eqref{eq:LLR-final}, $X_{a-1}^{a}$ denotes the sample observed from arm $a$ when it is selected for the first time at $n=a-1$, and $N(n, \mathbf{d}, \mathbf{i}, a, j)$ denotes the number of times up to time $n$ the state $(\mathbf{d}, \mathbf{i})$ appeared, arm $a$ was selected subsequently, and state $j \in \mathcal{S}$ was observed on arm $a$. Noting that the the action-observation pair $(a,j)$ together with $(\mathbf{d}, \mathbf{i})$ defines the next state $(\mathbf{d}', \mathbf{i}')$ uniquely, we write
\begin{equation}
    \sum_{j\in \mathcal{S}}\ N(n, \mathbf{d}, \mathbf{i}, a, j) \ \log \frac{P_{\theta_a}^{d_{a}}(j \mid i_{a})}{P_{\lambda_a}^{d_{a}}(j \mid i_{a})} = \sum_{(\mathbf{d}', \mathbf{i}') \in \mathbb{S}_R} N(n, \mathbf{d}, \mathbf{i}, a, \mathbf{d}', \mathbf{i}') \ \log \frac{Q_{\boldsymbol{\theta}, R}(\mathbf{d}', \mathbf{i}'|\mathbf{d}, \mathbf{i}, a)}{Q_{\boldsymbol{\lambda}, R}(\mathbf{d}', \mathbf{i}'|\mathbf{d}, \mathbf{i}, a)}, 
\end{equation}
where $N(n, \mathbf{d}, \mathbf{i}, a, \mathbf{d}', \mathbf{i}')$ denotes the number of times up to time $n$ the state $(\mathbf{d}, \mathbf{i})$ appeared, arm $a$ was chosen subsequently, and the state $(\mathbf{d}', \mathbf{i}')$ resulted. We also note that
\begin{equation}
    \mathbb{E}_{\boldsymbol{\theta}}[N(n, \mathbf{d}, \mathbf{i}, a, \mathbf{d}', \mathbf{i}')] = \mathbb{E}_{\boldsymbol{\theta}}[N(n, \mathbf{d}, \mathbf{i}, a)] \cdot Q_{\boldsymbol{\theta}, R}(\mathbf{d}', \mathbf{i}'|\mathbf{d}, \mathbf{i}, a) \quad \forall n \geq K.
    \label{eq:state-action-state-and-state-action-relation}
\end{equation}

Fix $\delta \in (0,1)$ and a policy $\pi \in \Pi_R(\delta)$; here, $\Pi_R(\delta)$ is as defined in \eqref{eq:Pi(epsilon)}. Assume that $\mathbb{E}_{\boldsymbol{\theta}}[\tau_\pi] < +\infty $\footnote{If this condition is not satisfied, then $\mathbb{E}_{\boldsymbol{\theta}}[\tau_\pi]=+\infty$ and the lower bound \eqref{eq:lower-bound} holds trivially.}. Using standard change-of-measure arguments for restless bandits (cf. \cite[Appendix A-A]{karthik2022best}, it can be shown that
\begin{equation}
    \inf_{\boldsymbol{\lambda} \in \textsc{Alt}(\boldsymbol{\theta})} \mathbb{E}_{\boldsymbol{\theta}}[Z_{\boldsymbol{\theta}, \boldsymbol{\lambda}}(n)] \geq d(\delta, 1-\delta) \quad \forall n \geq K,
   \label{eq:change-of-measure}
\end{equation}
where $d(\delta, 1-\delta)$ denotes the KL divergence between the distributions $\text{Bernoulli}(\delta)$ and $\text{Bernoulli}(1-\delta)$. We note here that \eqref{eq:flow-conservation-property}, \eqref{eq:R-max-delay-in-terms-of-state-action-visitations}, \eqref{eq:state-action-state-and-state-action-relation}, and \eqref{eq:change-of-measure} hold when $n$ is replaced by $\tau_\pi$ with $\mathbb{E}_{\boldsymbol{\theta}}[\tau_\pi] < +\infty$. We then have
\begin{align}
    & \inf_{\boldsymbol{\lambda} \in \textsc{Alt}(\boldsymbol{\theta})} \mathbb{E}_{\boldsymbol{\theta}}[Z_{\boldsymbol{\theta}, \boldsymbol{\lambda}}(\tau_\pi)] \nonumber\\
    &= \inf_{\boldsymbol{\lambda} \in \textsc{Alt}(\boldsymbol{\theta})} \left[ \sum_{a=1}^{K} \mathbb{E}_{\boldsymbol{\theta}} \left[ \log \frac{\P_{\boldsymbol{\theta}}(X_{a-1}^{a})}{P_{\boldsymbol{\lambda}}(X_{a-1}^{a})} \right] + \sum_{(\mathbf{d}, \mathbf{i})\in \mathbb{S}_R}\ \sum_{a=1}^{K}\ \sum_{(\mathbf{d}', \mathbf{i}') \in \mathbb{S}_R} \mathbb{E}_{\boldsymbol{\theta}}[N(\tau_\pi, \mathbf{d}, \mathbf{i}, a, \mathbf{d}', \mathbf{i}')] \ \log \frac{Q_{\boldsymbol{\theta}, R}(\mathbf{d}', \mathbf{i}'|\mathbf{d}, \mathbf{i}, a)}{Q_{\boldsymbol{\lambda}, R}(\mathbf{d}', \mathbf{i}'|\mathbf{d}, \mathbf{i}, a)} \right] \nonumber\\
    &\stackrel{(a)}{=} \inf_{\boldsymbol{\lambda} \in \textsc{Alt}(\boldsymbol{\theta})} \left[ \sum_{a=1}^{K} \mathbb{E}_{\boldsymbol{\theta}} \left[ \log \frac{\P_{\boldsymbol{\theta}}(X_{a-1}^{a})}{P_{\boldsymbol{\lambda}}(X_{a-1}^{a})} \right] + \sum_{(\mathbf{d}, \mathbf{i})\in \mathbb{S}_R}\ \sum_{a=1}^{K} \mathbb{E}_{\boldsymbol{\theta}}[N(\tau_\pi, \mathbf{d}, \mathbf{i}, a)] \ D_{\text{\rm KL}}(Q_{\boldsymbol{\theta}, R}(\cdot \mid\mathbf{d}, \mathbf{i}, a) \| Q_{\boldsymbol{\lambda}, R}(\cdot \mid\mathbf{d}, \mathbf{i}, a)) \right] \nonumber\\
    &= \inf_{\boldsymbol{\lambda} \in \textsc{Alt}(\boldsymbol{\theta})} \bigg[ \sum_{a=1}^{K} \mathbb{E}_{\boldsymbol{\theta}} \left[ \log \frac{\P_{\boldsymbol{\theta}}(X_{a-1}^{a})}{P_{\boldsymbol{\lambda}}(X_{a-1}^{a})} \right] \nonumber\\
    &\hspace{4cm} +  \sum_{(\mathbf{d}, \mathbf{i})\in \mathbb{S}_R}\ \sum_{a=1}^{K} \mathbb{E}_{\boldsymbol{\theta}} [N(\tau_\pi, \mathbf{d}, \mathbf{i}, a)] \ D_{\text{\rm KL}}(Q_{\boldsymbol{\theta}, R}(\cdot \mid\mathbf{d}, \mathbf{i}, a) \| Q_{\boldsymbol{\lambda}, R}(\cdot \mid\mathbf{d}, \mathbf{i}, a)) \bigg] \nonumber\\
    &= \inf_{\boldsymbol{\lambda} \in \textsc{Alt}(\boldsymbol{\theta})} \Bigg[ \sum_{a=1}^{K} \mathbb{E}_{\boldsymbol{\theta}} \left[ \log \frac{\P_{\boldsymbol{\theta}}(X_{a-1}^{a})}{P_{\boldsymbol{\lambda}}(X_{a-1}^{a})} \right] \nonumber\\
    &\hspace{1cm} + (\mathbb{E}_{\boldsymbol{\theta}}[\tau_\pi-K+1])\,  \sum_{(\mathbf{d}, \mathbf{i})\in \mathbb{S}_R}\ \sum_{a=1}^{K} \frac{\mathbb{E}_{\boldsymbol{\theta}}[N(\tau_\pi, \mathbf{d}, \mathbf{i}, a)]}{\mathbb{E}_{\boldsymbol{\theta}}[\tau_\pi-K+1]} \ D_{\text{\rm KL}}(Q_{\boldsymbol{\theta}, R}(\cdot \mid\mathbf{d}, \mathbf{i}, a) \| Q_{\boldsymbol{\lambda}, R}(\cdot \mid\mathbf{d}, \mathbf{i}, a)) \Bigg] \nonumber\\
    &\leq \sup_{\nu \in \Sigma_R(\boldsymbol{\theta})} \ \inf_{\boldsymbol{\lambda} \in \textsc{Alt}(\boldsymbol{\theta})} \Bigg[ \sum_{a=1}^{K} \mathbb{E}_{\boldsymbol{\theta}} \left[ \log \frac{\P_{\boldsymbol{\theta}}(X_{a-1}^{a})}{P_{\boldsymbol{\lambda}}(X_{a-1}^{a})} \right] \nonumber\\
    &\hspace{3cm} + (\mathbb{E}_{\boldsymbol{\theta}}[\tau_\pi-K+1])\,  \sum_{(\mathbf{d}, \mathbf{i})\in \mathbb{S}_R}\ \sum_{a=1}^{K} \nu(\mathbf{d}, \mathbf{i}, a) \ D_{\text{\rm KL}}(Q_{\boldsymbol{\theta}, R}(\cdot \mid\mathbf{d}, \mathbf{i}, a) \| Q_{\boldsymbol{\lambda}, R}(\cdot \mid\mathbf{d}, \mathbf{i}, a)) \Bigg],
    \label{eq:proof-of-lower-bound-1}
\end{align}
where $(a)$ above follows from applying \eqref{eq:state-action-state-and-state-action-relation} to $\tau_\pi$, and the last line follows by noting that 
$$
\left\lbrace \frac{\mathbb{E}_{\boldsymbol{\theta}}[N(\tau_\pi, \mathbf{d}, \mathbf{i}, a)]}{\mathbb{E}_{\boldsymbol{\theta}}[\tau_\pi-K+1]} \right\rbrace_{(\mathbf{d}, \mathbf{i}, a)}
$$
is a probability distribution on $\mathbb{S}_R \times [K]$, satisfies the flow conservation property \eqref{eq:flow-conservation-property}, the $R$-max-delay constraint \eqref{eq:R-max-delay-in-terms-of-state-action-visitations}, and therefore an element of $\Sigma_R(\boldsymbol{\theta})$. The above fractional term may hence be upper bounded by the supremum over all elements of $\Sigma_R(\boldsymbol{\theta})$, thereby leading to \eqref{eq:proof-of-lower-bound-1}. We thus have
\begin{align}
    \frac{d(\delta, 1-\delta)}{\log(1/\delta)} 
    &\leq \sup_{\nu \in \Sigma_R(\boldsymbol{\theta})} \ \inf_{\boldsymbol{\lambda} \in \textsc{Alt}(\boldsymbol{\theta})} \Bigg[ \sum_{a=1}^{K} \mathbb{E}_{\boldsymbol{\theta}} \left[ \log \frac{\P_{\boldsymbol{\theta}}(X_{a-1}^{a})}{P_{\boldsymbol{\lambda}}(X_{a-1}^{a})} \right] \nonumber\\
    &\hspace{2cm} + (\mathbb{E}_{\boldsymbol{\theta}}[\tau_\pi-K+1])\,  \sum_{(\mathbf{d}, \mathbf{i})\in \mathbb{S}_R}\ \sum_{a=1}^{K} \nu(\mathbf{d}, \mathbf{i}, a) \ D_{\text{\rm KL}}(Q_{\boldsymbol{\theta}, R}(\cdot \mid\mathbf{d}, \mathbf{i}, a) \| Q_{\boldsymbol{\lambda}, R}(\cdot \mid\mathbf{d}, \mathbf{i}, a)) \Bigg]
    \label{eq:proof-of-lower-bound-2}
\end{align}
for all $\pi \in \Pi_R(\delta)$. Noting that the first term on the right hand side of \eqref{eq:proof-of-lower-bound-2} is not a function $\delta$, and using the fact that $d(\delta, 1-\delta) / \log(1/\delta) \to 1$ as $\delta \downarrow 0$, we arrive at \eqref{eq:lower-bound}.
%-------------------------------------------------------------------

\section{Proof of Lemma \ref{lem:attainment-of-supremum}}
\label{appndx:proof-of-attainment-of-supremum}

Fix $(\nu, \boldsymbol{\theta}) \in \Sigma_R(\boldsymbol{\theta}) \times \Theta^K$ such that $a^\star(\boldsymbol{\theta})$ is unique. Define
\begin{equation}
    u(\nu, \boldsymbol{\theta}, \boldsymbol{\lambda}) \coloneqq \sum_{(\mathbf{d}, \mathbf{i}) \in \mathbb{S}_R}\  \sum_{a=1}^{K} \ \nu(\mathbf{d}, \mathbf{i}, a) \,  D_{\text{\rm KL}}(Q_{\boldsymbol{\theta}, R}(\cdot \mid\mathbf{d}, \mathbf{i}, a) \| Q_{\boldsymbol{\lambda}, R}(\cdot \mid\mathbf{d}, \mathbf{i}, a)).
    \label{eq:u(nu,theta,lambda)}
\end{equation}
Clearly, $u$ is continuous in its arguments. From \eqref{eq:Alt-theta-2}, it is evident that for each $\boldsymbol{\theta} \in \Theta$, the set of alternatives $\textsc{Alt}(\boldsymbol{\theta})$ is non-compact (in fact, an open subset of $\mathbb{R}^K$). To show that $\psi$ in \eqref{eq:psi} is continuous, we employ \cite[Theorem 1.2]{feinberg2014berges}, a version of Berge's maximum theorem \cite[p. 84]{hu1997handbook} for non-compact image sets. For completeness, we restate this result below.
\begin{definition}($\mathbb{K}$-inf compactness)
    Let $\mathbb{X}, \mathbb{Y}$ be two topological spaces, $\Phi: \mathbb{X} \to 2^{\mathbb{Y}}$ be a mapping from $\mathbb{X}$ to the power set of $\mathbb{Y}$, and let $\mathbb{K}(\mathbb{X})$ denote the collection of all compact subsets of $\mathbb{X}$. Let
    $$
    \textsf{Gr}_{\mathbb{X}}(\Phi) \coloneqq \{(x,y) \in \mathbb{X} \times \mathbb{Y}: y \in \Phi(x)\}
    $$
    denote the graph of $\Phi$.
    A function $u:\mathbb{X} \times \mathbb{Y} \to \mathbb{R}$ is called $\mathbb{K}$-inf compact on $\textsf{Gr}_{\mathbb{X}}(\Phi)$ if for every compact set $K \in \mathbb{K}(\mathbb{X})$, the lower level set $$ \{(x, y)\in K\times \mathbb{Y}: y\in \Phi(x), \ u(x, y) \leq \alpha \} $$ is compact for all $\alpha \in \mathbb{R}$.
\end{definition}
\begin{lemma}\cite[Theorem 1.2]{feinberg2014berges}
    \label{lem:Berges-maximum-theorem-for-noncompact-image-sets}
    Let $\mathbb{X}, \mathbb{Y}$ be two topological spaces, $\Phi:\mathbb{X} \to 2^{\mathbb{Y}}$ be a mapping from $\mathbb{X}$ to the power set of $\mathbb{Y}$, and $u:\mathbb{X} \times \mathbb{Y} \to \mathbb{R}$ be a real-valued mapping on $\mathbb{X} \times \mathbb{Y}$. Let $\mathbb{K}(\mathbb{X})$ denote the collection of all compact subsets of $\mathbb{X}$. Assume that
    \begin{enumerate}
        \item $\mathbb{X}$ is compactly generated, i.e., each set $A \subseteq \mathbb{X}$ is closed in $\mathbb{X}$ if $A \cap K$ is closed in $K$ for every $K \in \mathbb{K}(\mathbb{X})$.
        \item $\Phi$ is lower semi-continuous.
        \item $u$ is $\mathbb{K}$-inf compact and upper semi-continuous on $\textsf{Gr}_{\mathbb{X}}(\Phi)$.
    \end{enumerate}
    Then, the function $v:\mathbb{X} \to \mathbb{R}$ defined via $v(x) = \inf_{y \in \Phi(x)} u(x,y)$ is continuous.
\end{lemma}
In the following, we carefully verify that the hypotheses of Lemma~\ref{lem:Berges-maximum-theorem-for-noncompact-image-sets} are satisfied in the context of our setting, with $\mathbb{X} = \Sigma_R(\boldsymbol{\theta}) \times \Theta^K$, $\mathbb{Y} = \Theta^K$, and $\Phi:\mathbb{X} \to 2^{\mathbb{Y}}$ as the mapping $(\nu, \boldsymbol{\theta}, \boldsymbol{\lambda}) \mapsto \Phi(\nu, \boldsymbol{\theta}, \boldsymbol{\lambda})=\textsc{Alt}(\boldsymbol{\theta}) \subset \mathbb{Y}$.
\begin{enumerate}
    \item \textbf{Verification of hypothesis 1}: The first hypotheses to verify is that $\mathbb{X}$ is compactly generated. This hypothesis holds in our work as $\mathbb{X}=\Sigma_R(\boldsymbol{\theta}) \times \Theta^K$ is locally compact and hence compactly generated \cite[Lemma 46.3, p.~283]{munkres2000munkres}.
    
    \item \textbf{Verification of hypothesis 2:} The second hypothesis to verify is that $\Phi$ is lower-hemicontinuous at $(\nu, \boldsymbol{\theta})$. In order to verify this hypothesis, we show a slightly stronger result, namely that $\Phi$ is hemicontinuous (or simply continuous) at $(\nu, \boldsymbol{\theta})$. Let $(\nu_n, \boldsymbol{\theta}_n)$ be a sequence in $\Sigma_R(\boldsymbol{\theta}) \times \Theta^K$ such that $(\nu_n, \boldsymbol{\theta}_n) \to (\nu, \boldsymbol{\theta})$ as $n\to \infty$. Let\footnote{Any $\varepsilon$ such that $\max_{a\neq a^\star(\boldsymbol{\theta})}\eta_a + \varepsilon < \eta_{a^\star(\boldsymbol{\theta})} - \varepsilon$  works.} $$ \varepsilon = \frac{\eta_{a^\star(\boldsymbol{\theta})} - \max_{a \neq a^\star(\boldsymbol{\theta})} \eta_a}{4}.$$ 
    From Lemma~\ref{lem:important-properties-of-family}, we know that $\Dot{A}$ is continuous; therefore, there exists $\xi(\varepsilon)$ such that for all $\theta, \lambda \in \Theta$ satisfying $|\theta - \lambda| < \xi(\varepsilon)$, we have $|\eta_{\theta} - \eta_{\lambda}| < \varepsilon$. Also, there exists $N=N(\varepsilon)$ such that for $$\forall n\geq N, \quad \| (\nu_n, \boldsymbol{\theta}_n) - (\nu, \boldsymbol{\theta}) \| < \xi(\varepsilon), \quad \| \boldsymbol{\theta}_n - \boldsymbol{\theta} \| < \xi(\varepsilon). $$ Because of the choice of $\varepsilon$ and the fact that $a^\star(\boldsymbol{\theta})$ is unique, it follows that for all $n\geq N$, $$\Phi(\nu_n, \boldsymbol{\theta}_n) = \textsc{Alt}(\boldsymbol{\theta}_n) = \textsc{Alt}(\boldsymbol{\theta}) = \Phi(\nu, \boldsymbol{\theta}),$$ thus establishing the continuity of $\Phi$ at $(\nu, \boldsymbol{\theta})$. 
    %The above arguments may be extended to show that $\Phi$ is continuous at all $(\nu, \boldsymbol{\theta})$ such that $a^\star(\boldsymbol{\theta})$ is unique.
    
    \item \textbf{Verification of hypothesis 3}: The third and last hypothesis to verify is that $u: \mathbb{X} \times \mathbb{Y} \to \mathbb{R}$ as defined in \eqref{eq:u(nu,theta,lambda)} satisfies the following properties:
    \begin{enumerate}
        \item $u$ is upper-semicontinuous on its graph $$\textsf{Gr}_{\mathbb{X}}(\Phi) = \{(\nu, \boldsymbol{\theta}, \boldsymbol{\lambda})\in \mathbb{X}\times \mathbb{Y}: \boldsymbol{\lambda} \in \Phi(\nu, \boldsymbol{\theta})\} = \{(\nu, \boldsymbol{\theta}, \boldsymbol{\lambda})\in \mathbb{X}\times \mathbb{Y}: \boldsymbol{\lambda} \in \textsc{Alt}(\boldsymbol{\theta})\} .$$
        \item $u$ is $\mathbb{K}$-inf compact on its graph $\textsf{Gr}_{\mathbb{X}}(\Phi)$.
    \end{enumerate}
    Because $u$ is continuous, it is also upper-semicontinuous on its graph, and hypothesis 3(a) follows. Hypothesis 3(b) follows from \cite[Lemma 2.1(i)]{feinberg2013berge}, noting that $\Phi$ is upper-hemicontinuous (a fact that follows from the continuity of $\Phi$ established above).
\end{enumerate}
With the hypotheses of Lemma~\ref{lem:Berges-maximum-theorem-for-noncompact-image-sets} now verified, it follows from the preceding result that $\psi$ is continuous at $(\nu, \boldsymbol{\theta})$. The above arguments may be extended to show that $\psi$ is continuous at all $(\nu, \boldsymbol{\theta})$ such that $a^\star(\boldsymbol{\theta})$ is unique.\footnote{Recall that the uniqueness of $a^\star(\boldsymbol{\theta})$ was used in the verification of hypothesis 2.} 

Using the continuity of $\psi$ just established, and noting that (a) $\Sigma_R(\boldsymbol{\theta})$ is compact for all $\boldsymbol{\theta} \in \Theta^K$, and (b) $\boldsymbol{\theta} \mapsto \Sigma_R(\boldsymbol{\theta})$ is hemicontinuous (i.e., upper-hemicontinuous and lower-hemicontinuous) thanks to the continuity of $\boldsymbol{\theta} \mapsto Q_{\boldsymbol{\theta}, R}$, a simple application of the Berge's maximum theorem \cite[p.~84]{hu1997handbook} yields that $\boldsymbol{\theta} \mapsto T^\star(\boldsymbol{\theta})$ is continuous and that $\boldsymbol{\theta} \mapsto \mathcal{W}^\star(\boldsymbol{\theta})$ is upper-hemicontinuous and compact-valued. This completes the proof.

\section{Proof of Lemma~\ref{lem:sufficient-exploration-of-state-actions}}
\subsection{Proof of Part 1}
\label{subsec:proof-of-part-1}
The proof below is an adaptation of the one in \cite[Appendix D.1]{al2021navigating}. To prove \eqref{eq:almost-sure-divergence-of-state-action-visitations}, we consider the event
\begin{equation}
    \mathcal{E} = \bigcup_{(\mathbf{d}', \mathbf{i}', a')} \{\exists M > 0 \text{ such that }\forall n \geq K, \ \ N(n, \mathbf{d}', \mathbf{i}', a') < M\},
    \label{eq:event-of-zero-probability}
\end{equation}
and demonstrate that $\P_{\boldsymbol{\theta}}(\mathcal{E})=0$. For each $z' = (\mathbf{d}', \mathbf{i}', a')$, let 
\begin{equation}
    \mathcal{E}_{z'} \coloneqq \{\exists M > 0 \text{ such that }\forall n \geq K, \ \ N(n, \mathbf{d}', \mathbf{i}', a') < M\}.
    \label{eq:E-z-prime-event}
\end{equation}
We shall demonstrate below that $\P_{\boldsymbol{\theta}}(\mathcal{E}_{z'})=0$ for all $z' \in \mathbb{S}_R \times [K]$. Towards this, let 
\begin{equation}
    m_{\boldsymbol{\theta}} \coloneqq \max_{(\mathbf{d}, \mathbf{i}), (\mathbf{d}', \mathbf{i}') \in \mathbb{S}_R} \min\{r \geq 1: ~\exists \pi\;\; \text{\rm  such that }\; Q_{\boldsymbol{\theta}, \pi}^r(\mathbf{d}', \mathbf{i}'|\mathbf{d}, \mathbf{i})>0\}\ .
    \label{eq:m-theta-definition}
\end{equation}
A straightforward extension of Lemma~\ref{lem:MDP-is-communicating} (without the $R$-max-delay constraint) to include the $R$-max-delay constraint shows that the MDP $\mathcal{M}_{\boldsymbol{\theta}, R}$ is communicating, i.e., for all $(\mathbf{d}, \mathbf{i}), (\mathbf{d}', \mathbf{i}') \in \mathbb{S}_R$, there exists $r \in \mathbb{N}$ and a policy $\pi$ such that $Q_{\boldsymbol{\theta}, \pi}^r(\mathbf{d}', \mathbf{i}'|\mathbf{d}, \mathbf{i}) > 0$. In other words, there exists a path of length $r$ to reach the state $(\mathbf{d}', \mathbf{i}')$ starting from the state $(\mathbf{d}, \mathbf{i})$. Observe that $r \leq m_{\boldsymbol{\theta}}$ by the definition of $m_{\boldsymbol{\theta}}$. Furthermore, for any $a, a' \in [K]$, the state-action pair $(\mathbf{d}', \mathbf{i}', a')$ may be reached starting from $(\mathbf{d}, \mathbf{i}, a)$ in at most $m_{\boldsymbol{\theta}} + 1$ steps in the following manner:
\begin{equation}
    (\mathbf{d}, \mathbf{i}) \xrightarrow[\text{select }a]{} (\mathbf{d}'', \mathbf{i}'') \xrightarrow[\text{shortest path from }(\mathbf{d}'', \mathbf{i}'') \text{ to } (\mathbf{d}', \mathbf{i}')]{} (\mathbf{d}', \mathbf{i}').
    \label{eq:shortest-path}
\end{equation}
In \eqref{eq:shortest-path}, $(\mathbf{d}'', \mathbf{i}'')$ is an intermediate state that results from selecting arm $a$ in state $(\mathbf{d}, \mathbf{i})$. Observe that the length of the shortest path from $(\mathbf{d}'', \mathbf{i}'')$ to $(\mathbf{d}', \mathbf{i}')$ is no more than $m_{\boldsymbol{\theta}}$. We therefore have
\begin{equation}
    \gamma_{\boldsymbol{\theta}}
    \coloneqq \min_{(\mathbf{d}, \mathbf{i}, a), (\mathbf{d}', \mathbf{i}', a')} \ \max_{\substack{1 \leq r \leq m_{\boldsymbol{\theta}}+1 \\ \text{Policy }\pi}} Q_{\boldsymbol{\theta}, \pi}^r(\mathbf{d}', \mathbf{i}', a'|\mathbf{d}, \mathbf{i}, a) > 0.
    \label{eq:gamma-theta-strictly-positive}
\end{equation}
Next, we note the following result which is a simple consequence of the finiteness of $\mathbb{S}_R \times [K]$.
\begin{lemma}
    \label{lem:state-action-infinitely-often}
    There exists $(\mathbf{d}, \mathbf{i}, a) \in \mathbb{S}_R \times [K]$ such that under the uniform policy $\pi^{\text{\rm unif}}$, the state-action pair $(\mathbf{d}, \mathbf{i}, a)$ appears infinitely often.
\end{lemma}
Let $\mathcal{Z}=\{(\mathbf{d}, \mathbf{i}, a): (\mathbf{d}, \mathbf{i}, a) \text{ appears infinitely often under }\pi^{\text{\rm unif}} \}$. Going further, we fix $z = (\mathbf{d}, \mathbf{i}, a) \in \mathcal{Z}$ and $z' = (\mathbf{d}', \mathbf{i}', a') \in \mathbb{S}_R \times [K]$ arbitrarily.
If $z' \in \mathcal{Z}$, then it is immediate that $\P_{\boldsymbol{\theta}}(\mathcal{E}_{z'})=0$. Suppose that $z' \notin \mathcal{Z}$. From the communicating property of the MDP $\mathcal{M}_{\boldsymbol{\theta}, R}$, we know that there exists $(r^*, \pi^*)$ (possibly depending on $z, z'$) such that $Q_{\boldsymbol{\theta}, \pi^*}^{r^*}(z'|z) > 0$. Note that for all $n \geq K$,
\begin{align}
    Q_{\boldsymbol{\theta}, \pi_n}(z'|z) 
    &= Q_{\boldsymbol{\theta}, \pi_n}(\mathbf{d}', \mathbf{i}', a'|\mathbf{d}, \mathbf{i}, a) \nonumber\\
    &= Q_{\boldsymbol{\theta}, R}(\mathbf{d}', \mathbf{i}'|\mathbf{d}, \mathbf{i}, a) \cdot \pi_n(a'|\mathbf{d}', \mathbf{i}') \nonumber\\
    &\stackrel{(a)}{=} Q_{\boldsymbol{\theta}, R}(\mathbf{d}', \mathbf{i}'|\mathbf{d}, \mathbf{i}, a) \cdot (\varepsilon_n\pi^{\text{\rm unif}}(a'|\mathbf{d}', \mathbf{i}')+ (1-\varepsilon_n)\, \pi_{\widehat{\boldsymbol{\theta}}(n-1)}^\eta(a'|\mathbf{d}', \mathbf{i}')) \nonumber\\
    &\geq Q_{\boldsymbol{\theta}, R}(\mathbf{d}', \mathbf{i}'|\mathbf{d}, \mathbf{i}, a) \cdot \frac{\varepsilon_n}{K} \nonumber\\
    &\geq Q_{\boldsymbol{\theta}, R}(\mathbf{d}', \mathbf{i}'|\mathbf{d}, \mathbf{i}, a) \cdot \frac{\varepsilon_n}{K} \cdot \pi^*(a'|\mathbf{d}', \mathbf{i}') \nonumber\\
    &= \frac{\varepsilon_n}{K} \, Q_{\boldsymbol{\theta}, \pi^*}(z'|z),
    \label{eq:proof-of-divergence-of-state-action-visits-2}
\end{align}
where $(a)$ above follows from the definition of $\pi_n$ in \eqref{eq:pi-n-definition}. 

Let $\tau_k(z)$ denote the time instant at which the state-action $z$ appears for the $k$th time, $k \in \mathbb{N}$. Then, conditioned on a realisation of $\{\pi_n\}_{n=1}^{\infty}$ and $\{\tau_k(z)\}_{z \in \mathcal{Z}, k \in \mathbb{N}}$, we have
\begin{align}
    &\P_{\boldsymbol{\theta}}(\mathcal{E}_{z'} \mid \{\pi_n\}_{n=1}^{\infty}, \{\tau_k(z)\}_{z \in \mathcal{Z}, k \in \mathbb{N}}) \nonumber\\
    &\leq \P_{\boldsymbol{\theta}}\bigg( \exists N \text{ s.t. } \forall k \geq N, \ (\mathbf{d}(\tau_k(z)+r^*), \mathbf{i}(\tau_k(z)+r^*), A_{\tau_k(z)+r^*})\neq z' \mid \{\pi_n\}_{n=1}^{\infty}, \{\tau_k(z)\}_{z \in \mathcal{Z}, k \in \mathbb{N}} \bigg) \nonumber\\
    &\leq \sum_{N=1}^{\infty} \P_{\boldsymbol{\theta}} \bigg( \forall k \geq N, \ (\mathbf{d}(\tau_k(z)+r^*), \mathbf{i}(\tau_k(z)+r^*), A_{\tau_k(z)+r^*})\neq z' \mid \{\pi_n\}_{n=1}^{\infty}, \{\tau_k(z)\}_{z \in \mathcal{Z}, k \in \mathbb{N}} \bigg) \nonumber\\
    &\stackrel{(a)}{=} \sum_{N=1}^{\infty} \ \prod_{k=N}^{\infty} \P_{\boldsymbol{\theta}} \bigg( (\mathbf{d}(\tau_k(z)+r^*), \mathbf{i}(\tau_k(z)+r^*), A_{\tau_k(z)+r^*})\neq z' \mid \tau_k(z), \, \{\pi_n\}_{n=\tau_k(z)+1}^{\tau_k(z)+r^*} \bigg) \nonumber\\
    &= \sum_{N=1}^{\infty} \ \prod_{k=N}^{\infty} \left[ 1 - \left(\prod_{n=\tau_k(z)+1}^{\tau_k(z)+r^*} Q_{\boldsymbol{\theta}, \pi_n}\right)(z, z') \right] \nonumber\\
    &\leq \sum_{N=1}^{\infty} \ \prod_{k=N}^{\infty} \left[ 1 - \left(\prod_{n=\tau_k(z)+1}^{\tau_k(z)+r^*} \frac{\varepsilon_n}{K}\, Q_{\boldsymbol{\theta}, \pi^*}\right)(z, z') \right] \nonumber\\ 
    &\leq \sum_{N=1}^{\infty} \ \prod_{k=N}^{\infty} \left[ 1 - \frac{\gamma_{\boldsymbol{\theta}}}{K^{r^*}} \left(\prod_{n=\tau_k(z)+1}^{\tau_k(z)+r^*} \varepsilon_n \right) \right] \nonumber\\  
    &\leq \sum_{N=1}^{\infty} \ \prod_{k=N}^{\infty} \left[ 1 - \frac{\gamma_{\boldsymbol{\theta}}}{K^{m_{\boldsymbol{\theta}}}} \left(\prod_{n=\tau_k(z)+1}^{\tau_k(z)+m_{\boldsymbol{\theta}}+1} \varepsilon_n \right) \right],
    \label{eq:proof-of-divergence-of-state-action-visits-3}
\end{align}
where $(a)$ above follows from the strong Markov property, and the last line follows by noting that $r^* \leq m_{\boldsymbol{\theta}}+1$ and $\varepsilon_n < 1$ for all $n$. Because \eqref{eq:proof-of-divergence-of-state-action-visits-3} holds for all realisations of $\{\pi_n\}_{n=1}^{\infty}$, marginalising across all such realisations, we get
\begin{equation}
    \P_{\boldsymbol{\theta}}(\mathcal{E}_{z'} \mid \{\tau_k(z)\}_{z \in \mathcal{Z}, k \in \mathbb{N}}) \leq \sum_{N=1}^{\infty} \ \prod_{k=N}^{\infty} \left[ 1 - \frac{\gamma_{\boldsymbol{\theta}}}{K^{m_{\boldsymbol{\theta}}}} \left(\prod_{n=\tau_k(z)+1}^{\tau_k(z)+m_{\boldsymbol{\theta}}+1} \varepsilon_n \right) \right] \quad \forall z \in \mathcal{Z}.
    \label{eq:proof-of-divergence-of-state-action-visits-4}
\end{equation}
This implies that
\begin{align}
    \prod_{z \in \mathcal{Z}} \P_{\boldsymbol{\theta}}(\mathcal{E}_{z'} \mid \{\tau_k(z)\}_{z \in \mathcal{Z}, k \in \mathbb{N}}) \leq \sum_{N_z: z \in \mathcal{Z}} \ \prod_{z \in \mathcal{Z}} \ \prod_{k=N_z}^{\infty} \left[ 1 - \frac{\gamma_{\boldsymbol{\theta}}}{K^{m_{\boldsymbol{\theta}}}} \left(\prod_{n=\tau_k(z)+1}^{\tau_k(z)+m_{\boldsymbol{\theta}}+1} \varepsilon_n \right) \right].
    \label{eq:proof-of-divergence-of-state-action-visits-5}
\end{align}
We now note that for any $k \in \mathbb{N}$, there exists a state-action pair $z_k \in \mathbb{S}_R \times [K]$ such that $\tau_k(z_k) \leq k \, S_R \, K$, i.e., the state-action pair $z_k$ is seen $k$ times up to time $n=k\,S_R\,K$ (this can be easily argued via induction on $k$ and using the fact that $|\mathbb{S}_R \times [K]| = S_R\,K$). For this choice of $\{z_k\}_{k=1}^{\infty}$, we have
\begin{align}
    \prod_{z \in \mathcal{Z}} \ \prod_{k=N_z}^{\infty} \left[ 1 - \frac{\gamma_{\boldsymbol{\theta}}}{K^{m_{\boldsymbol{\theta}}}} \left(\prod_{n=\tau_k(z)+1}^{\tau_k(z)+m_{\boldsymbol{\theta}}+1} \varepsilon_n \right) \right] 
    &\leq \prod_{z \in \mathcal{Z}} \ \prod_{k=\max_{z \in \mathcal{Z}} N_z}^{\infty} \left[ 1 - \frac{\gamma_{\boldsymbol{\theta}}}{K^{m_{\boldsymbol{\theta}}}} \left(\prod_{n=\tau_k(z)+1}^{\tau_k(z)+m_{\boldsymbol{\theta}}+1} n^{-\frac{1}{2(1+S_R)}} \right) \right] \nonumber\\
    &= \prod_{k=\max_{z \in \mathcal{Z}} N_z}^{\infty} \ \prod_{z \in \mathcal{Z}}  \left[ 1 - \frac{\gamma_{\boldsymbol{\theta}}}{K^{m_{\boldsymbol{\theta}}}} \left(\prod_{n=\tau_k(z)+1}^{\tau_k(z)+m_{\boldsymbol{\theta}}+1} n^{-\frac{1}{2(1+S_R)}} \right) \right] \nonumber\\
    &\leq \prod_{k=\max_{z \in \mathcal{Z}} N_z}^{\infty}  \left[ 1 - \frac{\gamma_{\boldsymbol{\theta}}}{K^{m_{\boldsymbol{\theta}}}} \left(\prod_{n=\tau_k(z_k)+1}^{\tau_k(z_k)+m_{\boldsymbol{\theta}}+1} n^{-\frac{1}{2(1+S_R)}} \right) \right] \nonumber\\
    &\leq \prod_{k=\max_{z \in \mathcal{Z}} N_z}^{\infty}  \left[ 1 - \frac{\gamma_{\boldsymbol{\theta}}}{K^{m_{\boldsymbol{\theta}}}} \left(\prod_{n=k\,S_R\,K+1}^{k\,S_R\,K+m_{\boldsymbol{\theta}}+1} n^{-\frac{1}{2(1+S_R)}} \right) \right] \nonumber\\
    &=0
    \label{eq:proof-of-divergence-of-state-action-visits-6}
\end{align}
for all choices of $\{N_z: z \in \mathcal{Z}\}$, where the penultimate line follows by noting that because $n \mapsto n^{-\frac{1}{2(1+S_R)}}$ is monotone decreasing, and $\tau_k(z_k) \leq k\,S_R\,K$ for all $k \in \mathbb{N}$ by definition, we have 
$$
\prod_{n=\tau_k(z_k)+1}^{\tau_k(z_k)+m_{\boldsymbol{\theta}}+1} n^{-\frac{1}{2(1+S_R)}} ~\geq~ \prod_{n=k\,S_R\,K+1}^{k\,S_R\,K+m_{\boldsymbol{\theta}}+1} n^{-\frac{1}{2(1+S_R)}} \quad \forall k \geq 1.
$$
Thus, \eqref{eq:proof-of-divergence-of-state-action-visits-6} implies that for all realisations of $\{\tau_k(z)\}_{z \in \mathcal{Z}, k \in \mathbb{N}}$,
\begin{equation}
    \prod_{z \in \mathcal{Z}} \P_{\boldsymbol{\theta}}(\mathcal{E}_{z'} \mid \{\tau_k(z)\}_{z \in \mathcal{Z}, k \in \mathbb{N}})=0.
\end{equation}
Marginalising over all such realisations leads to $\P_{\boldsymbol{\theta}}(\mathcal{E}_{z'})=0$ for all $z'$, which in turn implies that $\P_{\boldsymbol{\theta}}(\mathcal{E})=0$.

\subsection{Proof of Part 2}
\label{subsec:proof-of-part-2}
%The proof of this part follows along the exact lines as in \cite{al2021navigating} and is omitted for brevity.
Our proof is inspired by the proof of Lemma~11 in \cite{al2021navigating}, with necessary modifications for the problem studied here. Let $\lambda_\alpha(\boldsymbol{\theta})$ be as defined in the statement of the lemma; we will soon specify the constant $\sigma_{\boldsymbol{\theta}}$ appearing in the definition of $\lambda_{\alpha}(\boldsymbol{\theta})$. Let $\mathcal{Z}$ and $\{\tau_k(z): k \in \mathbb{N}\}$ be as defined in Section~\ref{subsec:proof-of-part-1} above. In what follows, we prove that when $\varepsilon_n = n^{-\frac{1}{2(1+S_R)}}$, we have
\begin{equation}
    \P_{\boldsymbol{\theta}} \bigg(\forall z \in \mathcal{Z}, ~ \forall k \in \mathbb{N}, ~ \tau_k(z) \leq \lambda_\alpha(\boldsymbol{\theta}) \, k^4 \bigg) \geq 1-\alpha \quad \forall \alpha \in (0,1).
    \label{eq:proof-of-divergence-of-state-action-visits-7}
\end{equation}
The desired result in \eqref{eq:sufficient-exploration-of-states-and-actions} then follows from \eqref{eq:proof-of-divergence-of-state-action-visits-7} by noting that for every $z = (\mathbf{d}, \mathbf{i}, a) \in \mathcal{Z}$ and $k \in \mathbb{N}$, we have $N(\tau_k(z), \mathbf{d}, \mathbf{i}, a) = k$, and therefore
$$
\bigg\lbrace \tau_k(z) \leq \lambda_\alpha(\boldsymbol{\theta}) \, k^4 \bigg\rbrace = \bigg\lbrace N(\tau_k(z), \mathbf{d}, \mathbf{i}, a) \geq \left(\frac{\tau_k(z)}{\lambda_\alpha(\boldsymbol{\theta})}\right)^{1/4} \bigg\rbrace.
$$
We now proceed to prove \eqref{eq:proof-of-divergence-of-state-action-visits-7}. Fix $\alpha \in (0,1)$. Let $g:\mathbb{N} \to \mathbb{N}$ be a strictly increasing function with $g(0)=0$ and $g(k) \geq k$ for all $k \in \mathbb{N}$. Let 
\begin{equation}
    \mathcal{H} = \bigg\lbrace \forall z \in \mathcal{Z}, ~ \forall k \in \mathbb{N}, ~ \tau_k(z) \leq g(k) \bigg\rbrace.
    \label{eq:event-E-in-proof-of-divergence-of-state-action-visits}
\end{equation}
We shall prove below that
\begin{equation}
    \mathbb{P}_{\boldsymbol{\theta}}(\overline{\mathcal{H}}) \leq K\, S_R \, \sum_{k=1}^{\infty} \ \prod_{j=0}^{\left\lfloor \frac{g(k) - g(k-1) - 1}{m_{\boldsymbol{\theta}}+2} \right\rfloor - 1} \left[ 1 - \sigma_{\boldsymbol{\theta}} \prod_{l=1}^{m_{\boldsymbol{\theta}}+2} \varepsilon_{g(k-1) + (m_{\boldsymbol{\theta}}+2)\cdot j + l} \right],
    \label{eq:proof-of-divergence-of-state-action-visits-8}
\end{equation}
where $\overline{\mathcal{H}}$ denotes the complement of $\mathcal{H}$, the constant $m_{\boldsymbol{\theta}}$ is as defined in \eqref{eq:m-theta-definition} and $\sigma_{\boldsymbol{\theta}}$ is a constant that depends only on $\boldsymbol{\theta}$; see Lemma~\ref{lem:norm-of-product-of-Aj} below for a formal definition of $\sigma_{\boldsymbol{\theta}}$. We then tune $\varepsilon_n$ and $g$ suitably so that the right-hand side of \eqref{eq:proof-of-divergence-of-state-action-visits-8} equals $\alpha$. Note that
\begin{align}
    \overline{\mathcal{H}} 
    &= \bigcup_{z \in \mathcal{Z}} \ \bigcup_{k=1}^{\infty} \bigg\lbrace \tau_k(z) > g(k) \bigg\rbrace \nonumber\\
    &= \bigcup_{z \in \mathcal{Z}} \ \bigcup_{k=1}^{\infty} \bigg\lbrace \tau_k(z) > g(k) \text{ and } \forall j \leq k-1, ~ \tau_{j}(z) \leq g(j)\bigg\rbrace,
\end{align}
whence it follows from an application of the union bound that
\begin{align}
    &\mathbb{P}_{\boldsymbol{\theta}}(\overline{\mathcal{H}}) 
    \leq \sum_{z \in \mathcal{Z}} \bigg[\mathbb{P}_{\boldsymbol{\theta}}(\tau_1(z) > g(1)) +  \sum_{k=2}^{\infty} \mathbb{P}_{\boldsymbol{\theta}} \bigg( \tau_k(z) > g(k) \text{ and } \forall j \leq k-1, ~ \tau_{j}(z) \leq g(j) \bigg) \bigg] \nonumber\\
    &\leq \sum_{z \in \mathcal{Z}} \bigg[\mathbb{P}_{\boldsymbol{\theta}}(\tau_1(z) > g(1)) + \sum_{k=2}^{\infty} \mathbb{P}_{\boldsymbol{\theta}} \bigg( \tau_k(z) > g(k), ~ \tau_{k-1}(z) \leq g(k-1) \bigg) \bigg] \nonumber\\
    &\leq \sum_{z \in \mathcal{Z}} \bigg[\mathbb{P}_{\boldsymbol{\theta}}(\tau_1(z) > g(1)) + \sum_{k=2}^{\infty} \mathbb{P}_{\boldsymbol{\theta}} \bigg( \tau_k(z) - \tau_{k-1}(z) > g(k) - g(k-1), ~ \tau_{k-1}(z) \leq g(k-1) \bigg) \bigg] \nonumber\\
    &\leq \sum_{z \in \mathcal{Z}} \bigg[\mathbb{P}_{\boldsymbol{\theta}}(\tau_1(z) > g(1)) + \sum_{k=2}^{\infty} \ \sum_{n=1}^{g(k-1)} \mathbb{P}_{\boldsymbol{\theta}} \bigg( \tau_k(z) - \tau_{k-1}(z) > g(k) - g(k-1) \mid \tau_{k-1}(z) = n \bigg) \, \mathbb{P}_{\boldsymbol{\theta}}(\tau_{k-1}(z) = n) \bigg] \nonumber\\
    &= \sum_{z \in \mathcal{Z}} \bigg[a_1(z) + \sum_{k=2}^{\infty} \ \sum_{n=1}^{g(k-1)} a_{k,n}(z) \, \mathbb{P}_{\boldsymbol{\theta}}(\tau_{k-1}(z) = n) \bigg],
    \label{eq:proof-of-divergence-of-state-action-visits-9}
\end{align}
where 
\begin{align}
    a_1(z) 
    &\coloneqq \mathbb{P}_{\boldsymbol{\theta}}(\tau_1(z) > g(1)), \label{eq:a-1-z-definition} \\
    a_{k,n}(z) &\coloneqq \mathbb{P}_{\boldsymbol{\theta}} \bigg( \tau_k(z) - \tau_{k-1}(z) > g(k) - g(k-1) \mid \tau_{k-1}(z) = n \bigg). \label{eq:a-k-n-of-z-definition}
\end{align}
We now upper bound $\alpha_{k,n}(z)$ for each $z \in \mathcal{Z}$ and $k \in \mathbb{N}$.

\noindent \textbf{Upper bounding $\alpha_{k,n}(z)$:} Fix $z \in \mathcal{Z}$ and $k \in \mathbb{N}$. In order to upper bound $\alpha_{k,n}(z)$, we write the transition kernel $Q_{\boldsymbol{\theta}, \pi_n}$ for any $n$ as follows:
\begin{equation}
    Q_{\boldsymbol{\theta}, \pi_n} = 
    \begin{pmatrix}
        A_n(z) & [Q_{\boldsymbol{\theta}, \pi_n}(z|z'): z' \neq z]^\top \\
        [Q_{\boldsymbol{\theta}, \pi_n}(z'|z): z' \neq z]^\top & Q_{\boldsymbol{\theta}, \pi_n}(z|z)
    \end{pmatrix}.
    \label{eq:rewriting-transition-kernel}
\end{equation}
In writing \eqref{eq:rewriting-transition-kernel}, we assume without loss of generality that the last row and last column of $Q_{\boldsymbol{\theta}, \pi_n}$ correspond to state $z$. Given $z' \neq z$, let $\mathbf{q}_n(z', \neg z) \coloneqq [Q_{\boldsymbol{\theta}, \pi_n}(z''|z'): z'' \neq z]^\top$. Then, for all $N \in \mathbb{N}$, it follows that
\begin{equation}
    \mathbb{P}_{\boldsymbol{\theta}}(\tau_k(z) - \tau_{k-1}(z) > N \mid \tau_{k-1}(z)=n) = \mathbf{q}_{n}(z, \neg z) \cdot \prod_{l=n+1}^{n+N-1} A_l(z) \cdot \mathbf{1}.
    \label{eq:proof-of-divergence-of-state-action-visits-10}
\end{equation}
In \eqref{eq:proof-of-divergence-of-state-action-visits-10}, $\mathbf{1}$ denotes the all-ones vector. 

We now note the following result of interest. We continue with the above proof after the proof of the below result.
\begin{lemma}
    \label{lem:norm-of-product-of-Aj}
    Let $Q_{\boldsymbol{\theta}, \pi_n}$ be as expressed in \eqref{eq:rewriting-transition-kernel}. Let
    \begin{align}
        \sigma_{\boldsymbol{\theta}, 1} &\coloneqq \min\{Q_{\boldsymbol{\theta}, \pi^{\text{\rm unif}}}(z'|z): z, z' \in \mathcal{Z}, ~ Q_{\boldsymbol{\theta}, \pi^{\text{\rm unif}}}(z'|z)>0\}, \label{eq:alpha-theta-1-definition} \\
        \sigma_{\boldsymbol{\theta}, 2} &\coloneqq \min\{Q_{\boldsymbol{\theta}, \pi^{\text{\rm unif}}}^{r}(z'|z): z, z' \in \mathcal{Z}, ~ r \in \{1, \ldots, m_{\boldsymbol{\theta}}+1\}, ~ Q_{\boldsymbol{\theta}, \pi^{\text{\rm unif}}}^{r}(z'|z)>0\}. \label{eq:alpha-theta-2-definition}
    \end{align}
    where $m_{\boldsymbol{\theta}}$ is as defined in \eqref{eq:m-theta-definition}. Further, let $\sigma_{\boldsymbol{\theta}} \coloneqq \sigma_{\boldsymbol{\theta}, 1} \, \sigma_{\boldsymbol{\theta}, 2}$. Then,
    \begin{equation}
        \bigg \| \prod_{l=n+1}^{n+m_{\boldsymbol{\theta}}+2} A_l(z) \bigg \|_{\infty} \leq 1-\sigma_{\boldsymbol{\theta}} \, \prod_{l=n+1}^{n+m_{\boldsymbol{\theta}}+2}\varepsilon_l \ , \qquad \forall n \in \mathbb{N}.
        \label{eq:norm-of-product-of-Aj}
    \end{equation}
\end{lemma}
\begin{proof}[Proof of Lemma~\ref{lem:norm-of-product-of-Aj}]
    Let $\mathcal{Z}=\mathbb{S}_R \times [K]$, with cardinality $|\mathcal{Z}|=K \, S_R$, be enumerated as $\mathcal{Z} = \{1, \ldots, K \, S_R-1, z\}$, where $z \in \mathcal{Z}$ corresponds to the last row and last column in the expression for $Q_{\boldsymbol{\theta}, \pi_n}$ in \eqref{eq:rewriting-transition-kernel}. In particular, we note that $A_n(z)$ in \eqref{eq:rewriting-transition-kernel} is a matrix whose rows and columns are indexed by the elements from $\{1, \ldots, K \, S_R - 1\}$. Given $n_1, n_2 \in \mathbb{N}$, let
    \begin{equation}
        r_{z'}(n_1, n_2) \coloneqq \sum_{z''=1}^{K \, S_R - 1} \left( \prod_{l=n_1+1}^{n_2}A_l(z) \right)(z''|z')
        \label{eq:proof-of-divergence-of-state-action-visits-11}
    \end{equation}
    denote the sum of elements of the matrix $\prod_{l=n_1+1}^{n_2}A_l(z)$ corresponding to the row indexed by $z'$. We will now prove that 
    \begin{equation}
        r_{z'}(n, n+m_{\boldsymbol{\theta}}+2) \leq 1-\sigma_{\boldsymbol{\theta}}\prod_{l=n+1}^{n+m_{\boldsymbol{\theta}}+2} \varepsilon_l \quad \forall z' \in \{1, \ldots, K \, S_R - 1\}, ~~ \forall n \in \mathbb{N}.
        \label{eq:proof-of-divergence-of-state-action-visits-12}
    \end{equation}
    The desired result in \eqref{eq:norm-of-product-of-Aj} then follows from \eqref{eq:proof-of-divergence-of-state-action-visits-12} by noting that 
    $$
    \bigg \| \prod_{l=n+1}^{n+m_{\boldsymbol{\theta}}+2} A_l(z) \bigg \|_{\infty} = \max \bigg\lbrace r_{z'}(n, n+m_{\boldsymbol{\theta}}+2): z' \in \{1, \ldots, K \, S_R - 1\} \bigg\rbrace.
    $$

    Fix an arbitrary $z^* \in \mathcal{Z}$ such that $Q_{\boldsymbol{\theta}, \pi^{\text{\rm unif}}}(z|z^*) \geq \sigma_{\boldsymbol{\theta}, 1}$; such a $z^*$ is guaranteed to exist because the MDP $\mathcal{M}_{\boldsymbol{\theta}, R}$ is communicating (Lemma~\ref{lem:MDP-R-is-communicating}). Then, for all $n \in \mathbb{N}$,
    \begin{align}
        r_{z^*}(n, n+1)
        &= \sum_{z''=1}^{K \, S_R - 1} \bigg(A_{n+1}(z)\bigg)(z''|z^*) \nonumber\\
        &= 1 - Q_{\boldsymbol{\theta}, \pi_{n+1}}(z|z^*) \nonumber\\
        &\leq 1 - \varepsilon_{n+1} \, Q_{\boldsymbol{\theta}, \pi^{\text{\rm unif}}}(z|z^*) \nonumber\\
        &\leq 1 - \sigma_{\boldsymbol{\theta}, 1} \, \varepsilon_{n+1}.
        \label{eq:proof-of-divergence-of-state-action-visits-13}
    \end{align}
    Similarly, for any $n_1, n_2 \in \mathbb{N}$, we have
    \begin{align}
        r_{z^*}(n_1, n_1+n_2)
        &= \sum_{z''=1}^{K \, S_R - 1} \left(\prod_{l=n_1+1}^{n_1+n_2} A_l(z) \right)(z''|z^*) \nonumber\\
        &= \sum_{z''=1}^{K \, S_R - 1} \ \sum_{z^{'''}=1}^{K \, S_R - 1} \left(\prod_{l=n_1+1}^{n_1+n_2-1} A_l(z) \right)(z'''|z^*) \cdot \bigg(A_{n_1+n_2}(z)\bigg)(z''|z''') \nonumber\\
        &= \sum_{z^{'''}=1}^{K \, S_R - 1} \left(\prod_{l=n_1+1}^{n_1+n_2-1} A_l(z) \right)(z'''|z^*) \cdot \left[\sum_{z''=1}^{K \, S_R - 1} \bigg(A_{n_1+n_2}(z)\bigg)(z''|z''')\right] \nonumber\\
        &= \sum_{z^{'''}=1}^{K \, S_R - 1} \left(\prod_{l=n_1+1}^{n_1+n_2-1} A_l(z) \right)(z'''|z^*) \cdot r_{z'''}(n_1+n_2-1, n_1+n_2) \nonumber\\
        &\stackrel{(a)}{\leq}  \sum_{z^{'''}=1}^{K \, S_R - 1} \left(\prod_{l=n_1+1}^{n_1+n_2-1} A_l(z) \right)(z'''|z^*) \nonumber\\
        &= r_{z^*}(n_1, n_1+n_2-1) \nonumber\\
        &\vdots \nonumber\\
        &= r_{z^*}(n_1, n_1+1) \nonumber\\
        &\leq 1 - \sigma_{\boldsymbol{\theta}, 1} \, \varepsilon_{n_1+1},
        \label{eq:proof-of-divergence-of-state-action-visits-14}
    \end{align}
    where $(a)$ above follows from the fact that $r_{z'''}(n_1+n_2-1, n_1+n_2) \leq 1$ for all $z''' \in \mathcal{Z}$, and the last line above follows from \eqref{eq:proof-of-divergence-of-state-action-visits-13}. Along similar lines as above, we note that for all $z' \in \mathcal{Z}$ and $n_1 \in \{1, \ldots, m_{\boldsymbol{\theta}}+1\}$,
    \begin{align}
        r_{z'}(n, n+m_{\boldsymbol{\theta}}+2)
        &= \sum_{z''=1}^{K \, S_R - 1} \left(\prod_{l=n+1}^{n+m_{\boldsymbol{\theta}}+2} A_l(z) \right)(z''|z') \nonumber\\
        &= \sum_{z''=1}^{K \, S_R - 1} \ \sum_{z^{'''}=1}^{K \, S_R - 1} \left(\prod_{l=n+1}^{n+n_1} A_l(z) \right)(z'''|z') \cdot \left(\prod_{l=n+n_1+1}^{n+m_{\boldsymbol{\theta}}+2} A_{l}(z)\right)(z''|z''') \nonumber\\
        &= \sum_{z^{'''}=1}^{K \, S_R - 1} \left(\prod_{l=n+1}^{n+n_1} A_l(z) \right)(z'''|z') \cdot \left[\sum_{z''=1}^{K \, S_R - 1} \left(\prod_{l=n+n_1+1}^{n+m_{\boldsymbol{\theta}}+2} A_{l}(z)\right)(z''|z''')\right] \nonumber\\
        &= \sum_{z^{'''}=1}^{K \, S_R - 1} \left(\prod_{l=n+1}^{n+n_1} A_l(z) \right)(z'''|z') \cdot r_{z'''}(n+n_1, n+m_{\boldsymbol{\theta}}+2) \nonumber\\
        &\leq  \left(\prod_{l=n+1}^{n+n_1}A_l(z)\right)(z^*|z') \cdot r_{z^*}(n+n_1, n+m_{\boldsymbol{\theta}}+2) + \sum_{z''' \neq z^*}\left(\prod_{l=n+1}^{n+n_1}A_l(z)\right)(z'''|z') \nonumber\\
        &\leq (1-\sigma_{\boldsymbol{\theta}, 1} \, \varepsilon_{n+n_1+1}) \left(\prod_{l=n+1}^{n+n_1}A_l(z)\right)(z^*|z') + 1 - \left(\prod_{l=n+1}^{n+n_1}A_l(z)\right)(z^*|z') \nonumber\\
        &= 1 - \sigma_{\boldsymbol{\theta}, 1} \, \varepsilon_{n+n_1+1} \, \left(\prod_{l=n+1}^{n+n_1}A_l(z)\right)(z^*|z') \nonumber\\
        &\leq 1 - \sigma_{\boldsymbol{\theta}, 1} \, \varepsilon_{n+m_{\boldsymbol{\theta}}+2} \, \left(\prod_{l=n+1}^{n+n_1}A_l(z)\right)(z^*|z'),
        \label{eq:proof-of-divergence-of-state-action-visits-15}
    \end{align}
    where the last line follows from the fact that $\varepsilon_n$ is a decreasing sequence.
    By virtue of the fact that the MDP $\mathcal{M}_{\boldsymbol{\theta}, R}$ is communicating, there exists $r=r(z', z^*)$ such that $Q_{\boldsymbol{\theta}, \pi^{\text{\rm unif}}}^{r}(z^*|z') \geq \sigma_{\boldsymbol{\theta}, 2} > 0$. Here, $r \leq m_{\boldsymbol{\theta}}+1$; for a formal justification of this, see Section~\ref{subsec:proof-of-part-1}. Using the preceding fact and setting $n_1=r$ in \eqref{eq:proof-of-divergence-of-state-action-visits-15}, we get
    \begin{align}
        \left(\prod_{l=n+1}^{n+r}A_l(z)\right)(z^*|z') 
        &\geq \left(\prod_{l=n+1}^{n+r} \varepsilon_l \, Q_{\boldsymbol{\theta}, \pi^{\text{\rm unif}}}\right)(z^*|z') \nonumber\\
        &= \left(\prod_{l=n+1}^{n+r} \varepsilon_l \right) \cdot Q_{\boldsymbol{\theta}, \pi^{\text{\rm unif}}}^r(z^*|z') \nonumber\\
        &\geq \left(\prod_{l=n+1}^{n+r} \varepsilon_l \right) \cdot \sigma_{\boldsymbol{\theta}, 2}.
        \label{eq:proof-of-divergence-of-state-action-visits-16}
    \end{align}
    Combining \eqref{eq:proof-of-divergence-of-state-action-visits-15} and \eqref{eq:proof-of-divergence-of-state-action-visits-16}, we get
    \begin{align}
        r_{z'}(n, n+m_{\boldsymbol{\theta}}+2) 
        &\leq 1-\sigma_{\boldsymbol{\theta}, 1} \, \sigma_{\boldsymbol{\theta}, 2} \, \varepsilon_{n+m_{\boldsymbol{\theta}}+2} \, \prod_{l=n+1}^{n+r} \, \varepsilon_l \nonumber\\
        &\leq 1 - \sigma_{\boldsymbol{\theta}} \,  \varepsilon_{n+m_{\boldsymbol{\theta}}+2} \prod_{l=n+1}^{n+m_{\boldsymbol{\theta}}+1} \, \varepsilon_l \nonumber\\
        &= 1 - \sigma_{\boldsymbol{\theta}} \prod_{l=n+1}^{n+m_{\boldsymbol{\theta}}+2} \, \varepsilon_l,
        \label{eq:proof-of-divergence-of-state-action-visits-17}
    \end{align}
    where the penultimate line follows from using $r \leq m_{\boldsymbol{\theta}}+1$.
\end{proof}
Plugging $N=g(k)-g(k-1)$ in \eqref{eq:proof-of-divergence-of-state-action-visits-10}, and using Lemma~\ref{lem:norm-of-product-of-Aj}, we have for all $k \geq 2$ that
\begin{align}
    \alpha_{k,n}(z) 
    &= \mathbb{P}_{\boldsymbol{\theta}} \bigg( \tau_k(z) - \tau_{k-1}(z) > g(k) - g(k-1) \mid \tau_{k-1}(z) = n \bigg) \nonumber\\
    &= \mathbf{q}_{n}(z, \neg z) \cdot \prod_{l=n+1}^{n+g(k)-g(k-1)-1} A_l(z) \cdot \mathbf{1} \nonumber\\
    &\stackrel{(a)}{\leq} \bigg \| \prod_{l=n+1}^{n+g(k)-g(k-1)-1} A_l(z) \bigg \|_{\infty} \nonumber\\
    &\leq \bigg \| \prod_{l=(m_{\boldsymbol{\theta}}+2) \left\lfloor \frac{g(k)-g(k-1)-1}{m_{\boldsymbol{\theta}}+2} \right\rfloor + 1}^{g(k)-g(k-1)-1} A_{n+l}(z) \bigg \|_{\infty} \times \prod_{j=0}^{\left\lfloor \frac{g(k)-g(k-1)-1}{m_{\boldsymbol{\theta}}+2} \right\rfloor - 1} \ \bigg \| \prod_{l=1}^{m_{\boldsymbol{\theta}}+2} A_{n+(m_{\boldsymbol{\theta}}+2)j + l}(z) \bigg \|_{\infty} \nonumber\\
    &\leq \prod_{j=0}^{\left\lfloor \frac{g(k)-g(k-1)-1}{m_{\boldsymbol{\theta}}+2} \right\rfloor - 1} \bigg[1-\sigma_{\boldsymbol{\theta}} \prod_{l=1}^{m_{\boldsymbol{\theta}}+2} \varepsilon_{n+(m_{\boldsymbol{\theta}}+2)j + l}\bigg] \nonumber\\
    &\leq \prod_{j=0}^{\left\lfloor \frac{g(k)-g(k-1)-1}{m_{\boldsymbol{\theta}}+2} \right\rfloor - 1} \bigg[1-\sigma_{\boldsymbol{\theta}} \prod_{l=1}^{m_{\boldsymbol{\theta}}+2} \varepsilon_{g(k-1)+(m_{\boldsymbol{\theta}}+2)j + l}\bigg],
    \label{eq:proof-of-divergence-of-state-action-visits-18}
\end{align}
where $(a)$ above follows by noting that $\|\mathbf{q}_n(z, \neg z)\|_1 \leq 1$, and the last line follows from noting that $n = \tau_{k-1}(z) \leq k-1 \leq g(k-1)$ and that $\varepsilon_n$ is decreasing in $n$. Going forward, let 
\begin{equation}
    b_k \coloneqq \prod_{j=0}^{\left\lfloor \frac{g(k)-g(k-1)-1}{m_{\boldsymbol{\theta}}+2} \right\rfloor - 1} \bigg[1-\sigma_{\boldsymbol{\theta}} \prod_{l=1}^{m_{\boldsymbol{\theta}}+2} \varepsilon_{g(k-1)+(m_{\boldsymbol{\theta}}+2)j + l}\bigg], \quad k \geq 2.
    \label{eq:b-k-definition}
\end{equation}
Along similar lines as above, it can be shown that
\begin{align}
    a_1(z)
    &= \mathbb{P}_{\boldsymbol{\theta}}(\tau_1(z) > g(1)) \nonumber\\
    &\leq \prod_{j=0}^{\left\lfloor \frac{g(1)-g(0)-1}{m_{\boldsymbol{\theta}}+2} \right\rfloor - 1} \bigg[1-\sigma_{\boldsymbol{\theta}} \prod_{l=1}^{m_{\boldsymbol{\theta}}+2} \varepsilon_{g(0)+(m_{\boldsymbol{\theta}}+2)j + l}\bigg],
    \label{eq:b-1-definition}
\end{align}
where in writing the last line, we use $g(0)=0$. Let $b_1$ denote the right-hand side of \eqref{eq:b-1-definition}.

As the last step in the proof, we tune $g$ and $\{\varepsilon_n: n \geq 1\}$ so that $\sum_{k=1}^{\infty} b_k \leq \alpha$.

\textbf{Tuning $g$ and $\{\varepsilon_n: n \geq 1\}$:} Noting that $\varepsilon_n$ is a decreasing sequence, we have
\begin{align}
    b_k 
    &= \prod_{j=0}^{\left\lfloor \frac{g(k)-g(k-1)-1}{m_{\boldsymbol{\theta}}+2} \right\rfloor - 1} \bigg[1-\sigma_{\boldsymbol{\theta}} \prod_{l=1}^{m_{\boldsymbol{\theta}}+2} \varepsilon_{g(k-1)+(m_{\boldsymbol{\theta}}+2)j + l}\bigg] \nonumber\\
    &\leq \prod_{j=0}^{\left\lfloor \frac{g(k)-g(k-1)-1}{m_{\boldsymbol{\theta}}+2} \right\rfloor - 1} \left[1-\sigma_{\boldsymbol{\theta}}(\varepsilon_{g(k-1)+(m_{\boldsymbol{\theta}}+2)j+m_{\boldsymbol{\theta}}+2})^{m_{\boldsymbol{\theta}}+2}\right] \nonumber\\
    &\leq \left[1-\sigma_{\boldsymbol{\theta}}(\varepsilon_{g(k)})^{m_{\boldsymbol{\theta}}+2}\right]^{\left\lfloor \frac{g(k)-g(k-1)-1}{m_{\boldsymbol{\theta}}+2} \right\rfloor} \nonumber\\
    &\leq \left[1-\sigma_{\boldsymbol{\theta}}(\varepsilon_{g(k)})^{1+S_R}\right]^{\left\lfloor \frac{g(k)-g(k-1)-1}{1+S_R} \right\rfloor},
    \label{eq:proof-of-divergence-of-state-action-visits-19}
\end{align}
where the last line follows from noting that $m_{\boldsymbol{\theta}} \leq S_R - 1$.
For $g(k) = \lambda k^4$ where $\lambda \in \mathbb{N}$, and $\varepsilon_n=n^{-\frac{1}{2(1+S_R)}}$, we get
\begin{equation}
    \left\lfloor \frac{g(k)-g(k-1)-1}{1+S_R} \right\rfloor \geq \frac{\lambda \, k^3}{1+S_R}, \quad (\varepsilon_{g(k)})^{1+S_R} = \frac{1}{k^2\sqrt{\lambda}}.
\end{equation}
Therefore, it follows that for all $k \geq 1$,
\begin{align}
    b_k
    &\leq \left[1-\frac{\sigma_{\boldsymbol{\theta}}}{k^2 \sqrt{\lambda}}\right]^{\frac{\lambda k^3}{1+S_R}} \nonumber\\
    &\leq \exp\left(-\frac{\lambda k^3 \, \sigma_{\boldsymbol{\theta}}}{(1+S_R) \, k^2 \sqrt{\lambda}}\right) \nonumber\\
    &= \exp\left(-\frac{k \, \sigma_{\boldsymbol{\theta}}\, \sqrt{\lambda}}{1+S_R}\right).
    \label{eq:proof-of-divergence-of-state-action-visits-20}
\end{align}
Using \eqref{eq:proof-of-divergence-of-state-action-visits-20} in \eqref{eq:proof-of-divergence-of-state-action-visits-8}, we get
\begin{align}
    \P_{\boldsymbol{\theta}}(\overline{\mathcal{H}}) 
    &\leq K\, S_R \, \sum_{k=1}^{\infty}b_k \nonumber\\
    &\leq K \, S_R \, \sum_{k=1}^{\infty} \exp\left(-\frac{k \, \sigma_{\boldsymbol{\theta}}\, \sqrt{\lambda}}{1+S_R}\right) \nonumber\\
    &=\frac{K \, S_R \, \exp\left(-\frac{\sqrt{\lambda} \, \sigma_{\boldsymbol{\theta}}}{1+S_R}\right)}{1-\exp\left(-\frac{\sqrt{\lambda} \, \sigma_{\boldsymbol{\theta}}}{1+S_R}\right)}.
    \label{eq:proof-of-divergence-of-state-action-visits-21}
\end{align}
Let $h(\lambda)$ denote the quantity on the right-hand side of \eqref{eq:proof-of-divergence-of-state-action-visits-21}. Then, setting 
$$
\lambda = \lambda_{\alpha}(\boldsymbol{\theta}) = \frac{(1+S_R)^2}{\sigma_{\boldsymbol{\theta}}^2} \log^2 \left(1+\frac{K \, S_R}{\alpha}\right), 
$$
we get $h(\lambda_\alpha(\boldsymbol{\theta})) = \alpha$. This completes the proof.

\section{Proof of Lemma~\ref{lem:concentration-of-empirical-arm-means}}
%A proof of this result follows from a stronger result, namely Lemma~\ref{lem:concentration-of-state-action-visitations}, via an application of the Borel--Cantelli lemma.

\begin{proof}
Let 
\begin{equation}
    M_n = 
    \begin{pmatrix}
    \mathbf{1}_{\{A_0=1\}} \ \mathbf{1}_{\{A_0=2\}} \ \cdots \  \mathbf{1}_{\{A_0=K\}} \\
    \mathbf{1}_{\{A_1=1\}} \ \mathbf{1}_{\{A_1=2\}} \ \cdots \  \mathbf{1}_{\{A_1=K\}} \\
    \vdots \\
    \mathbf{1}_{\{A_n=1\}} \ \mathbf{1}_{\{A_n=2\}} \ \cdots \ \mathbf{1}_{\{A_n=K\}} \\
    \end{pmatrix}, 
    \quad
    E_n^{1} = 
    \begin{pmatrix}
    f(\bar{X}_0) \\ 
    f(\bar{X}_1) \\
    \vdots \\
    f(\bar{X}_n)
    \end{pmatrix}, 
    \quad 
    E_n^2 = 
    \begin{pmatrix}
    \eta_{\theta_1} \\ 
    \eta_{\theta_2} \\
    \vdots \\
    \eta_{\theta_K}
    \end{pmatrix}.
    \label{eq:M-n-and-E-n-k}
\end{equation}
Then, we have
\begin{align}
    M_n^\top M_n &= {\rm diag}(N_1(n), \ldots, N_K(n)),
    \label{eq:M-n-transpose-times-M-n}
\end{align}
where ${\rm diag}(\cdot)$ refers to the operator that converts an ordered set of elements into a diagonal matrix. Note that
\begin{align}
    \widehat{\eta}_{a}(n) - \eta_{\theta_a} 
    &= \frac{1}{N_a(n)} \sum_{t=0}^{n} \mathbf{1}_{\{A_t=a\}}\, (f(\bar{X}_t) - \eta_{\theta_a}) \nonumber\\
    &= \frac{1}{N_a(n)} \, (M_n^\top E_n^1 - M_n^\top M_n E_n^2)_{a} \nonumber\\
    &= \frac{1}{N_a(n)} \, (M_n^\top E_n)_{a},
    \label{eq:proof-of-concentration-of-empirical-arm-means-1}
\end{align}
where $E_n \coloneqq E_n^1 - M_n E_n^2$. Consider the event
\begin{equation}
    \mathcal{E} \coloneqq \bigg\lbrace \forall(\mathbf{d}, \mathbf{i}, a) \in \mathbb{V}, \quad \forall n \geq K, \quad N(n, \mathbf{d}, \mathbf{i}, a) \geq \left\lceil \frac{n}{\lambda_{\alpha}(\boldsymbol{\theta})} \right\rceil^{1/4} - 1 \bigg\rbrace,
        \label{eq:event-E-of-interest-proof-of-arm-means-concentration}
\end{equation}
where $\lambda_{\alpha}(\boldsymbol{\theta})$ and $\mathbb{V}$ are as defined in Lemma~\ref{lem:sufficient-exploration-of-state-actions}. Under $\mathcal{E}$, we note that for all $a \in [K]$,
\begin{align}
    N_a(n) 
    &= \sum_{(\mathbf{d}, \mathbf{i}) \in \mathbb{V}_{R,a}} N(n, \mathbf{d}, \mathbf{i}, a) \nonumber\\
    &= \sum_{(\mathbf{d}, \mathbf{i}) \in \mathbb{S}_R} N(n, \mathbf{d}, \mathbf{i}, a) \nonumber\\
    &\geq S_R \left(\frac{n^{1/4}}{\lambda_{\alpha}(\boldsymbol{\theta})^{1/4}} - 1\right) \nonumber\\
    &\geq c\, n^{1/4} \qquad \forall n \geq n_0,
    \label{eq:proof-of-concentration-of-empirical-arm-means-2}
\end{align}
where $n_0$ and $c$ are defined as
\begin{equation}
    n_0 \coloneqq \inf \left\lbrace n \geq K: n^{1/4} \geq 2\lambda_{\alpha}(\boldsymbol{\theta})^{1/4} \right\rbrace, \quad c \coloneqq \frac{S_R}{2 \lambda_{\alpha}(\boldsymbol{\theta})^{1/4}}.
    \label{eq:definition-of-n0-and-c}
\end{equation}
Eq. \eqref{eq:proof-of-concentration-of-empirical-arm-means-2} implies that $M_n^\top M_n$ is invertible for all $n \geq n_0$ under $\mathcal{E}$, in which case it follows from \eqref{eq:proof-of-concentration-of-empirical-arm-means-1} that
\begin{equation}
    \widehat{\boldsymbol{\eta}}(n) - \boldsymbol{\eta} = (M_n^\top M_n)^{-1} M_n^\top E_n \quad \forall n \geq n_0,
\end{equation}
where $\boldsymbol{\eta} \coloneqq [\eta_{\theta_a}: a \in [K]]^\top$. We then note that under $\mathcal{E}$,
\begin{align}
    \| \widehat{\boldsymbol{\eta}}(n) - \boldsymbol{\eta} \|_2
    &= \|(M_n^\top M_n)^{-1} (M_n^\top E_n)\|_2 \nonumber\\
    &= \sqrt{(M_n^\top E_n)^\top\,  (M_n^\top M_n)^{-1}\,  ((M_n^\top M_n)^{-1})^\top \,  (M_n^\top E_n)} \nonumber\\ 
    &\stackrel{(a)}{\leq} \frac{1}{\sqrt{\min_{a} N_a(n)}} \, \sqrt{(M_n^\top E_n)^\top\,  (M_n^\top M_n)^{-1} \,  (M_n^\top E_n)}\nonumber\\
    &\leq \frac{1}{\sqrt{c}\, n^{1/8}}\sqrt{(M_n^\top E_n)^\top\,  (M_n^\top M_n)^{-1} \,  (M_n^\top E_n)}
    \label{eq:proof-of-concentration-of-empirical-arm-means-3}
\end{align}
for all $n \geq n_0$, where $|| \cdot ||_2$ denotes vector $2$-norm, and $(a)$ follows by noting that each entry of $((M_n^\top M_n)^{-1})^\top = {\rm diag}\left(N_1(n)^{-1}, \ldots, N_K(n)^{-1}\right)$ may be upper bounded by $\frac{1}{\min_{a} N_a(n)}$. Because 
$$
N_a(n) + c\, n^{1/4} < 2N_a(n) \quad \forall n \geq n_0, \ \forall a \in [K],
$$
we have $2(M_n^\top M_n  + c\, n^{1/4} I_{K})^{-1} \succ (M_n^\top M_n)^{-1}$ for all $n \geq n_0$, where the notation $A\succ M$ means that $A-M$ is positive definite, and $I_K$ denotes the $K \times K$ identity matrix. Thus, it follows that under $\mathcal{E}$,
\begin{align}
    \|\widehat{\boldsymbol{\eta}}(n) - \boldsymbol{\eta} \|_2
    & < \frac{\sqrt{2}}{\sqrt{c} \,  n^{1/8}} \sqrt{(M_n^\top E_n)^\top\,  (M_n^\top M_n+c\, n^{1/4} I_K)^{-1} \,  (M_n^\top E_n)} \quad \forall \, n\geq n_0.
    \label{eq:qhat-q-for-m-2}
\end{align}
Introducing the shorthand notation $\|x\|_A = \sqrt{x^\top A x}$ for any positive definite matrix $A$, we show below that under $\mathcal{E}$, 
$$
\|M_n^\top E_n\|_{(M_n^\top M_n + c\, n^{1/4} I_K)^{-1}}=o(n^{\beta}) \quad \forall \beta \in (0,1/8).
$$

Fix $\beta\in (0,1/8)$ arbitrarily. In order to show that $\|M_n^\top E_n\|_{(M_n^\top M_n + c\, n^\alpha I_K)^{-1}}=o(n^{\beta})$ under $\mathcal{E}$, we invoke the following result from \cite{jedra2020optimal}.
\begin{proposition}(\cite[Proposition 4]{jedra2020optimal})
\label{prop:concentration-inequality}
Let $\{\mathcal{F}_n:n\geq 0\}$ be a filtration. Let $\{\mathbf{\eta}_n:n\geq 0\}$ be a real-valued stochastic process with zero mean and adapted to $\{\mathcal{F}_n\}$ (i.e., $\eta_n$ is $\mathcal{F}_n$-measurable for all $n\geq 0$). Furthermore, suppose that $\mathbb{E}[\exp(x \eta_n)|\mathcal{F}_n] \leq \exp(-x^2\sigma^2/2)$ for all $x\in \mathbb{R}$ and $n\ge 0$. Let $(\mathbf{a}_n:n\geq 0)$ be an $\mathbb{R}^d$-valued process adapted to $\{\mathcal{F}_t\}$. Let $U$ be a positive definite matrix. Then, for all $\delta>0$,
\begin{equation}
    P\left(\|Q_n^\top R_n\|^2_{(Q_n^\top Q_n+U)^{-1}} > 2\sigma^2 \, \log\left(\frac{\sqrt{{\rm det}((Q_n^\top Q_n+U)U^{-1})}}{\delta}\right)\right) \leq \delta,
\end{equation}
where $Q_n=[\mathbf{a}_0 \ \mathbf{a}_1 \ \cdots \mathbf{a}_n]^\top$ and $R_n = [\mathbf{\eta}_0 \ \mathbf{\eta}_1 \ \cdots \mathbf{\eta}_n]^\top$.
\end{proposition}
Applying Proposition \ref{prop:concentration-inequality} with the filtration in \eqref{eq:filtration}, $d=K$, $U=c\, n^{1/4} I_K$, 
$$
\mathbf{a}_t = [\mathbf{1}_{\{A_t=1\}} \ \mathbf{1}_{\{A_t=2\}} \ \cdots \ \mathbf{1}_{\{A_t=K\}}]^\top, \quad \eta_t = f(\bar{X}_t) - \sum_{a=1}^{K} \mathbf{1}_{\{A_t=a\}}\, \eta_{\theta_a}, \quad 0 \leq t \leq n,
$$ 
and $\sigma^2 = 2 \, M_f$ (noting that $|\eta_t| \leq M_f + \max_a \eta_{\theta_a} \leq 2\, M_f$ a.s. for all $t$, where $M_f=\max_{i \in \mathcal{S}} f(i)$), we get that for all $\delta>0$,
\begin{align}
    & \P_{\boldsymbol{\theta}}^{\mathcal{E}}\bigg(\frac{1}{n^\beta}\|M_n^\top E_n\|_{(M_n^\top M_n + c\, n^{1/4})^{-1}} > \frac{\sqrt{2 \, M_f}}{n^\beta} \, \sqrt{\log\left(\frac{\sqrt{{\rm det}((M_n^\top M_n+c\, n^{1/4} I_K)(c\, n^{1/4} I_K)^{-1})}}{\delta}\right)}\bigg) \leq \delta,
    \label{eq:concentration-inequality-proof-1}
\end{align}
where $\P_{\boldsymbol{\theta}}^{\mathcal{E}}(\cdot) = \P_{\boldsymbol{\theta}}(\cdot\; | \; \mathcal{E})$.
Noting that under $\mathcal{E}$,
\begin{align}
    {\rm det}((M_n^\top M_n+c\, n^{1/4} I_K)(c\, n^{1/4} I_K)^{-1}) 
    &= \prod_{a=1}^{K} \frac{N_a(n)+c\, n^{1/4}}{c\, n^{1/4}} \nonumber\\
    &\leq \prod_{a=1}^{K} \frac{n + 1 + c\, n^{1/4}}{c\, n^{1/4}}\nonumber\\
    &= \left(\frac{n^{3/4}}{c} + 1 + \frac{1}{c\, n^{1/4}}\right)^K \nonumber\\
    & \leq \left( \frac{2 n^{3/4}}{c}\right)^K
    \label{eq:concentration-inequality-proof-2}
\end{align}
for all $n \geq n_1$, where 
\begin{equation}
    n_1 \coloneqq \inf \left\lbrace n \geq n_0: 1 + \frac{1}{c\, n^{1/4}} \leq \frac{n^{3/4}}{c} \right\rbrace.
    \label{eq:n-1-definition}
\end{equation}
We then have
\begin{align}
    \P_{\boldsymbol{\theta}}^{\mathcal{E}}\left(\frac{1}{n^\beta}\|M_n^\top E_n\|_{(M_n^\top M_n + c\, n^{1/4} I_K)^{-1}} > \frac{\sqrt{2 \, M_f}}{n^\beta} \sqrt{\log\left(\frac{2^{K/2}\,  n^{3K/8}}{c^{K/2} \,  \delta}\right)}\right) \leq \delta \quad \forall n \geq n_1.
     \label{eq:concentration-inequality-proof-3}
\end{align}
Equivalently, for all $\xi > 0$, we have 
\begin{equation}
    \P_{\boldsymbol{\theta}}^{\mathcal{E}}\left(\frac{1}{n^\beta}\|M_n^\top E_n\|_{(M_n^\top M_n + c\, n^{1/4} I_K)^{-1}} > \xi \right) \leq \frac{2^{K/2}\, n^{3K/8}}{c^{K/2}} \ \exp\left(-\frac{\xi^2 \, n^{2\beta}}{2 \, M_f}\right) \quad \forall n \geq n_1.
    \label{eq:concentration-inequality-proof-4}
\end{equation}
Using \eqref{eq:concentration-inequality-proof-4} in \eqref{eq:qhat-q-for-m-2}, we get that for any $\xi>0$,
\begin{equation}
    \P_{\boldsymbol{\theta}}^{\mathcal{E}}(\|\widehat{\boldsymbol{\eta}}(n) - \boldsymbol{\eta}\| > \xi) \leq \frac{2^{K/2}\, n^{3K/8}}{c^{K/2}} \ \exp\left(-\frac{c\, \xi^2 \, n^{1/4}}{2(2 \, M_f)}\right) \quad \forall n \geq n_1.
    \label{eq:concentration-inequality-proof-5}
\end{equation}

Finally, fixing $N \geq K$ and applying \eqref{eq:concentration-inequality-proof-5} to the event $\mathcal{E}=\mathcal{E}_N$, where $\mathcal{E}_N$ is as defined in \eqref{eq:event-E-N-definition}, and denoting the constants $n_0$, $c$, and $n_1$ more explicitly as $n_0(N)$, $c_N$, and $n_1(N)$\footnote{We assume, without loss of generality, that $N^5 \geq \max\{n_0(N), n_1(N)\}$. If this is not true, we may replace $N^5, N^6$ by $N^b, N^{b+1}$ respectively, where $b$ is any integer satisfying $N^b \geq \max\{n_0(N), n_1(N)\}$.} respectively, we get
\begin{align}
    \P_{\boldsymbol{\theta}}(\overline{C_N^2(\xi)})
    &\leq \P_{\boldsymbol{\theta}}(\overline{\mathcal{E}_N}) + \P_{\boldsymbol{\theta}}(\overline{C_N^2(\xi)} \cap \mathcal{E}_N) \nonumber\\
    &\stackrel{(a)}{\leq} \frac{1}{N^2} + \sum_{n=N^5}^{N^6} \P_{\boldsymbol{\theta}}^{\mathcal{E}_N}(\|\widehat{\boldsymbol{\eta}}(n) - \boldsymbol{\eta}\| > \xi )\nonumber\\
    &\stackrel{(b)}{\leq} \frac{1}{N^2} + \sum_{n=N^5}^{N^6} \frac{2^{K/2}\, n^{3K/8}}{c_N^{K/2}} \ \exp\left(-\frac{c_N\, \xi^2 \, n^{1/4}}{2(2 \, M_f)}\right) \nonumber\\
    &\stackrel{(c)}{\leq} \frac{1}{N^2} + \frac{2^{K/2 + 2} \, K^{K/4}}{\sigma_{\boldsymbol{\theta}}^{K/4}}\, \sum_{n=N^5}^{N^6} n^{3K/8} \, N \, \exp\left(-\frac{\sqrt{\sigma_{\boldsymbol{\theta}}}\, \xi^2 \, n^{1/4}}{8\, \sqrt{K} \, (2 \, M_f) \, N}\right) \nonumber\\
    &\stackrel{(d)}{\leq} \frac{1}{N^2} + \frac{2^{K/2 + 2} \, K^{K/4}}{\sigma_{\boldsymbol{\theta}}^{K/4}}\, N^{9K/4 + 1} \, \exp\left(-\frac{\sqrt{\sigma_{\boldsymbol{\theta}}}\, \xi^2 \, N^{1/4}}{8\, \sqrt{K} \, (2 \, M_f)}\right) (N^6 - N^5 + 1) \nonumber\\
    &\leq \frac{1}{N^2} + \frac{2^{K/2 + 2} \, K^{K/4}}{\sigma_{\boldsymbol{\theta}}^{K/4}}\, N^{9K/4 + 7} \, \exp\left(-\frac{\sqrt{\sigma_{\boldsymbol{\theta}}}\, \xi^2 \, N^{1/4}}{8\, \sqrt{K} \, (2 \, M_f)}\right),
    \label{eq:concentration-inequality-proof-6}
\end{align}
where 
\begin{itemize}
    \item $(a)$ follows from noting that  $\P_{\boldsymbol{\theta}}(\overline{\mathcal{E}_N}) \leq 1/N^2$ (applying Lemma~\ref{lem:sufficient-exploration-of-state-actions} with $\alpha=1/N^2$).

    \item $(b)$ follows from \eqref{eq:concentration-inequality-proof-5} (with $c$ and $n_1$ replaced by $c_N$ and $n_1(N)$ respectively).

    \item $(c)$ follows by noting that  
    \begin{align*}
        c_N 
        &= \frac{S_R}{2 \, \lambda_{\boldsymbol{\theta}}(N)^{1/4}} \nonumber\\
        &= \frac{S_R \sqrt{\sigma_{\boldsymbol{\theta}}}}{2\, \sqrt{1+S_R}\, \sqrt{\log (1+KS_R\, N^2)}} \nonumber\\
        &\geq \frac{S_R \sqrt{\sigma_{\boldsymbol{\theta}}}}{2\, \sqrt{1+S_R} \, \sqrt{KS_R\, N^2}} \nonumber\\
        &= \frac{\sqrt{S_R \, \sigma_{\boldsymbol{\theta}}}}{2\, N\,  \sqrt{K(1+S_R)}} \nonumber\\
        &\geq \frac{\sqrt{\sigma_{\boldsymbol{\theta}}}}{2\, N \sqrt{2K}} \nonumber\\
        &\geq \frac{\sqrt{\sigma_{\boldsymbol{\theta}}}}{4\, N\,  \sqrt{K}},
    \end{align*}
    where the penultimate line above follows from using $1+S_R \leq 2\, S_R$.

    \item $(d)$ follows by noting that $n^{3K/8}<(N^6)^{3K/8}=N^{9K/4}$ and $n^{1/4} \geq N^{5/4}$ for $n \in \{N^5, \ldots, N^6\}$.
\end{itemize}
This completes the proof.
\iffalse
%{\color{black}this is a little misleading. Should use a lower case letter for the integer $|\mathbb{V}_{R, a}|$.} 

Finally, we note that
\begin{align}
    \widehat{\eta}^a(n) - \eta_{\theta_a} 
    & = \frac{1}{N_a(n)} \sum_{(\mathbf{d}, \mathbf{i}) \in \mathbb{V}_{R, a}} \ \sum_{t=K}^{n} \mathbf{1}_{\{\mathbf{d}(n)=\mathbf{d}, \ \mathbf{i}(n)=\mathbf{i}, \ A_n=a\}} \, (f(\bar{X}_t) - \eta_{\theta_a}) \nonumber\\
    &=\sum_{(\mathbf{d}, \mathbf{i}) \in \mathbb{V}_{R, a}} \frac{N(n, \mathbf{d}, \mathbf{i}, a)}{N_a(n)} \cdot \frac{1}{N(n, \mathbf{d}, \mathbf{i}, a)} \sum_{t=K}^{n} \mathbf{1}_{\{\mathbf{d}(t)=\mathbf{d}, \ \mathbf{i}(t)=\mathbf{i}, \ A_t=a\}} \, (f(\bar{X}_t) - \eta_{\theta_a}) \nonumber\\
    & = \sum_{(\mathbf{d}, \mathbf{i}) \in \mathbb{V}_{R, a}} \frac{N(n, \mathbf{d}, \mathbf{i}, a)}{N_a(n)} \cdot o(1)\nonumber\\
    & = o(1) \quad \text{a.s.}.
\end{align}
for all $a \in \mathcal{A}$. This completes the proof.
\fi
\end{proof}

%-------------------------------------------------------------------

\section{Proof of Lemma~\ref{lem:concentration-of-state-action-visitations}}
Before we begin the proof, we note the following results of interest.
\begin{lemma}\cite[Theorem 2]{schweitzer1968perturbation}
    \label{lem:condition-number}
    Let $Q_1$ (resp. $Q_2$) be the transition kernel of a Markov chain with stationary distribution $\omega_1$ (resp. $\omega_2$). Let $Z_1 \coloneqq (I-Q_1+\mathbf{1}\omega_1^\top)^{-1}$, where $I$ is the identity matrix (of the same dimension as $Q_1$ and $Q_2$), and $\mathbf{1}$ denotes the all-ones vector. Then,
    \begin{equation}
        \omega_2^\top - \omega_1^\top = \omega_2^\top [Q_2 - Q_1] \, Z_1, \quad \|\omega_2 - \omega_1\|_1 \leq \|Z_1\|_{\infty}\, \|Q_2 - Q_1\|_{\infty}.
        \label{eq:condition-number}
    \end{equation}
\end{lemma}

\begin{lemma}\cite[Lemma 23]{al2021navigating}
    \label{lem:poisson-equation}
    Let $Q_{\boldsymbol{\theta,\pi}}$ be the transition kernel of MDP $\mathcal{M}_{\boldsymbol{\theta}, R}$ under policy $\pi$. Let $\omega_{\boldsymbol{\theta}, \pi}$ denote the stationary distribution of $Q_{\boldsymbol{\theta,\pi}}$, and let 
    \begin{equation}
        \|Q_{\boldsymbol{\theta}, \pi}^n - \mathbf{1} \omega_{\boldsymbol{\theta}, \pi}^\top\|_{\infty} \leq C_{\pi} \, \rho_\pi^n \quad \forall n \in \mathbb{N}.
        \label{eq:relation-of-interest-poisson-lemma}
    \end{equation} 
    Then, for any bounded function $g$ on $\mathcal{Z}= \mathbb{S}_R \times [K]$, the function
    \begin{equation}
        \widehat{g}_\pi(z) = \sum_{n=0}^{\infty} \ \sum_{z' \in \mathcal{Z}} Q_{\boldsymbol{\theta}, \pi}^n (z'|z) \, \bigg(g(z') - \langle \omega_{\boldsymbol{\theta}, \pi}, g \rangle\bigg), \quad z \in \mathcal{Z},
        \label{eq:poisson-equation-1}
    \end{equation}
    is well-defined and satisfies the so-called {\em Poisson equation} 
    \begin{equation}
        \widehat{g}_\pi(z) - \sum_{z' \in \mathcal{Z}} Q_{\boldsymbol{\theta}, \pi}(z'|z) \, \widehat{g}_\pi(z') = g(z) - \langle \omega_{\boldsymbol{\theta}, \pi}, g \rangle \quad \forall z \in \mathcal{Z}.
        \label{eq:poisson-equation-2}
    \end{equation}
    Furthermore,
    \begin{equation}
        \|\widehat{g}_\pi\|_{\infty} \leq L_{\pi} \, \|g\|_{\infty},
        \label{eq:one-kernel}
    \end{equation}
    where $L_{\pi} \coloneqq \dfrac{C_{\pi}}{1-\rho_\pi}$, and for any pair of policies $\pi, \pi'$,
    \begin{equation}
        \max_{z \in \mathcal{Z}} \bigg\lvert \sum_{z' \in \mathcal{Z}} Q_{\boldsymbol{\theta}, \pi}(z'|z) \, \widehat{g}_{\pi}(z') - \sum_{z' \in \mathcal{Z}} Q_{\boldsymbol{\theta}, \pi'}(z'|z) \, \widehat{g}_{\pi'}(z')\bigg\rvert \leq L_{\pi} \, \|g\|_{\infty} \bigg[\|\omega_{\boldsymbol{\theta}, \pi} - \omega_{\boldsymbol{\theta}, \pi'}\|_1 + L_{\pi} \, \|Q_{\boldsymbol{\theta}, \pi} - Q_{\boldsymbol{\theta}, \pi'}\|_{\infty}\bigg].
        \label{eq:two-kernels}
    \end{equation}
\end{lemma}

\begin{proof}[Proof of Lemma~\ref{lem:concentration-of-state-action-visitations}]
    We now begin the proof of Lemma~\ref{lem:concentration-of-state-action-visitations}. Let us fix $z = (\mathbf{d}, \mathbf{i}, a) \in \mathbb{S}_R \times [K]$. Let 
    \begin{align}
        D_{\boldsymbol{\theta}, \nu}^{n} 
        &\coloneqq \frac{N(n, \mathbf{d}, \mathbf{i}, a)}{n-K+1} - \omega_{\boldsymbol{\theta}, \nu}^\star(\mathbf{d}, \mathbf{i}, a) \nonumber\\
        &= \frac{1}{n-K+1} \sum_{t=K}^{n} \mathbf{1}_{\{\mathbf{d}(t)=\mathbf{d}, \mathbf{i}(t)=\mathbf{i}, A_t=a\}} - \omega_{\boldsymbol{\theta}, \nu}^\star(\mathbf{d}, \mathbf{i}, a) \nonumber\\
        &= \frac{\sum\limits_{t=K}^{\sqrt{N}} [\mathbf{1}_{\{\mathbf{d}(t)=\mathbf{d}, \mathbf{i}(t)=\mathbf{i}, A_t=a\}} - \omega_{\boldsymbol{\theta}, \nu}^\star(\mathbf{d}, \mathbf{i}, a)]}{n-K+1} + \frac{\sum\limits_{t=\sqrt{N}+1}^{n} [\mathbf{1}_{\{\mathbf{d}(t)=\mathbf{d}, \mathbf{i}(t)=\mathbf{i}, A_t=a\}} - \omega_{\boldsymbol{\theta}, \nu}^\star(\mathbf{d}, \mathbf{i}, a)]}{n-K+1} \nonumber\\
        &= \underbrace{\frac{\sum\limits_{t=K}^{\sqrt{N}} [\mathbf{1}_{\{\mathbf{d}(t)=\mathbf{d}, \mathbf{i}(t)=\mathbf{i}, A_t=a\}} - \omega_{\boldsymbol{\theta}, \nu}^\star(\mathbf{d}, \mathbf{i}, a)]}{n-K+1}}_{D_{\boldsymbol{\theta}, \nu}^{n}(1)} + \underbrace{\frac{\sum\limits_{t=\sqrt{N}+1}^{n} [\mathbf{1}_{\{\mathbf{d}(t)=\mathbf{d}, \mathbf{i}(t)=\mathbf{i}, A_t=a\}} - \omega_{\boldsymbol{\theta}, t-1}^\star(\mathbf{d}, \mathbf{i}, a)]}{n-K+1}}_{D_{\boldsymbol{\theta}, \nu}^{n}(2)} \nonumber\\
        &\hspace{5cm} + \underbrace{\frac{\sum\limits_{t=\sqrt{N}+1}^{n} [\omega_{\boldsymbol{\theta}, t-1}^\star(\mathbf{d}, \mathbf{i}, a) - \omega_{\boldsymbol{\theta}, \nu}^\star(\mathbf{d}, \mathbf{i}, a)]}{n-K+1}}_{D_{\boldsymbol{\theta}, \nu}^{n}(3)}, 
        \label{eq:proof-of-state-action-concentration-1} 
    \end{align}
    where for each $t \geq \sqrt{N}+1$, the quantity $\omega_{\boldsymbol{\theta}, t-1}^\star$ denotes the stationary distribution of the kernel $Q_{\boldsymbol{\theta}, \pi}$ under the policy $\pi = \pi_{t-1} = \varepsilon_{t-1} \, \pi^{\text{\rm unif}} + (1-\varepsilon_{t-1}) \, \pi_{\widehat{\boldsymbol{\theta}}(t-2)}^\eta$ obtained by using a fixed $\nu_{t-2}^\star \in \mathcal{W}^\star(\widehat{\boldsymbol{\theta}}(t-2))$ in \eqref{eq:lambda-n-definition}); such a policy is indeed ergodic thanks to Lemma~\ref{lem:ergodicity-of-MDP-under-unif-policy}. We now upper bound each of the terms in \eqref{eq:proof-of-state-action-concentration-1} individually. First, note that for all $n \geq N^{3/4}$ and for all $N \geq (1/\xi)^4$, we have
    \begin{align}
        |D_{\boldsymbol{\theta}, \nu}^{n}(1)| 
        &\leq \frac{\sqrt{N}-K+1}{n-K+1} \nonumber\\
        &\leq \frac{\sqrt{N}}{n} \nonumber\\
        &\leq \frac{1}{\sqrt{N}} \nonumber\\
        &\leq \xi.
        \label{eq:proof-of-state-action-concentration-2}
    \end{align}
    Next, we note using Lemma~\ref{lem:condition-number} that
    \begin{align}
        |D_{\boldsymbol{\theta}, \nu}^{n}(3)|
        &\leq \frac{\sum\limits_{t=\sqrt{N}+1}^{n} \bigg\lvert \omega_{\boldsymbol{\theta}, t-1}^\star(\mathbf{d}, \mathbf{i}, a) - \omega_{\boldsymbol{\theta}, \nu}^\star(\mathbf{d}, \mathbf{i}, a) \bigg\rvert}{n-K+1} \nonumber\\
        &\leq \frac{\sum\limits_{t=\sqrt{N}+1}^{n} 
        \| \omega_{\boldsymbol{\theta}, t-1}^\star - \omega_{\boldsymbol{\theta}, \nu}^\star \|_{1}}{n-K+1} \nonumber\\
        &\stackrel{(a)}{\leq} \kappa_{\boldsymbol{\theta}, \nu} \, \frac{\sum\limits_{t=\sqrt{N}+1}^{n} 
        \| Q_{\boldsymbol{\theta}, \pi_{t-1}} - Q_{\boldsymbol{\theta}, \pi_{\boldsymbol{\theta}}^{\eta}(\nu)} \|_{\infty}}{n-K+1} \nonumber\\
        &\stackrel{(b)}{=} \kappa_{\boldsymbol{\theta}, \nu} \, \frac{\sum\limits_{t=\sqrt{N}+1}^{n} 
        \| \varepsilon_{t-1} \, Q_{\boldsymbol{\theta}, \pi^{\text{\rm unif}}} + (1-\varepsilon_{t-1}) \,  Q_{\boldsymbol{\theta}, \pi_{\widehat{\boldsymbol{\theta}}(t-2)}^{\eta}} - Q_{\boldsymbol{\theta}, \pi_{\boldsymbol{\theta}}^{\eta}(\nu)} \|_{\infty}}{n-K+1} \nonumber\\
        &\stackrel{(c)}{\leq} \kappa_{\boldsymbol{\theta}, \nu} \, \frac{\sum\limits_{t=\sqrt{N}+1}^{n} 
        \varepsilon_{t-1} + \| Q_{\boldsymbol{\theta}, \pi_{\widehat{\boldsymbol{\theta}}(t-2)}^{\eta}} - Q_{\boldsymbol{\theta}, \pi_{\boldsymbol{\theta}}^{\eta}(\nu)} \|_{\infty}}{n-K+1} \nonumber\\
        &\stackrel{(d)}{\leq} \kappa_{\boldsymbol{\theta}, \nu} \, \frac{\sum\limits_{t=\sqrt{N}+1}^{n} 
        \varepsilon_{t-1} + \| \pi_{\widehat{\boldsymbol{\theta}}(t-2)}^\eta - \pi_{\boldsymbol{\theta}}^{\eta}(\nu) \|_{\infty}}{n-K+1}, 
        \label{eq:proof-of-state-action-concentration-3}
    \end{align}
    where $\kappa_{\boldsymbol{\theta}, \nu} = \|(I+Q_{\boldsymbol{\theta}, \pi_{\boldsymbol{\theta}}^{\eta}(\nu)} + \mathbf{1}\, (\omega_{\boldsymbol{\theta}, \nu}^\star)^\top)^{-1}\|_{\infty}$, $(a)$ above follows from Lemma~\ref{lem:concentration-of-state-action-visitations} with $\pi_{\boldsymbol{\theta}}^{\eta}(\nu)$ denoting the quantity in \eqref{eq:lambda-n-definition} with (i) $\widehat{\boldsymbol{\theta}}(n)$ replaced by $\boldsymbol{\theta}$, and (ii) $\nu_n^\star$ replaced by $\nu$, $(b)$ above follows from the definition of the transition kernel $Q_{\boldsymbol{\theta}, \pi_{t-1}}$, $(c)$ above follows from the observation that
    \begin{align}
        \varepsilon_{t-1} \, Q_{\boldsymbol{\theta}, \pi^{\text{\rm unif}}} + (1-\varepsilon_{t-1}) \,  Q_{\boldsymbol{\theta}, \pi_{\widehat{\boldsymbol{\theta}}(t-2)}^{\eta}}
        &\leq \varepsilon_{t-1} + Q_{\boldsymbol{\theta}, \pi_{\widehat{\boldsymbol{\theta}}(t-2)}^{\eta}},
        \label{eq:temporary-equation}
    \end{align}
    and $(d)$ above follows from the observation that
    \begin{align*}
        \| Q_{\boldsymbol{\theta}, \pi_{\widehat{\boldsymbol{\theta}}(t-2)}^{\eta}} - Q_{\boldsymbol{\theta}, \pi_{\boldsymbol{\theta}}^{\eta}(\nu)} \|_{\infty}
        &= \| Q_{\boldsymbol{\theta}, R} \, [\pi_{\widehat{\boldsymbol{\theta}}(t-2)}^{\eta} - \pi_{\boldsymbol{\theta}}^{\eta}(\nu)]\|_{\infty} \nonumber\\
        &\leq \| \pi_{\widehat{\boldsymbol{\theta}}(t-2)}^\eta - \pi_{\boldsymbol{\theta}}^{\eta}(\nu) \|_{\infty},
    \end{align*}
    where the last line above is an artefact of $\|Q_{\boldsymbol{\theta}, R}\| \leq 1$. Thanks to Lemma~\ref{lem:sufficient-exploration-of-state-actions}, we have $\boldsymbol{\theta}(n) \to \boldsymbol{\theta}$ almost surely as $n \to \infty$. We then note that by virtue of the upper-hemicontinuity of the map $\boldsymbol{\lambda} \to \mathcal{W}^\star(\boldsymbol{\lambda})$, there exists $N_1=N_1(\xi)>0$ and $\rho(\xi)>0$ such that for all $N \geq N_1$ and for all $t \geq \sqrt{N}+1$, we have
    \begin{align}
        &\| \pi_{\widehat{\boldsymbol{\theta}}(t-2)}^\eta - \pi_{\boldsymbol{\theta}}^{\eta}(\nu) \|_{\infty} \nonumber\\
        &= \bigg\| \frac{\eta\, \nu_{\widehat{\boldsymbol{\theta}}(t-2)}^{\text{\rm unif}}(\mathbf{d}, \mathbf{i}, a) + (1-\eta)\, \nu_{t-2}^\star(\mathbf{d}, \mathbf{i}, a)}{\eta\, \mu_{\widehat{\boldsymbol{\theta}}(t-2)}^{\text{\rm unif}}(\mathbf{d}, \mathbf{i}) + (1-\eta)\, \sum_{a'=1}^{K} \nu_{t-2}^\star(\mathbf{d}, \mathbf{i}, a')} - \frac{\eta\, \nu_{{\boldsymbol{\theta}}}^{\text{\rm unif}}(\mathbf{d}, \mathbf{i}, a) + (1-\eta)\, \nu(\mathbf{d}, \mathbf{i}, a)}{\eta\, \mu_{{\boldsymbol{\theta}}}^{\text{\rm unif}}(\mathbf{d}, \mathbf{i}) + (1-\eta)\, \sum_{a'=1}^{K} \nu(\mathbf{d}, \mathbf{i}, a')} \bigg\| \nonumber\\
        &\leq \rho(\xi).
        \label{eq:proof-of-state-action-concentration-4}
    \end{align}
    Using \eqref{eq:proof-of-state-action-concentration-4} in \eqref{eq:proof-of-state-action-concentration-3}, we get
    \begin{equation}
        |D_{\boldsymbol{\theta}, \nu}^{n}(3)| \leq \kappa_{\boldsymbol{\theta}, \nu} \, \rho(\xi) \quad \forall n \geq \sqrt{N}+1, ~\forall N \geq N_1(\xi).
        \label{eq:proof-of-state-action-concentration-5}
    \end{equation}

    Lastly, in order to bound $D_{\boldsymbol{\theta}, \nu}^n(2)$, we use the function $\widehat{g}_{t-1} = \widehat{g}_{\pi_{t-1}}$, the solution to the Poisson equation in \eqref{eq:poisson-equation-2} with $g(z) = \mathbf{1}_{\{(\mathbf{d}, \mathbf{i}, a)\}}(z)$. From Lemma~\ref{lem:poisson-equation}, we know that $\widehat{g}_{t-1}$ exists, and using the shorthand notation $z_t=(\mathbf{d}(t), \mathbf{i}(t), A_t)$, we may express $D_{\boldsymbol{\theta}, \nu}^{n}(2)$ as
    \begin{align}
        D_{\boldsymbol{\theta}, \nu}^n(2) 
        &= \frac{\sum\limits_{t=\sqrt{N}+1}^{n} [\widehat{g}_{t-1}(z_t) - Q_{\boldsymbol{\theta}, \pi_{t-1}} \, \widehat{g}_{t-1}(z_t)]}{n-K+1} \nonumber\\
        &= D_{\boldsymbol{\theta}, \nu}^n(2,1) + D_{\boldsymbol{\theta}, \nu}^n(2,2) + D_{\boldsymbol{\theta}, \nu}^n(2,3),
        \label{eq:proof-of-state-action-concentration-6}
    \end{align}
    where 
    \begin{align}
        D_{\boldsymbol{\theta}, \nu}^n(2,1) 
        &\coloneqq \frac{\sum\limits_{t=\sqrt{N}+1}^{n} [\widehat{g}_{t-1}(z_t) - Q_{\boldsymbol{\theta}, \pi_{t-1}} \, \widehat{g}_{t-1}(z_{t-1})]}{n-K+1}, \label{eq:D-2-1} \\
        D_{\boldsymbol{\theta}, \nu}^n(2,2) 
        &\coloneqq \frac{\sum\limits_{t=\sqrt{N}+1}^{n} [Q_{\boldsymbol{\theta}, \pi_{t}}\, \widehat{g}_{t}(z_t) - Q_{\boldsymbol{\theta}, \pi_{t-1}} \, \widehat{g}_{t-1}(z_{t})]}{n-K+1}, \label{eq:D-2-2} \\
        D_{\boldsymbol{\theta}, \nu}^n(2,3) 
        &\coloneqq \dfrac{Q_{\boldsymbol{\theta}, \pi_{\sqrt{N}}}\, \widehat{g}_{\sqrt{N}}(z_{\sqrt{N}}) - Q_{\boldsymbol{\theta}, \pi_{n}}\, \widehat{g}_{n}(z_n)}{n-K+1}, \label{eq:D-2-3} 
    \end{align}

    \textbf{Bounding $D_{\boldsymbol{\theta}, \nu}^n(2,1)$:} To bound \eqref{eq:D-2-1}, we note that $G_n = (n-K+1) \, D_{\boldsymbol{\theta}, \nu}^n(2,1)$ is a martingale. Furthermore, by Lemma~\ref{lem:poisson-equation}, we have 
    \begin{align}
        |G_n - G_{n-1}| 
        &= \bigg\lvert \widehat{g}_{n-1}(z_n) - Q_{\boldsymbol{\theta}, \pi_{n-1}} \, \widehat{g}_{n-1}(z_{n-1}) \bigg\rvert \nonumber\\
        &\leq 2\, \|\widehat{g}_{n-1}\|_{\infty} \nonumber\\
        &\leq 2 \, \frac{C_{n-1}}{1-\rho_{n-1}},
        \label{eq:proof-of-state-action-concentration-7}
    \end{align}
    where $C_{n-1}$ and $\rho_{n-1}$ are constants as defined in Lemma~\ref{lem:bounding-infinity-norm-between-matrices}. For all for all $n \geq \sqrt{N}+1$ and for all $N \geq \max\{N_1(\xi) + 1, (1/\xi)^{4(1+S_R)}\}$, where $N_1(\xi)$ is as defined above, we note that
    \begin{align}
        &\varepsilon_{n-1} 
        = (n-1)^{-\frac{1}{2(1+S_R)}} \leq N^{-\frac{1}{4(1+S_R)}} \leq \xi, \label{eq:varepsilon-t-1-condition} \\
        &\| \pi_{\widehat{\boldsymbol{\theta}}(n-2)}^\eta - \pi_{\boldsymbol{\theta}}^{\eta}(\nu) \|_{\infty} \leq \rho(\xi), \label{eq:pi-theta-t-2-condition} \\
        & \|\omega_{\boldsymbol{\theta}, n-1}^{\star} - \omega_{\boldsymbol{\theta}, \nu}^\star \|_1 \leq \varepsilon_{n-1} + \|\pi_{\widehat{\boldsymbol{\theta}}(n-2)}^\eta - \pi_{\boldsymbol{\theta}}^{\eta}(\nu) \|_{\infty}\| \leq 2\,\xi. \label{eq:omega-condition}
    \end{align}
    Therefore, we have
    \begin{align}
        \frac{C_{n-1}}{1-\rho_{n-1}} 
        &= \frac{2/\bar{\sigma}(\varepsilon_{n-1}, \pi_{\widehat{\boldsymbol{\theta}}(n-2)}^\eta, \omega_{\boldsymbol{\theta}, n-1}^\star)}{1-\bar{\sigma}(\varepsilon_{n-1}, \pi_{\widehat{\boldsymbol{\theta}}(n-2)}^\eta, \omega_{\boldsymbol{\theta}, n-1}^\star)} \nonumber\\
        &\leq L_{\xi} \coloneqq \max_{\substack{
        \varepsilon: |\varepsilon| \leq \xi \\
        \pi: \|\pi - \pi_{\boldsymbol{\theta}}^{\eta}(\nu)\|_{\infty} \leq \rho(\xi) \\
        \omega: \|\omega - \omega_{\boldsymbol{\theta}, \nu}^\star\|_{1} \leq 2\, \xi}} \frac{2/\bar{\sigma}(\varepsilon, \pi, \omega)}{1-\bar{\sigma}(\varepsilon, \pi, \omega)}.
        \label{eq:proof-of-state-action-concentration-8}
    \end{align}
    In \eqref{eq:proof-of-state-action-concentration-8}, $\bar{\sigma}$ is as defined in \eqref{eq:sigma-bar-definition}. Using \eqref{eq:proof-of-state-action-concentration-7}, \eqref{eq:proof-of-state-action-concentration-8}, and the Azuma-Hoeffding inequality for bounded martingale sequences, we get
    \begin{align}
        \P_{\boldsymbol{\theta}}\bigg(|D_{\boldsymbol{\theta}, \nu}^n(2,1)| \geq 2\, L_{\xi} \, \xi\bigg)
        &= \P_{\boldsymbol{\theta}}\bigg(|G_n| \geq (n-K+1) \, 2 \, L_{\xi} \, \xi\bigg) \nonumber\\
        &= O\bigg(\exp(-n \, \xi^2)\bigg) \quad \forall n \geq \sqrt{N}+1, ~\forall N \geq \max\{N_1(\xi), (1/\xi)^{4(1+S_R)}\}.
        \label{eq:bound-for-D-2-1}
    \end{align}

    \textbf{Bounding $D_{\boldsymbol{\theta}, \nu}^n(2,2)$:} Writing $L_t = C_t/(1-\rho_t)$ for all $t$, we note from Lemma~\ref{lem:poisson-equation} that
    \begin{align}
        |D_{\boldsymbol{\theta}, \nu}^n(2,2)|
        &= \frac{\sum\limits_{t=\sqrt{N}+1}^{n} L_t \bigg[ \|\omega_{\boldsymbol{\theta}, t}^{\star} - \omega_{\boldsymbol{\theta}, t-1}^{\star}\|_1 + L_{t-1} \, \|Q_{\boldsymbol{\theta}, \pi_{t}} - Q_{\boldsymbol{\theta}, \pi_{t-1}}\|_{\infty} \bigg]}{n-K+1} \nonumber\\
        &\leq \frac{\sum\limits_{t=\sqrt{N}+1}^{n} L_{\xi} \bigg[ \|\omega_{\boldsymbol{\theta}, t}^{\star} - \omega_{\boldsymbol{\theta}, \nu}^\star\|_1 + \|\omega_{\boldsymbol{\theta}, t-1}^{\star} - \omega_{\boldsymbol{\theta}, \nu}^\star\|_1 + L_{\xi} \, \|Q_{\boldsymbol{\theta}, \pi_{t}} - Q_{\boldsymbol{\theta}, \pi_{t-1}}\|_{\infty} \bigg]}{n-K+1}.
        \label{eq:proof-of-state-action-concentration-9}
    \end{align}
    We then note that
    \begin{align}
        \|Q_{\boldsymbol{\theta}, \pi_{t}} - Q_{\boldsymbol{\theta}, \pi_{t-1}}\|_{\infty}
        &\leq \|Q_{\boldsymbol{\theta}, \pi_{t}} - Q_{\boldsymbol{\theta}, \pi_{\boldsymbol{\theta}}^\eta(\nu)}\| + \|Q_{\boldsymbol{\theta}, \pi_{t-1}} - Q_{\boldsymbol{\theta}, \pi_{\boldsymbol{\theta}}^\eta(\nu)}\| \nonumber\\
        &\leq \varepsilon_{t-1} + \varepsilon_t + \|\pi_{\widehat{\boldsymbol{\theta}}(t-1)}^{\eta} - \pi_{\boldsymbol{\theta}}^\eta(\nu)\|_{\infty} + \|\pi_{\widehat{\boldsymbol{\theta}}(t-2)}^{\eta} - \pi_{\boldsymbol{\theta}}^\eta(\nu)\|_{\infty} \nonumber\\
        &\leq 2\, \varepsilon_t + 2\, \xi \quad \forall t \geq N_1(\xi)+1 \nonumber\\
        &\leq 2\, t^{-\frac{1}{2(1+S_R)}} + 2\, \xi \quad \forall t \geq N_1(\xi)+1 \nonumber\\
        &\leq 4\xi \quad \forall t \geq \max\{N_1(\xi)+1, (1/\xi)^{2(1+S_R)}\},
        \label{eq:proof-of-state-action-concentration-10}
    \end{align}
    where the second line follows from \eqref{eq:temporary-equation}, and the third line follows from using $\varepsilon_{t-1} > \varepsilon_t$. Plugging \eqref{eq:proof-of-state-action-concentration-10} in \eqref{eq:proof-of-state-action-concentration-9}, we get
    \begin{align}
        |D_{\boldsymbol{\theta}, \nu}^n(2,2)| 
        &\leq \frac{\sum\limits_{t=\sqrt{N}+1}^{n} L_{\xi} \bigg[ \|\omega_{\boldsymbol{\theta}, t}^{\star} - \omega_{\boldsymbol{\theta}, \nu}^\star\|_1 + \|\omega_{\boldsymbol{\theta}, t-1}^{\star} - \omega_{\boldsymbol{\theta}, \nu}^\star\|_1 + 4 \, L_{\xi} \, \xi \bigg]}{n-K+1} \nonumber\\
        &\leq L_{\xi} \bigg[2\, \kappa_{\boldsymbol{\theta}, \nu} \, \rho(\xi) + 4 \, L_{\xi} \, \xi\bigg],
        \label{eq:proof-of-state-action-concentration-11}
    \end{align}
    where the last line above follows from using the chain of inequalities in \eqref{eq:proof-of-state-action-concentration-3} along with \eqref{eq:temporary-equation} and \eqref{eq:proof-of-state-action-concentration-4}. 

    \textbf{Bounding $D_{\boldsymbol{\theta}, \nu}^n(2,3)$:} To bound \eqref{eq:D-2-3}, we note that
    \begin{align}
        |D_{\boldsymbol{\theta}, \nu}^n(2,3)|
        &\leq \bigg\lvert \dfrac{Q_{\boldsymbol{\theta}, \pi_{\sqrt{N}}}\, \widehat{g}_{\sqrt{N}}(z_{\sqrt{N}}) - Q_{\boldsymbol{\theta}, \pi_{n}}\, \widehat{g}_{n}(z_n)}{n-K+1} \bigg\rvert \nonumber\\
        &\leq \dfrac{\| \widehat{g}_{\sqrt{N}}\|_{\infty} + \| \widehat{g}_{\sqrt{n}}\|_{\infty}}{n-K+1} \nonumber\\
        &\leq \frac{L_{\sqrt{N}} + L_{\sqrt{n}}}{n-K+1} \nonumber\\
        &\leq \frac{2\, L_{\xi}}{n-K+1} \nonumber\\
        &\leq 2 \, L_{\xi} \, \xi \quad \forall n \geq 1/\xi + K - 1,
        \label{eq:proof-of-state-action-concentration-12}
    \end{align}
    where the second line above follows from using $\|Q_{\boldsymbol{\theta}, \pi_{\sqrt{N}}}\|_{\infty} \leq 1$, $\|Q_{\boldsymbol{\theta}, \pi_{\sqrt{n}}}\|_{\infty} \leq 1$, and the second line above follows from Lemma~\ref{lem:poisson-equation}. 

    Combining all of the results above, we get for every $(\mathbf{d}, \mathbf{i}, a) \in \mathbb{S}_R \times [K]$,
    \begin{align}
        \P_{\boldsymbol{\theta}}\left(\left\lvert  \frac{N(n, \mathbf{d}, \mathbf{i}, a)}{n-K+1} - \omega_{\boldsymbol{\theta}, \nu}^\star(\mathbf{d}, \mathbf{i}, a)\right\rvert > K_{\xi} (\boldsymbol{\theta}, \nu) \, \xi \right) = O\bigg(\exp \left(-n \xi^2\right)\bigg),
        \label{eq:for-every-dia}
    \end{align}
    where $K_{\xi}(\boldsymbol{\theta}, \nu) \coloneqq 1 + \kappa_{\boldsymbol{\theta}, \nu} \, \frac{\rho(\xi)}{\xi} + 4\, L_{\xi} + L_{\xi}\bigg[2 \, \kappa_{\boldsymbol{\theta}, \nu} \, \frac{\rho(\xi)}{\xi} + 4\, L_{\xi} \bigg]$. Noting that 
    $$
    \lim_{\xi \downarrow 0} L_{\xi} = \frac{2/\bar{\sigma}(0, \pi_{\boldsymbol{\theta}}(\nu), \omega_{\boldsymbol{\theta}, \nu}^\star)}{1-\bar{\sigma}(0, \pi_{\boldsymbol{\theta}}(\nu), \omega_{\boldsymbol{\theta}, \nu}^\star)} < +\infty, \quad \lim_{\xi \downarrow 0} \frac{\rho(\xi)}{\xi} = 0,
    $$
    we get $\limsup_{\xi \downarrow 0} K_{\xi}(\boldsymbol{\theta}, \nu)< +\infty$. The final result in \eqref{eq:concentration-of-state-action-visitations} follows from \eqref{eq:for-every-dia} via union bound.
\end{proof}

%---------------------------------------------------------------------------------------

\section{Proof of Proposition \ref{prop:stop-in-finite-time-and-error-prob-less-than-delta}}
\label{appndx:stop-in-finite-time-and-error-prob-less-than-delta}
Our proof is inspired by the proof of Lemma~15 in \cite{al2021navigating}, which in turn is inspired by the proof of Proposition~1 in \cite{jonsson2020planning} that is based on the idea of constructing mixture martingales. Note that
\begin{align}
    & \P_{\boldsymbol{\theta}}(\tau < \infty, \ \eta_{\widehat{a}} < \eta_{a^\star(\boldsymbol{\theta})}) \nonumber\\
    &= \P_{\boldsymbol{\theta}}(\exists n \geq K: \tau_\pi=n, \ \eta_{\widehat{a}} < \eta_{a^\star(\boldsymbol{\theta})}) \nonumber\\
    &\leq \P_{\boldsymbol{\theta}}(\exists n \geq K: Z(n) \geq \zeta(n, \delta), \ \eta_{\widehat{a}} < \eta_{a^\star(\boldsymbol{\theta})}) \nonumber\\
    &= \P_{\boldsymbol{\theta}} \left(\exists n \geq K: \inf_{\boldsymbol{\lambda} \in \textsc{Alt}(\widehat{\boldsymbol{
    \theta}}(n)} \sum_{(\mathbf{d}, \mathbf{i}) \in \mathbb{S}_R} \ \sum_{a = 1}^{K} N(n, \mathbf{d}, \mathbf{i}, a) \, D_{\text{\rm KL}}(\widehat{Q}_n(\cdot \mid\mathbf{d}, \mathbf{i}, a) \| Q_{\boldsymbol{\lambda}, R}(\cdot \mid\mathbf{d}, \mathbf{i}, a)) \geq \zeta(n, \delta), \ \eta_{\widehat{a}} < \eta_{a^\star(\boldsymbol{\theta})} \right) \nonumber\\ 
    &\stackrel{(a)}{\leq} \P_{\boldsymbol{\theta}} \left(\exists n \geq K: \sum_{(\mathbf{d}, \mathbf{i}) \in \mathbb{S}_R} \ \sum_{a = 1}^{K} N(n, \mathbf{d}, \mathbf{i}, a) \, D_{\text{\rm KL}}(\widehat{Q}_n(\cdot \mid\mathbf{d}, \mathbf{i}, a) \| Q_{\boldsymbol{\theta}, R}(\cdot \mid\mathbf{d}, \mathbf{i}, a)) \geq \zeta(n, \delta), \ \eta_{\widehat{a}} < \eta_{a^\star(\boldsymbol{\theta})} \right) \nonumber\\
    &\leq \P_{\boldsymbol{\theta}} \left(\exists n \geq K: \sum_{(\mathbf{d}, \mathbf{i}) \in \mathbb{S}_R} \ \sum_{a = 1}^{K} N(n, \mathbf{d}, \mathbf{i}, a) \, D_{\text{\rm KL}}(\widehat{Q}_n(\cdot \mid\mathbf{d}, \mathbf{i}, a) \| Q_{\boldsymbol{\theta}, R}(\cdot \mid\mathbf{d}, \mathbf{i}, a)) \geq \zeta(n, \delta)\right),
    \label{eq:proof-of-delta-PAC-1}
\end{align}
where $(a)$ above follows by noting that $\boldsymbol{\theta} \in \textsc{Alt}(\widehat{\boldsymbol{\theta}}(n))$ under the event $\{\eta_{\widehat{a}} < \eta_{a^\star(\boldsymbol{\theta})}\}$. Below, we construct a martingale $\{M_n\}_{n=K}^{\infty}$ such that for all $n \geq K$,
\begin{align}
    M_n 
    &\geq \exp \bigg( \sum_{(\mathbf{d}, \mathbf{i}) \in \mathbb{S}_R} \ \sum_{a = 1}^{K} N(n, \mathbf{d}, \mathbf{i}, a) \, D_{\text{\rm KL}}(\widehat{Q}_n(\cdot \mid\mathbf{d}, \mathbf{i}, a) \| Q_{\boldsymbol{\theta}, R}(\cdot \mid\mathbf{d}, \mathbf{i}, a)) \nonumber\\
    &\hspace{4cm} - (S_R - 1)\, \sum_{(\mathbf{d}, \mathbf{i}) \in \mathbb{S}_R} \ \sum_{a=1}^{K} \, \log\left( e \left[ 1 + \frac{N(n, \mathbf{d}, \mathbf{i}, a)}{S_R-1}\right]\right)\bigg) \quad \text{almost surely}.
    \label{eq:proof-of-delta-PAC-2}
\end{align}
Then, using Doob's maximal inequality, we argue that
\begin{equation}
    \P_{\boldsymbol{\theta}}(\exists n \geq K: M_n > 1/\delta) \leq \delta,
    \label{eq:proof-of-delta-PAC-3}
\end{equation}
which straightforwardly proves the desired result.

\vspace{0.3cm}

\textbf{Preliminary notations:} 

We start with some notations. Given $\boldsymbol{\gamma} \in \mathbb{R}^{S_R-1}$ and a probability distribution $\boldsymbol{p} = (p_1, \ldots, p_{S_R})$ on $\mathbb{S}_R$, let
\begin{equation}
    \langle \boldsymbol{\gamma}, \boldsymbol{p} \rangle \coloneqq p_{S_R} + \sum_{s=1}^{S_R-1} p_s \, \lambda_s, \quad \phi_{\boldsymbol{p}}(\boldsymbol{\gamma}) \coloneqq \log \left(p_{S_R} + \sum_{s=1}^{S_R-1} p_s\, e^{\lambda_s}\right).
    \label{eq:phi-p-definition}
\end{equation}
Given $(\mathbf{d}, \mathbf{i}, a) \in \mathbb{S}_R$, let $\phi_{\boldsymbol{\theta}, \mathbf{d}, \mathbf{i}, a}(\boldsymbol{\gamma})=\phi_{Q_{\boldsymbol{\theta}, R}(\cdot \mid\mathbf{d}, \mathbf{i}, a)}(\boldsymbol{\gamma})$. Define $N(K-1, \mathbf{d}, \mathbf{i}, a)=0$ for all $(\mathbf{d}, \mathbf{i}, a) \in \mathbb{S}_R \times [K]$.

\vspace{0.3cm}

\textbf{Construction of a martingale for each state-action:}

Fix $(\mathbf{d}, \mathbf{i}, a) \in \mathbb{S}_R \times [K]$ and $\boldsymbol{\gamma} \in \mathbb{R}^{S_R-1}$. Let
\begin{equation}
    M_n^{\boldsymbol{\gamma}}(\mathbf{d}, \mathbf{i}, a) \coloneqq \exp \left( N(n, \mathbf{d}, \mathbf{i}, a) \, \left[ \langle \boldsymbol{\gamma}, \widehat{Q}_n(\cdot \mid\mathbf{d}, \mathbf{i}, a) \rangle - \phi_{\boldsymbol{\theta}, \mathbf{d}, \mathbf{i}, a}(\boldsymbol{\gamma}) \right] \right), \quad n \geq K-1,
    \label{eq:M-n-gamma-martingale}
\end{equation}
where $\widehat{Q}_n$ is as defined in \eqref{eq:categorical-distribution}. Then, $\{M_n^{\boldsymbol{\gamma}}(\mathbf{d}, \mathbf{i}, a)\}_{n=K-1}^{\infty}$ is a martingale with respect to the filtration defined in \eqref{eq:filtration}. Indeed, using the shorthand notation $z_n=(\mathbf{d}(n), \mathbf{i}(n), A_n)$, we note that almost surely,
\begin{align}
    & \mathbb{E}_{\boldsymbol{\theta}}[M_n^{\boldsymbol{\gamma}}(\mathbf{d}, \mathbf{i}, a) | \mathcal{F}_{n-1}, z_n = (\mathbf{d}, \mathbf{i}, a)] \\
    &= \mathbb{E}_{\boldsymbol{\theta}}\left[\exp \left( N(n, \mathbf{d}, \mathbf{i}, a) \, \left[ \langle \boldsymbol{\gamma}, \widehat{Q}_n(\cdot \mid\mathbf{d}, \mathbf{i}, a) \rangle - \phi_{\boldsymbol{\theta}, \mathbf{d}, \mathbf{i}, a}(\boldsymbol{\gamma}) \right] \right) \bigg\vert \mathcal{F}_{n-1}, z_n = (\mathbf{d}, \mathbf{i}, a)\right] \nonumber\\
    &= \mathbb{E}_{\boldsymbol{X} \sim Q_{\boldsymbol{\theta}, R}(\cdot \mid\mathbf{d}, \mathbf{i}, a)}\bigg[\exp \bigg( (N(n-1, \mathbf{d}, \mathbf{i}, a)+1) \, \bigg[ \left\langle \boldsymbol{\gamma}, \frac{N(n-1, \mathbf{d}, \mathbf{i}, a) \, \widehat{Q}_{n-1}(\cdot \mid\mathbf{d}, \mathbf{i}, a) + \boldsymbol{X}}{N(n-1, \mathbf{d}, \mathbf{i}, a) + 1} \right\rangle 
    \nonumber\\*
    &\hspace{12cm} - \phi_{\boldsymbol{\theta}, \mathbf{d}, \mathbf{i}, a}(\boldsymbol{\gamma}) \bigg] \bigg) \bigg\vert \mathcal{F}_{n-1}\bigg] \nonumber\\
    &= \mathbb{E}_{\boldsymbol{X} \sim Q_{\boldsymbol{\theta}, R}(\cdot \mid\mathbf{d}, \mathbf{i}, a)}\left[M_{n-1}^{\boldsymbol{\gamma}}(\mathbf{d}, \mathbf{i}, a) \exp\left( \langle \boldsymbol{\gamma}, \boldsymbol{X} \rangle - \phi_{\mathbf{d}, \mathbf{i}, a}(\boldsymbol{\gamma}) \right) \bigg\vert \mathcal{F}_{n-1}\right] \nonumber\\ 
    &= M_{n-1}^{\boldsymbol{\gamma}}(\mathbf{d}, \mathbf{i}, a).
    \label{eq:proof-of-delta-PAC-4}
\end{align}
In the above set of equalities, $\boldsymbol{X}=[\mathbf{1}_{\{\mathbf{d}(n+1)=\mathbf{d}', \mathbf{i}(n+1)=\mathbf{i}'\}}: (\mathbf{d}', \mathbf{i}') \in \mathbb{S}_R]^\top$, and the last line follows by noting that
\begin{equation}
    \log \mathbb{E}_{\boldsymbol{X} \sim Q_{\boldsymbol{\theta}, R}(\cdot \mid\mathbf{d}, \mathbf{i}, a)}[\exp(\langle \boldsymbol{\gamma}, \boldsymbol{X} \rangle)] = \phi_{\mathbf{d}, \mathbf{i}, a}(\boldsymbol{\gamma}).
    \label{eq:proof-of-delta-PAC-5}
\end{equation}
When $z_n \neq (\mathbf{d}, \mathbf{i}, a)$, we have $N(n, \mathbf{d}, \mathbf{i}, a) = N(n-1, \mathbf{d}, \mathbf{i}, a)$ and $\widehat{Q}_n(\cdot \mid\mathbf{d}, \mathbf{i}, a) = \widehat{Q}_{n-1}(\cdot \mid\mathbf{d}, \mathbf{i}, a)$, and therefore $\mathbb{E}_{\boldsymbol{\theta}}[M_n^{\boldsymbol{\gamma}}(\mathbf{d}, \mathbf{i}, a)M_n\mid\mathcal{F}_{n-1}, z_n \neq (\mathbf{d}, \mathbf{i}, a)] = M_{n-1}^{\boldsymbol{\gamma}}(\mathbf{d}, \mathbf{i}, a)$ holds trivially.

\vspace{0.3cm}

\textbf{Construction of a mixture martingale:}

For each $(\mathbf{d}, \mathbf{i}, a) \in \mathbb{S}_R \times [K]$, we now construct a {\em mixture} of the martingales $\{M_{n}^{\boldsymbol{\gamma}}(\mathbf{d}, \mathbf{i}, a)\}_{n=K-1}^{\infty}$ as follows. For any $\boldsymbol{\gamma} \in \mathbb{R}^{S_R-1}$ and probability distribution $\boldsymbol{p}$ on $\mathbb{S}_R$, let $\{\boldsymbol{p}^{\boldsymbol{\gamma}}: \boldsymbol{\gamma} \in \mathbb{R}^{S_R-1}\}$ denote the exponential family generated by $\boldsymbol{p}$ with the corresponding parameters defined by $\boldsymbol{\gamma}$, i.e., for any $s \in \mathbb{S}_R$, the $s$-th component $p_s^{\boldsymbol{\gamma}}$ of $\boldsymbol{p}^{\boldsymbol{\gamma}}$ is given by
\begin{equation}
    p_s^{\boldsymbol{\gamma}} = \frac{p_s\, e^{\gamma_s}}{p_{S_R} + \sum_{s=1}^{S_R-1} p_s \, e^{\gamma_s}},
    \label{eq:exponential-family-p-gamma}
\end{equation}
where $\gamma_{S_R}=0$ by convention. Then, it follows that for any fixed probability distribution $\boldsymbol{p}$ on $\mathbb{S}_R$, the mapping in \eqref{eq:exponential-family-p-gamma} defines a bijection between $\mathbb{R}^{S_R-1}$ and the space of probability distributions on $\mathbb{S}_R$. Given any probability distribution $p$ on $\mathbb{S}_R$, let $\boldsymbol{\gamma}_{\boldsymbol{p}}$ denote the value of $\boldsymbol{\gamma} \in \mathbb{R}^{m-1}$ such that $\boldsymbol{p}^{\boldsymbol{\gamma}_{\boldsymbol{p}}}=\boldsymbol{p}$. For later use, we record the below result from \cite{jonsson2020planning} specialized to our setting.
\begin{lemma}\cite[Lemma 3]{jonsson2020planning}
    \label{lem:property-of-exponential-family}
    For any two probability distributions $\boldsymbol{p}, \boldsymbol{p}'$ on $\mathbb{S}_R$ and $\boldsymbol{\gamma} \in \mathbb{R}^{S_R-1}$, 
    \begin{equation}
        \langle \boldsymbol{\gamma}, \boldsymbol{p}' \rangle - \phi_{\boldsymbol{p}}(\boldsymbol{\gamma}) = D_{\text{\rm KL}}(\boldsymbol{p}' \| \boldsymbol{p}) - D_{\text{\rm KL}}(\boldsymbol{p}' \| \boldsymbol{p}^{\boldsymbol{\gamma}}),
    \end{equation}
    where $p^{\boldsymbol{\gamma}}$ denotes the member of the exponential family generated by $\boldsymbol{p}$ with parameters given by $\boldsymbol{\gamma}$, defined via \eqref{eq:exponential-family-p-gamma}.
\end{lemma}

Consider the following mixture of martingales obtained by choosing $\boldsymbol{p}$ according to the Dirichlet distribution $\mathcal{D}(1, \ldots, 1)$ on the simplex of dimension $S_R-1$:
\begin{equation}
    M_n(\mathbf{d}, \mathbf{i}, a) = \int M_{n}^{\boldsymbol{\gamma}_{\boldsymbol{p}}}(\mathbf{d}, \mathbf{i}, a) \, \frac{\prod_{s=1}^{S_R} p_s}{(\prod_{s=1}^{S_R} \Gamma(1))/\Gamma(S_R)}\, \mathrm{d}p_1 \cdots \mathrm{d}p_{S_R}, \quad n \geq K-1.
    \label{eq:mixture-martingale}
\end{equation}
Clearly, because a convex combination of martingales is a martingale, the mixture in \eqref{eq:mixture-martingale} defines a martingale. Also, in \eqref{eq:mixture-martingale}, $\Gamma(\cdot)$ denotes the Gamma function with the property that $\Gamma(k)=(k-1)!$ for all $k \in \mathbb{N}$. 

\vspace{0.3cm}

\textbf{An almost sure lower bound on $M_n(\mathbf{d}, \mathbf{i}, a)$:}

We note that for each $(\mathbf{d}, \mathbf{i}, a) \in \mathbb{S}_R \times [K]$, almost surely,
\begin{align}
    & M_n(\mathbf{d}, \mathbf{i}, a) \nonumber\\
    &= \int \exp \left( N(n, \mathbf{d}, \mathbf{i}, a) \, \left[ \langle \boldsymbol{\gamma}_{\boldsymbol{p}}, \widehat{Q}_n(\cdot \mid\mathbf{d}, \mathbf{i}, a) \rangle - \phi_{\boldsymbol{\theta}, \mathbf{d}, \mathbf{i}, a}(\boldsymbol{\gamma}) \right] \right) \, \frac{\prod_{s=1}^{S_R} p_s}{(\prod_{s=1}^{S_R} \Gamma(1))/\Gamma(S_R)}\, \mathrm{d}p_1 \cdots \mathrm{d}p_{S_R} \nonumber\\
    &\stackrel{(a)}{=} \int \exp \left( N(n, \mathbf{d}, \mathbf{i}, a) \, \left[ D_{\text{\rm KL}}(\widehat{Q}_n(\cdot \mid\mathbf{d}, \mathbf{i}, a) \| Q_{\boldsymbol{\theta}, R}(\cdot \mid\mathbf{d}, \mathbf{i}, a)) - D_{\text{\rm KL}}(\widehat{Q}_n(\cdot \mid\mathbf{d}, \mathbf{i}, a) \| \boldsymbol{p})\right] \right) \, (S_R-1)! \, \mathrm{d}p_1 \cdots \mathrm{d}p_{S_R} \nonumber\\
    &\stackrel{(b)}{=} \exp \left( N(n, \mathbf{d}, \mathbf{i}, a) \, \left[ D_{\text{\rm KL}}(\widehat{Q}_n(\cdot \mid\mathbf{d}, \mathbf{i}, a) \| Q_{\boldsymbol{\theta}, R}(\cdot \mid\mathbf{d}, \mathbf{i}, a)) + H(\widehat{Q}_n(\cdot \mid\mathbf{d},\mathbf{i}, a))\right] \right) \nonumber\\
    &\hspace{7cm} \cdot (S_R-1)! \cdot \int \prod_{s \in \mathbb{S}_R} p_s^{1+N(n, \mathbf{d}, \mathbf{i}, a)\,\widehat{Q}_n(s|\mathbf{d}, \mathbf{i}, a)} \mathrm{d}p_1 \cdots \mathrm{d}p_{S_R} \nonumber\\
    &\stackrel{(c)}{=} \exp \left( N(n, \mathbf{d}, \mathbf{i}, a) \, \left[ D_{\text{\rm KL}}(\widehat{Q}_n(\cdot \mid\mathbf{d}, \mathbf{i}, a) \| Q_{\boldsymbol{\theta}, R}(\cdot \mid\mathbf{d}, \mathbf{i}, a)) + H(\widehat{Q}_n(\cdot \mid\mathbf{d},\mathbf{i}, a))\right] \right) \nonumber\\
    &\hspace{7cm} \cdot \frac{1}{\binom{N(n, \mathbf{d}, \mathbf{i}, a)}{N(n, \mathbf{d}, \mathbf{i}, a)\,\widehat{Q}_n(\cdot \mid\mathbf{d}, \mathbf{i}, a)}} \cdot \frac{1}{\binom{N(n, \mathbf{d}, \mathbf{i}, a) + S_R - 1}{S_R-1}},
    \label{eq:proof-of-delta-PAC-6}
\end{align}
where in the above set of equalities, $(a)$ follows from the application of Lemma~\ref{lem:property-of-exponential-family} to $\boldsymbol{p}=Q_{\boldsymbol{\theta}, R}(\cdot \mid\mathbf{d}, \mathbf{i}, a)$, $\boldsymbol{\gamma}=\boldsymbol{\gamma}_{\boldsymbol{p}}$, and $\boldsymbol{p}'=\widehat{Q}_n(\cdot \mid\mathbf{d}, \mathbf{i}, a)$, $H(\cdot)$ in $(b)$ denotes the Shannon entropy of the argument probability distribution, and in $(c)$, the notation $\binom{N}{\boldsymbol{x}}$ for any integer $N \in \mathbb{N}$ and a vector $\boldsymbol{x}=[x_1, \ldots, x_N]$ with $\sum_{i=1}^{N} x_i = N$ denotes the quantity
$$
\binom{N}{\boldsymbol{x}} =  \frac{N!}{\prod_{i=1}^{N} x_i!},
$$
whereas for any two integers $N,m \in \mathbb{N}$ such that $N>m$, the notation $\binom{N}{m}$ denotes the quantity 
$$
\binom{N}{m} = \frac{N!}{m! \times (N-m)!}.
$$
We now use the below result to upper bound the binomial coefficients appearing in \eqref{eq:proof-of-delta-PAC-6}.
\begin{lemma}\cite[Theorem 11.1.3]{cover2006wiley}
    \label{lem:upper-bounding-binomial-coefficient}
    For any integers $N,m \in \mathbb{N}$ such that $N>m$ and a non-negative vector $\boldsymbol{x}=[x_1, \ldots, x_N]$ such that $\sum_{i=1}^{N} x_i = N$, 
    \begin{equation}
        \binom{N}{\boldsymbol{x}} = \frac{N!}{\prod_{i=1}^{N} x_i!} \leq \exp\left(N\,H(\boldsymbol{x}/N)\right), \quad \binom{N}{m} \leq \exp\left(N \, H(m/N)\right),
        \label{eq:upper-bounding-binomial-coefficient}
    \end{equation}
    where $H(\boldsymbol{x}/N)$ denotes the Shannon entropy of the discrete probability distribution $\boldsymbol{x}/N \coloneqq [x_1/N, \ldots, x_N/N]$, and $H(m/N)$ denotes the Shannon entropy of the Bernoulli distribution with mean $m/N$.
\end{lemma}
Using Lemma~\ref{lem:upper-bounding-binomial-coefficient} to upper bound the binomial coefficients in \eqref{eq:proof-of-delta-PAC-6}, and simplifying the resulting expression, we get
\begin{align}
    & M_n(\mathbf{d}, \mathbf{i}, a) \nonumber\\
    &\geq \exp\bigg(N(n, \mathbf{d}, \mathbf{i}, a) \, D_{\text{\rm KL}}(\widehat{Q}_n(\cdot \mid\mathbf{d}, \mathbf{i}, a) \| Q_{\boldsymbol{\theta}, R}(\cdot \mid\mathbf{d}, \mathbf{i}, a)) \nonumber\\
    &\hspace{5cm} - (N(n, \mathbf{d}, \mathbf{i}, a) + S_R - 1) H((S_R-1)/(N(n, \mathbf{d}, \mathbf{i}, a) + S_R - 1))\bigg)
    \label{eq:proof-of-delta-PAC-7}
\end{align}
almost surely for every $(\mathbf{d}, \mathbf{i}, a) \in \mathbb{S}_R \times [K]$.

\vspace{0.3cm}

\textbf{A product martingale and proof completion:}

Taking product over all $(\mathbf{d}, \mathbf{i}, a) \in \mathbb{S}_R \times [K]$ in \eqref{eq:proof-of-delta-PAC-7}, we have almost surely
\begin{align}
    M_n 
    &\coloneqq \prod_{(\mathbf{d}, \mathbf{i}, a) \in \mathbb{S}_R \times [K]} M_n(\mathbf{d}, \mathbf{i}, a) \nonumber\\
    &\geq \exp\bigg(\sum_{(\mathbf{d}, \mathbf{i}, a) \in \mathbb{S}_R \times [K]} N(n, \mathbf{d}, \mathbf{i}, a) \, D_{\text{\rm KL}}(\widehat{Q}_n(\cdot \mid\mathbf{d}, \mathbf{i}, a) \| Q_{\boldsymbol{\theta}, R}(\cdot \mid\mathbf{d}, \mathbf{i}, a)) \nonumber\\
    &\hspace{3cm} - \sum_{(\mathbf{d}, \mathbf{i}, a) \in \mathbb{S}_R \times [K]} (N(n, \mathbf{d}, \mathbf{i}, a) + S_R - 1) H((S_R-1)/(N(n, \mathbf{d}, \mathbf{i}, a) + S_R - 1))\bigg).
\end{align}
We now prove that $\{M_n\}_{n=K-1}^{\infty}$ is a martingale and note that for each $(\mathbf{d}, \mathbf{i}, a) \in \mathbb{S}_R \times [K]$,
\begin{align}
    & (N(n, \mathbf{d}, \mathbf{i}, a) + S_R - 1) H((S_R-1)/(N(n, \mathbf{d}, \mathbf{i}, a) + S_R - 1)) \nonumber\\
    &= (S_R-1)\, \log\left(1 + \frac{N(n, \mathbf{d}, \mathbf{i}, a)}{S_R-1}\right) + N(n, \mathbf{d}, \mathbf{i}, a) \, \log \left(1+\frac{S_R-1}{N(n, \mathbf{d}, \mathbf{i}, a)}\right) \nonumber\\
    &\stackrel{(a)}{\leq} (S_R-1)\, \log\left(1 + \frac{N(n, \mathbf{d}, \mathbf{i}, a)}{S_R-1}\right) + (S_R-1) \nonumber\\
    &=(S_R-1)\, \log\left(e\left[1 + \frac{N(n, \mathbf{d}, \mathbf{i}, a)}{S_R-1}\right]\right),
\end{align}
where $(a)$ above follows from the inequality $\log(1+x) \leq x$. The fact that $\{M_n\}_{n=K-1}^{\infty}$ is a martingale (with respect to the filtration in \eqref{eq:filtration}) follows by noting that for any $(\mathbf{d}, \mathbf{i}, a) \in \mathbb{S}_R \times [K]$, using the shorthand notation $z_n=(\mathbf{d}(n), \mathbf{i}(n), A_n)$, we have
\begin{align}
    \mathbb{E}_{\boldsymbol{\theta}}[M_n\mid \mathcal{F}_{n-1}, z_n=(\mathbf{d}, \mathbf{i}, a)] 
    &= \mathbb{E}_{\boldsymbol{\theta}} \bigg[M_n(\mathbf{d}, \mathbf{i}, a) \times \prod_{(\mathbf{d}', \mathbf{i}', a') \neq (\mathbf{d}, \mathbf{i}, a)} M_n(\mathbf{d}', \mathbf{i}', a') ~ \bigg\vert ~ \mathcal{F}_{n-1}, \ z_n = (\mathbf{d}, \mathbf{i}, a)\bigg] \nonumber\\
    &= \mathbb{E}_{\boldsymbol{\theta}} \bigg[M_n(\mathbf{d}, \mathbf{i}, a) \times \prod_{(\mathbf{d}', \mathbf{i}', a') \neq (\mathbf{d}, \mathbf{i}, a)} M_{n-1}(\mathbf{d}', \mathbf{i}', a') ~\bigg\vert~ \mathcal{F}_{n-1}\bigg] \nonumber\\
    &= \mathbb{E}_{\boldsymbol{\theta}} \bigg[M_n(\mathbf{d}, \mathbf{i}, a) \mid \mathcal{F}_{n-1}\bigg] \times \prod_{(\mathbf{d}', \mathbf{i}', a') \neq (\mathbf{d}, \mathbf{i}, a)} M_{n-1}(\mathbf{d}', \mathbf{i}', a') \nonumber\\
    &= M_{n-1}(\mathbf{d}, \mathbf{i}, a) \times \prod_{(\mathbf{d}', \mathbf{i}', a') \neq (\mathbf{d}, \mathbf{i}, a)} M_{n-1}(\mathbf{d}', \mathbf{i}', a') \nonumber\\
    &= M_{n-1},
    \label{eq:proof-of-delta-PAC-8}
\end{align}
where the third line follows by noting that conditioned on $\mathcal{F}_{n-1}$, the random variable $M_n(\mathbf{d}, \mathbf{i}, a)$ is independent of $\{M_{n-1}(\mathbf{d}', \mathbf{i}', a'): (\mathbf{d}', \mathbf{i}', a') \neq (\mathbf{d}, \mathbf{i}, a)\}$. Applying $\mathbb{E}_{\boldsymbol{\theta}}[\cdot]$ on either side of \eqref{eq:proof-of-delta-PAC-8} and using the tower property of conditional expectations, we have
\begin{align}
    \mathbb{E}_{\boldsymbol{\theta}}[M_n] = \mathbb{E}_{\boldsymbol{\theta}}[\mathbb{E}_{\boldsymbol{\theta}}[M_n\mid\mathcal{F}_{n-1}, z_n]\mid \mathcal{F}_{n-1}] = M_{n-1}.
\end{align}
Noting that $\mathbb{E}_{\boldsymbol{\theta}}[M_n] = \mathbb{E}_{\boldsymbol{\theta}}[M_{K-1}] = 1$ for all $n \geq K-1$, and using Doob's maximal inequality, we get
\begin{equation}
    \P_{\boldsymbol{\theta}}\left(\exists n \geq K-1: M_n > \frac{1}{\delta}\right) \leq \delta\, \mathbb{E}_{\boldsymbol{\theta}}[M_{K-1}] = \delta.
    \label{eq:proof-of-delta-PAC-9}
\end{equation}
Clearly, \eqref{eq:proof-of-delta-PAC-9} implies \eqref{eq:proof-of-delta-PAC-3}, thus proving the desired result.
%-------------------------------------------------------------------

\section{Proof of Proposition~\ref{prop:almost-sure-upper-bound}}
Fix $\delta \in (0,1)$ and $\eta \in (0,1)$. Fix an arbitrary $\nu \in \mathcal{W}^\star(\boldsymbol{\theta})$, and let $\omega_{\boldsymbol{\theta}, \nu}^\star$ be as defined in Lemma~\ref{lem:concentration-of-state-action-visitations}. Consider the event
\begin{equation}
    \mathcal{E} = \left\lbrace \lim_{n \to \infty} \max_{(\mathbf{d}, \mathbf{i}, a)} \bigg\lvert \frac{N(n, \mathbf{d}, \mathbf{i}, a)}{n-K+1} - \omega_{\boldsymbol{\theta}, \nu}^\star(\mathbf{d}, \mathbf{i}, a) \bigg\rvert=0, \quad \lim_{n \to \infty} \|\boldsymbol{\eta}(n) - \boldsymbol{\eta}\|_{\infty}=0 \right\rbrace.
    \label{eq:proof-of-almost-sure-upper-bound-1}
\end{equation}
Thanks to Lemma~\ref{lem:concentration-of-empirical-arm-means}, Lemma~\ref{lem:concentration-of-state-action-visitations}, and the Borel--Cantelli lemma, we have $\P_{\boldsymbol{\theta}}(\mathcal{E})=1$ under the non-stopping version of $\pi^{\textsc{Rstl-Dtrack}}$ (with the same parameters as $\pi^{\textsc{Rstl-Dtrack}}$). Fix $\omega \in \mathcal{E}$ and $\gamma>0$ arbitrarily. Then, there exists $N_\gamma(\omega) \in \mathbb{N}$ independent of $\delta$ such that the following hold for all $n \geq N_\gamma(\omega)$:
\begin{align}
    \frac{Z(n, \omega)}{n} &\geq (1-\gamma) (\eta\, T_{\text{\rm unif}}^\star(\boldsymbol{\theta}) + (1-\eta)\, T_R^\star(\boldsymbol{\theta})), \label{eq:test-statistic-exceeds-the-correct-value} \\
    \zeta(n, \delta) &\leq \log\left(\frac{1}{\delta}\right) + \gamma\, (\eta\, T_{\text{\rm unif}}^\star(\boldsymbol{\theta}) + (1-\eta)\, T_R^\star(\boldsymbol{\theta})) \, n. \label{eq:threshold-has-the-correct-scaling}
\end{align}
The inequality in \eqref{eq:threshold-has-the-correct-scaling} follows by noting that $\zeta(n, \delta) =\log(1/\delta)+ O(\log(n)) = \log(1/\delta) + o(n)$. We then have
\begin{align}
    \tau(\omega) 
    &= \inf \{n \geq K: Z(n, \omega) \geq \zeta(n, \delta)\} \nonumber\\
    &\leq \inf \{n \geq N_\gamma(\omega): Z(n, \omega) \geq \zeta(n, \delta)\} \nonumber\\ 
    &\leq \inf \left\lbrace n \geq N_\gamma(\omega): n\, (1-\gamma) (\eta\, T_{\text{\rm unif}}^\star(\boldsymbol{\theta}) + (1-\eta)\, T_R^\star(\boldsymbol{\theta})) \geq \log\left(\frac{1}{\delta}\right) + \gamma\, (\eta\, T_{\text{\rm unif}}^\star(\boldsymbol{\theta}) + (1-\eta)\, T_R^\star(\boldsymbol{\theta})) \, n \right\rbrace \nonumber\\
    &\leq \inf\left\lbrace n \geq N_\gamma(\omega): n\, (1-2\gamma) (\eta\, T_{\text{\rm unif}}^\star(\boldsymbol{\theta}) + (1-\eta)\, T_R^\star(\boldsymbol{\theta})) \geq \log\left(\frac{1}{\delta}\right) \right\rbrace \nonumber\\
    &= \max\left\lbrace N_\gamma(\omega), \quad  \inf\left\lbrace n \geq 1: n\, (1-2\gamma) (\eta\, T_{\text{\rm unif}}^\star(\boldsymbol{\theta}) + (1-\eta)\, T_R^\star(\boldsymbol{\theta})) \geq \log\left(\frac{1}{\delta}\right) \right\rbrace \right\rbrace \nonumber\\
    &\leq \max\left\lbrace N_\gamma(\omega), \quad  \left\lceil \frac{\log(1/\delta)}{(1-2\gamma)(\eta\, T_{\text{\rm unif}}^\star(\boldsymbol{\theta}) + (1-\eta)\, T_R^\star(\boldsymbol{\theta}))} \right\rceil \right\rbrace \nonumber\\
    &\leq \max\left\lbrace N_\gamma(\omega), \quad  1+\frac{\log(1/\delta)}{(1-2\gamma)(\eta\, T_{\text{\rm unif}}^\star(\boldsymbol{\theta}) + (1-\eta)\, T_R^\star(\boldsymbol{\theta}))} \right\rbrace
    \label{eq:proof-of-almost-sure-upper-bound-2}
\end{align}
for all $\omega \in \mathcal{E}$, where in writing the last line above, we use the relation $\lceil x \rceil < 1+x$. Thus, it follows from \eqref{eq:proof-of-almost-sure-upper-bound-2} that $\P_{\boldsymbol{\theta}}(\tau < +\infty)=1$ under $\pi^{\textsc{Rstl-Dtrack}}$, which in turn implies from \eqref{eq:stop-in-finite-time-and-error-prob-less-than-delta} that $\pi^{\textsc{Rstl-Dtrack}} \in \Pi_R(\delta)$. Also, dividing both sides of \eqref{eq:proof-of-almost-sure-upper-bound-2} by $\log(1/\delta)$, letting $\delta \downarrow 0$ and noting that $N_\gamma(\omega)$ does not depend on $\delta$, we get
\begin{equation}
    \limsup_{\delta \downarrow 0} \frac{\tau(\omega)}{\log(1/\delta)} \leq \frac{1}{(1-2\gamma)(\eta\, T_{\text{\rm unif}}^\star(\boldsymbol{\theta}) + (1-\eta)\, T_R^\star(\boldsymbol{\theta}))} \quad \forall \omega \in \mathcal{E}.
    \label{eq:proof-of-almost-sure-upper-bound-3}
\end{equation}
Noting that the right-hand side of \eqref{eq:proof-of-almost-sure-upper-bound-3} holds for all $\gamma > 0$, we arrive at \eqref{eq:almost-sure-upper-bound} by taking limits as $\gamma \downarrow 0$.
%-------------------------------------------------------------------

\section{Proof of Proposition~\ref{prop:upper-bound-on-expected-stopping-time}}
\textbf{Notations:} We first introduce some notations. Let $\|\cdot\|_1$, $\|\cdot\|_2$, and $\|\cdot\|_{\infty}$ denote the vector $1$-norm, $2$-norm, and sup-norm operators respectively. Further, for matrices $Q,Q'$ of identical dimensions, let 
$$
\|Q-Q'\|_{\infty} \coloneqq \max_{(\mathbf{d}, \mathbf{i}, a)} \| Q(\cdot \mid\mathbf{d}, \mathbf{i}, a) -  Q'(\cdot \mid\mathbf{d}, \mathbf{i}, a) \|_1.
$$
For all $n \geq K$, let $\boldsymbol{N}(n) \coloneqq [N(n, \mathbf{d}, \mathbf{i}, a): (\mathbf{d}, \mathbf{i}, a) \in \mathbb{S}_R \times [K]]^\top$. Let $\widehat{\mathcal{M}}_n$ denote the MDP with transition kernel $\widehat{Q}_n$ defined in \eqref{eq:categorical-distribution}, and for any $\boldsymbol{\theta} \in \Theta^K$, let $\| \widehat{\mathcal{M}}_n - \mathcal{M}_{\boldsymbol{\theta}, R} \|_{\infty}$ be as defined in \eqref{eq:norm-between-MDPs}. 
\iffalse
For $\boldsymbol{\lambda}, \boldsymbol{\lambda}' \in \Theta^K$, let
\begin{equation}
    \| \mathcal{M}_{\boldsymbol{\lambda}, R} - \mathcal{M}_{\boldsymbol{\lambda}', R} \|_{\infty} \coloneqq \max_{(\mathbf{d}, \mathbf{i}, a) \in \mathbb{V}} \|Q_{\boldsymbol{\lambda}, R}(\cdot \mid\mathbf{d}, \mathbf{i}, a) - Q_{\boldsymbol{\lambda}', R}(\cdot \mid\mathbf{d}, \mathbf{i}, a)\|,
    \label{eq:norm-between-MDPs-lambda-lambdaprime}
\end{equation}
where $\mathbb{V}$ is the set of all {\em valid} state-action pairs, and $\|\cdot\|_1$ is the vector $1$-norm operator. 
\fi
Given $\boldsymbol{\theta} \in \Theta^K$, let $\boldsymbol{\eta}_{\boldsymbol{\theta}} \coloneqq [\eta_{\theta_a}: a \in [K]]^\top$. Also, let $\widehat{\boldsymbol{\eta}}(n) = [\widehat{\eta}_a(n): a \in [K]]^\top$, $\widehat{\theta}_a(n) = \Dot{A}^{-1}(\widehat{\eta}_a)$ for all $a \in [K]$, and $\widehat{\boldsymbol{\theta}}(n) = [\widehat{\theta}_a(n): a \in [K]]^\top$. 
For any $\boldsymbol{\theta} \in \Theta^K$ and $\nu \in \mathcal{W}^\star(\boldsymbol{\theta})$, let $\omega_{\boldsymbol{\theta}, \nu}^\star=[\omega_{\boldsymbol{\theta}, \nu}^\star (\mathbf{d}, \mathbf{i}, a): (\mathbf{d}, \mathbf{i}, a) \in \mathbb{S}_R \times [K]]^\top$ be as defined in Lemma~\ref{lem:concentration-of-state-action-visitations}. Let
\begin{equation}
    \psi^\prime(\boldsymbol{\theta}', \nu', Q') \coloneqq \inf_{\boldsymbol{\lambda} \in \textsc{Alt}(\boldsymbol{\theta}')} \ \sum_{(\mathbf{d}, \mathbf{i} ) \in \mathbb{S}_R} \ \sum_{a=1}^{K} \nu'(\mathbf{d}, \mathbf{i}, a) \, D_{\text{\rm KL}}(Q'(\cdot \mid\mathbf{d}, \mathbf{i}, a) \| Q_{\boldsymbol{\lambda}, R}(\cdot \mid\mathbf{d}, \mathbf{i}, a)).
    \label{eq:psi-prime-definition}
\end{equation}

\begin{proof}[Proof of Proposition~\ref{prop:upper-bound-on-expected-stopping-time}]
    Fix $\boldsymbol{\theta} \in \Theta^K$, $\eta \in (0,1)$, $\nu \in \mathcal{W}^\star(\boldsymbol{\theta})$, and $\delta \in (0,1)$. Assume that $\boldsymbol{\theta}$ is the underlying instance. Fix an arbitrary $\xi>0$. From Lemma~\ref{lem:property-of-exponential-family}, we know that the mapping $\boldsymbol{\theta}' \mapsto \boldsymbol{\eta}_{\boldsymbol{\theta}'}$ is continuous. Hence, there exists $\varepsilon=\varepsilon(\xi) \leq \xi$ such that
\begin{equation}
    \|\boldsymbol{\theta} - \boldsymbol{\theta}'\|_{2} < \varepsilon \implies \| \boldsymbol{\eta}_{\boldsymbol{\theta}} - \boldsymbol{\eta}_{\boldsymbol{\theta}'} \|_{2} < \xi.
    \label{eq:theta-close-implies-MDP-theta-close}
\end{equation} 
Given $N \geq K$, let\footnote{We assume, without loss of generality, that $N^{3/4}$ is an integer.}
\begin{equation}
    C_N^3(\xi) \coloneqq \bigcap_{n=N^5}^{N^6} \bigg\lbrace \bigg\| \frac{\boldsymbol{N}(n)}{n-K+1} - \omega_{\boldsymbol{\theta}, \nu}^\star \bigg\|_{\infty} \leq K_{\xi}(\boldsymbol{\theta}, \nu)\, \xi \bigg\rbrace,
    \label{eq:C-N-3-of-xi-event}
\end{equation}
where $K_\xi(\boldsymbol{\theta}, \nu)$ are as defined in Lemma~\ref{lem:concentration-of-state-action-visitations}. Let $C_N^1(\xi)$ and $C_N^2(\xi)$ be as defined in \eqref{eq:C-N-1-of-xi-event} and \eqref{eq:C-N-2-of-xi-event} respectively. Let
\begin{equation}
    \psi^\star(\boldsymbol{\theta}, \omega_{\boldsymbol{\theta}, \nu}^\star, \xi) \coloneqq \inf_{\substack{\boldsymbol{\theta}': \|\boldsymbol{\theta}' - \boldsymbol{\theta}\|_{2} \leq \varepsilon \\ \nu' : \|\nu' - \omega_{\boldsymbol{\theta}, \nu}^\star\|_{\infty} \leq K_{\xi}\, \xi \\
    Q': \|Q - Q_{\boldsymbol{\theta}, R}\|_{\infty} \leq \xi}} \psi'(\theta', \nu', Q'),
    \label{eq:psi-star-definition}
\end{equation}
Notice that as a consequence of \eqref{eq:almost-sure-divergence-of-state-action-visitations}, the following convergences hold almost surely:
\begin{equation}
    \|\widehat{\boldsymbol{\eta}}(n) - \boldsymbol{\eta}\|_{2} \to 0, \quad \|\widehat{\boldsymbol{\theta}}(n) - \boldsymbol{\theta}\|_2 \to 0, \quad \|\widehat{\mathcal{M}}_n - \mathcal{M}_{\boldsymbol{\theta}, R}\|_{\infty} \to 0.
    \label{eq:almost-sure-convergences-of-interest}
\end{equation}
Therefore, there exists $N_1=N_1(\xi) \geq K$ such that for all $N \geq N_1$, the following hold:
\begin{align}
    \|\widehat{\boldsymbol{\theta}}(n) - \boldsymbol{\theta}\|_2 &\leq \varepsilon, \qquad N^5 \leq n \leq N^6, \label{eq:empirical-parameters-within-xi-of-true-values}\\
    \|\widehat{\boldsymbol{\eta}}(n) - \boldsymbol{\eta}\|_2 &\leq \xi, \qquad N^5 \leq n \leq N^6, \label{eq:empirical-estimates-within-xi-of-true-values} \\
    \|\widehat{\mathcal{M}}_n - \mathcal{M}_{\boldsymbol{\theta}, R}\|_{\infty} &\leq \xi, \qquad N^5 \leq n \leq N^6, \label{eq:empirical-MDP-within-xi-of-true-MDP} \\
    \psi\left(\widehat{\boldsymbol{\theta}}(n), \frac{\boldsymbol{N}(n)}{n-K+1}\right) &\geq \psi^\star(\boldsymbol{\theta}, \omega_{\boldsymbol{\theta}, \nu}^\star, \xi), \qquad N^5 \leq n \leq N^6. \label{eq:empirical-psi-less-than-psi-star}
\end{align}
Furthermore, it follows from Lemma~\ref{lem:concentration-of-parameter-estimates} and Lemma~\ref{lem:concentration-of-empirical-arm-means} that $\mathbb{E}_{\boldsymbol{\theta}}[N_1(\xi)] < + \infty$ (this is simply an artefact of the probability term $\P_{\boldsymbol{\theta}}(N_1>n)$ decaying exponentially in $n$). Hence, it follows that for all $N \geq N_1$, the event $C_N^3(\xi)$ holds with high probability, conditioned on $C_N^1(\xi)$. Furthermore, noting that $N_a(n, \mathbf{d}, \mathbf{i}, a) \leq n+1$ and therefore $\zeta(n, \delta) = \log(1/\delta) + O(\log n)$, it follows that there exists $N_2=N_2(\xi) \geq K$ such that 
\begin{equation}
    \zeta(n, \delta) \leq \log\left(\frac{1}{\delta}\right) + \xi\, \psi^\star(\boldsymbol{\theta}, \omega_{\boldsymbol{\theta}, \nu}^\star, \xi)\, n \qquad \forall n \geq N_2.
    \label{eq:threshold-less-than-an-appropriate-constant}
\end{equation}
Let $N_3=N_3(\xi, \delta)$ be defined as
\begin{equation}
    N_3(\xi, \delta) = \inf \left\lbrace n \geq K: n - K + 1 \geq \frac{1}{(1-\xi)\, \psi^\star(\boldsymbol{\theta}, \omega_{\boldsymbol{\theta}, \nu}^\star, \xi)}\, \log\left(\frac{1}{\delta}\right) \right\rbrace.
    \label{eq:N-3-definition}
\end{equation}
Then, for all $N \geq \max\{N_1, N_2, N_3\}$, it follows that conditional on $C_N^1(\xi) \cap C_N^2(\xi) \cap C_N^3(\xi)$,
\begin{align*}
    Z(N) 
    &= \inf_{\boldsymbol{\lambda} \in \textsc{Alt}(\boldsymbol{\theta}(N))} \sum_{(\mathbf{d}, \mathbf{i}) \in \mathbb{S}_R} \ \sum_{a=1}^{K} N(N, \mathbf{d}, \mathbf{i}, a) \, D_{\text{\rm KL}}(\widehat{Q}_N(\cdot \mid\mathbf{d}, \mathbf{i}, a) \| Q_{\boldsymbol{\lambda}, R}(\cdot \mid\mathbf{d}, \mathbf{i}, a)) \\
    &\geq (N-K+1)\, \psi^\star(\boldsymbol{\theta}, \omega_{\boldsymbol{\theta}, \nu}^\star, \xi) \nonumber\\
    &\geq (N-K+1)\, \xi\, \psi^\star(\boldsymbol{\theta}, \omega_{\boldsymbol{\theta}, \nu}^\star, \xi) + \log\left(\frac{1}{\delta}\right) \nonumber\\
    &\geq \zeta(N, \delta),
\end{align*}
thereby proving that
\begin{equation}
    C_N^1(\xi) \cap C_N^2(\xi) \cap C_N^3(\xi) \subset \{\tau \leq N\} \qquad \forall N \geq \max\{N_1(\xi), N_2(\xi), N_3(\xi, \delta)\}.
\end{equation}
Noting that $N_1(\xi)$ is a random variable, while $N_2(\xi), N_3(\xi, \delta)$ are deterministic constants, for any $n \geq K$, we have 
\begin{align}
    &\mathbb{E}_{\boldsymbol{\theta}}[\tau\, \mathbf{1}_{\{N_1(\xi)=n\}}] \nonumber\\
    &= \sum_{N=1}^{\infty} \P_{\boldsymbol{\theta}}(\tau > N, \, N_1(\xi)=n) \nonumber\\
    &\leq \max\{n, N_2(\xi), N_3(\xi, \delta)\} + \sum_{N=\max\{n, N_2, N_3\}+1}^{\infty} \P_{\boldsymbol{\theta}}(\tau > N) \nonumber\\
    &\leq \max\{n, N_2(\xi), N_3(\xi, \delta)\} + \sum_{N=\max\{N_1, N_2, N_3\}+1}^{\infty} \P_{\boldsymbol{\theta}}\bigg(\overline{C_N^1(\xi)} \cup \overline{C_N^2(\xi)} \cup \overline{C_N^3(\xi)}\bigg) \nonumber\\
    &\leq  \max\{n, N_2(\xi), N_3(\xi, \delta)\} + \sum_{N=\max\{N_1, N_2, N_3\}+1}^{\infty} \bigg[\P_{\boldsymbol{\theta}}\bigg(\overline{C_N^1(\xi)}\bigg) + \P_{\boldsymbol{\theta}}\bigg(\overline{C_N^2(\xi)}\bigg) + \P_{\boldsymbol{\theta}}\bigg(\overline{C_N^3(\xi)} ~\bigg|~ C_N^2(\xi)\bigg) \bigg] \nonumber\\
    &\leq \max\{n, N_2(\xi), N_3(\xi, \delta)\} + \sum_{N=1}^{\infty} \bigg[\P_{\boldsymbol{\theta}}\bigg(\overline{C_N^1(\xi)}\bigg) + \P_{\boldsymbol{\theta}}\bigg(\overline{C_N^2(\xi)}\bigg) + \P_{\boldsymbol{\theta}}\bigg(\overline{C_N^3(\xi)}~\bigg|~ C_N^2(\xi)\bigg) \bigg] \nonumber\\
    &\leq n + N_2(\xi) + N_3(\xi, \delta) + \sum_{N=1}^{\infty} \bigg[\P_{\boldsymbol{\theta}}\bigg(\overline{C_N^1(\xi)}\bigg) + \P_{\boldsymbol{\theta}}\bigg(\overline{C_N^2(\xi)}\bigg) + \P_{\boldsymbol{\theta}}\bigg(\overline{C_N^3(\xi)}~\bigg|~ C_N^2(\xi)\bigg) \bigg],
    \label{eq:proof-of-upper-bound-1}
\end{align}
from which it follows that
\begin{align}
    \mathbb{E}_{\boldsymbol{\theta}}[\tau] 
    &= \sum_{n=1}^{\infty} \mathbb{E}_{\boldsymbol{\theta}}[\tau_{\delta}\, \mathbf{1}_{\{N_1(\xi)=n\}}] \nonumber\\
    &\leq \mathbb{E}_{\boldsymbol{\theta}}[N_1(\xi)] + N_2(\xi) + N_3(\xi, \delta) + \sum_{N=1}^{\infty} \bigg[\P_{\boldsymbol{\theta}}\bigg(\overline{C_N^1(\xi)}\bigg) + \P_{\boldsymbol{\theta}}\bigg(\overline{C_N^2(\xi)}\bigg) + \P_{\boldsymbol{\theta}}\bigg(\overline{C_N^3(\xi)} ~\bigg|~ C_N^2(\xi)\bigg) \bigg].
    \label{eq:proof-of-upper-bound-1-1}
\end{align}
From Lemma~\ref{lem:concentration-of-parameter-estimates}, Lemma~\ref{lem:concentration-of-empirical-arm-means}, and Lemma~\ref{lem:concentration-of-state-action-visitations}, we know that the infinite-summation term in \eqref{eq:proof-of-upper-bound-1-1} is finite. Dividing both sides of \eqref{eq:proof-of-upper-bound-1-1} by $\log(1/\delta)$ and taking limits as $\delta \downarrow 0$, we get
\begin{align}
    \limsup_{\delta \downarrow 0} \frac{\mathbb{E}_{\boldsymbol{\theta}}[\tau_\pi]}{\log(1/\delta)} 
    &\leq \limsup_{\delta \downarrow 0} \frac{N_3(\xi, \delta)}{\log(1/\delta)} \nonumber\\
    &= \frac{1}{(1-\xi)\,\psi^\star(\boldsymbol{\theta}, \omega_{\boldsymbol{\theta}, \nu}^\star, \xi)}.
    \label{eq:proof-of-upper-bound-2}
\end{align}
Taking limits as $\xi \downarrow 0$, using the fact that $\limsup_{\xi \downarrow 0} K_{\xi}(\boldsymbol{\theta}, \nu) < +\infty$ (cf. Lemma~\ref{lem:concentration-of-state-action-visitations}), and noting that
\begin{equation}
    \lim_{\xi \downarrow 0} \psi^\star(\boldsymbol{\theta}, \omega_{\boldsymbol{\theta}, \nu}^\star, \xi) = \psi(\omega_{\boldsymbol{\theta}, \nu}^\star, \boldsymbol{\theta}),
\end{equation}
where $\psi$ is as defined in \eqref{eq:psi}, we get
\begin{equation}
     \limsup_{\delta \downarrow 0} \frac{\mathbb{E}_{\boldsymbol{\theta}}[\tau_\pi]}{\log(1/\delta)} \leq \frac{1}{\psi(\omega_{\boldsymbol{\theta}, \nu}^\star, \boldsymbol{\theta})}.
    \label{eq:proof-of-upper-bound-3}
\end{equation}
Noting that 
$$
\omega_{\boldsymbol{\theta}, \nu}^\star = \eta\, \nu_{\boldsymbol{\theta}}^{\text{\rm unif}}+ (1-\eta)\, \nu,
$$
and that
\begin{align*}
    \psi(\omega_{\boldsymbol{\theta}, \nu}^\star, \boldsymbol{\theta})
    &= \inf_{\boldsymbol{\lambda} \in \textsc{Alt}(\boldsymbol{\theta})} \sum_{(\mathbf{d}, \mathbf{i}) \in \mathbb{S}_R} \ \sum_{a=1}^{K} \omega_{\boldsymbol{\theta}, \nu}^\star(\mathbf{d}, \mathbf{i}, a)\, D_{\text{\rm KL}}(Q_{\boldsymbol{\theta}, R}(\cdot \mid\mathbf{d}, \mathbf{i}, a) \| Q_{\boldsymbol{\lambda}, R}(\cdot \mid\mathbf{d}, \mathbf{i}, a)) \nonumber\\
    &= \inf_{\boldsymbol{\lambda} \in \textsc{Alt}(\boldsymbol{\theta})} \sum_{(\mathbf{d}, \mathbf{i}) \in \mathbb{S}_R} \ \sum_{a=1}^{K} (\eta\, \nu_{\boldsymbol{\theta}}^{\text{\rm unif}}(\mathbf{d}, \mathbf{i}, a)+ (1-\eta)\, \nu(\mathbf{d}, \mathbf{i}, a))\, D_{\text{\rm KL}}(Q_{\boldsymbol{\theta}, R}(\cdot \mid\mathbf{d}, \mathbf{i}, a) \| Q_{\boldsymbol{\lambda}, R}(\cdot \mid\mathbf{d}, \mathbf{i}, a)) \nonumber\\
    &\geq \eta\, T_{\text{\rm unif}}^\star(\boldsymbol{\theta}) + (1-\eta)\, T_R^\star(\boldsymbol{\theta}),
\end{align*}
where $T_R^\star(\boldsymbol{\theta})$ appears because $\nu \in \mathcal{W}^\star(\boldsymbol{\theta})$, 
we get
\begin{equation}
     \limsup_{\delta \downarrow 0} \frac{\mathbb{E}_{\boldsymbol{\theta}}[\tau_\pi]}{\log(1/\delta)} \leq \frac{1}{\eta\, T_{\text{\rm unif}}^\star(\boldsymbol{\theta}) + (1-\eta)\, T_R^\star(\boldsymbol{\theta})}.
    \label{eq:proof-of-upper-bound-4}
\end{equation}
Finally, taking limits as $\eta \downarrow 0$ in \eqref{eq:proof-of-upper-bound-4} yields \eqref{eq:upper-bound-on-expected-stopping-time}.
\end{proof} 
%-------------------------------------------------------------------

\section{Technical Results}
In this section, we record some technical results of interest. First, we introduce some notations. Fix $\boldsymbol{\theta} \in \Theta^K$. For any $n \geq K$, let $Q_n \coloneqq Q_{\boldsymbol{\theta}, \pi_n}$ denote the transition kernel of the MDP $\mathcal{M}_{\boldsymbol{\theta}, R}$ under $\pi_n$.
Along similar lines as in the proof of Lemma~\ref{lem:ergodicity-of-MDP-under-unif-policy}, it is easy to show that implies that $Q_n$ is ergodic for each $n \geq K$. Let $\omega_n^\star = [\omega_n^\star(\mathbf{d}, \mathbf{i}, a): (\mathbf{d}, \mathbf{i}, a) \in \mathbb{S}_R \times [K]]^\top$ denote the unique stationary distribution of $Q_n$. Let $W_n = \mathbf{1}\omega_n^T$ denote the rank-$1$ matrix each of whose rows is equal to $\omega_n^T$. For $r \in \mathbb{N}$, let $Q_n^r$ denote the $r$-fold self-product of $Q_n$. Let $\|Q_n^r - W_n\|_{\infty} = \max_{(\mathbf{d}, \mathbf{i}, a)} \|Q_n^r(\cdot \mid\mathbf{d}, \mathbf{i}, a) - W_n(\cdot \mid\mathbf{d}, \mathbf{i}, a)\|_{1}$, where $\|\cdot \|_1$ denotes vector $1$-norm. 

Let $\omega_{\boldsymbol{\theta}}^{\text{\rm unif}} = [\omega_{\boldsymbol{\theta}}^{\text{\rm unif}}(\mathbf{d}, \mathbf{i}, a): (\mathbf{d}, \mathbf{i}, a) \in \mathbb{S}_R \times [K]]^\top$ denote the unique stationary distribution of the MDP $\mathcal{M}_{\boldsymbol{\theta}, R}$ under $\pi^{\text{\rm unif}}$. Let $\mathbb{V}$ be as defined in Lemma~\ref{lem:sufficient-exploration-of-state-actions}. Thanks to Lemma~\ref{lem:ergodicity-of-MDP-under-unif-policy}, the following quantities are well-defined: 
\begin{align}
    r_{\boldsymbol{\theta}, \text{\rm unif}} &\coloneqq \min\left\lbrace r \in \mathbb{N}: \forall(\mathbf{d}, \mathbf{i}, a), (\mathbf{d}', \mathbf{i}', a') \in \mathbb{V}, \quad Q_{\boldsymbol{\theta}, \pi^{\text{\rm unif}}}^{r}(\mathbf{d}', \mathbf{i}', a'|\mathbf{d}, \mathbf{i}, a) > 0 \right\rbrace, \label{eq:r-unif-definition} \\
    \sigma_{\boldsymbol{\theta}, \text{\rm unif}} &\coloneqq \min_{(\mathbf{d}, \mathbf{i}, a), (\mathbf{d}', \mathbf{i}', a') \in \mathbb{V}} \frac{Q_{\boldsymbol{\theta}, \pi^{\text{\rm unif}}}^{r_{\boldsymbol{\theta}, \text{\rm unif}}}(\mathbf{d}', \mathbf{i}', a'|\mathbf{d}, \mathbf{i}, a)}{\omega_{\boldsymbol{\theta}}^{\text{\rm unif}}(\mathbf{d}', \mathbf{i}', a')}. \label{eq:sigma-unif-definition}
\end{align}
Given a policy $\pi=[\pi(a|\mathbf{d}, \mathbf{i}): (\mathbf{d}, \mathbf{i}, a) \in \mathbb{S}_R \times [K]]^\top$, let $U_{\pi} \coloneqq \{(\mathbf{d}, \mathbf{i}, a): \pi(a|\mathbf{d}, \mathbf{i}) > 0\}$. Let $U_{n} \coloneqq U_{\pi_{n}^\eta}$, $n \geq K$.

\subsection{A Bound on \texorpdfstring{$\|Q_n^r - W_n\|_{\infty}$}{aa} for \texorpdfstring{$r \in \mathbb{N}$}{r}}
The below result from \cite{levin2017markov} gives a bound on the rate of convergence of powers of an ergodic transition kernel to the stationary distribution.
\begin{lemma}\cite[Theorem 4.9]{levin2017markov}
    \label{lem:rate-of-convergence-levin-and-peres}
    Let $Q$ be an ergodic stochastic matrix on a finite state space $\mathcal{Z}$ with stationary distribution $\omega$. Let $W$ denote the rank-$1$ matrix each of whose rows is equal to $\omega^\top$. Suppose that there exists $\sigma>0$ and $r_0 \in \mathbb{N}$ such that $Q^{r_0}(z, z') \geq \sigma \omega(z')$ for all $z, z' \in \mathcal{Z}$. Then,
    \begin{equation}
        \|Q^r - W\|_{\infty} \leq 2(1-\sigma)^{r/r_0 - 1} \quad \forall r \in \mathbb{N}.
        \label{eq:rate-of-convergence-levin-and-peres}
    \end{equation}
\end{lemma}
We then have the below result for the transition kernel $Q_{n}^r$.
\begin{lemma}
    \label{lem:bounding-infinity-norm-between-matrices}
    Fix $\boldsymbol{\theta} \in \Theta^K$. For any $n \geq K$, let $\pi_n^\eta \coloneqq \pi_{\widehat{\boldsymbol{\theta}}(n)}^\eta$. Let
    \begin{align}
        \sigma(\varepsilon, \pi, \omega) &\coloneqq \left(\varepsilon^{r_{\boldsymbol{\theta}, \text{\rm unif}}}+ \left((1-\varepsilon)\, K \min_{(\mathbf{d}, \mathbf{i}, a) \in U_{\pi}} \pi(a|\mathbf{d}, \mathbf{i})\right)^{r_{\boldsymbol{\theta}, \text{\rm unif}}}\, \right) \cdot \sigma_{\boldsymbol{\theta},\text{\rm unif}} \cdot \left(\min_{(\mathbf{d}, \mathbf{i}, a)} \frac{\omega_{\boldsymbol{\theta}, \text{\rm unif}}(\mathbf{d}, \mathbf{i}, a)}{\omega(\mathbf{d}, \mathbf{i}, a)}\right), \label{eq:sigma-definition} \\
        \bar{\sigma}(\varepsilon, \pi, \omega) & \coloneqq 1-\sigma(\varepsilon, \pi, \omega), \label{eq:sigma-bar-definition}.
        %\mathcal{L}(\varepsilon, \pi, \omega) &\coloneqq \frac{2}{\bar{\sigma}(\varepsilon, \pi, \omega)[1-\bar{\sigma}(\varepsilon, \pi, \omega)^{1/r_{\boldsymbol{\theta}}, \text{\rm unif}}\, ]}. \label{eq:L-definition}
    \end{align}
    Then, we have
    \begin{equation}
        \| Q_n^r - W_n \|_{\infty} \leq C_n \, \rho_n^r \quad \forall r \in \mathbb{N},
        \label{eq:convergence-of-kernel-to-stationary-distribution}
    \end{equation}
    where $C_n = 2/\bar{\sigma}(\varepsilon_n, \pi_{n-1}^\eta, \omega_n)$ and $\rho_n = \bar{\sigma}(\varepsilon_n, \pi_{n-1}^\eta, \omega_n)^{1/r_{\boldsymbol{\theta}, \text{\rm unif}}}$.
\end{lemma}
\begin{proof}
    Recall that
    \begin{equation}
        Q_n(\mathbf{d}', \mathbf{i}', a'|\mathbf{d}, \mathbf{i}, a) = \varepsilon_n \, Q_{\boldsymbol{\theta}, \pi^{\text{\rm unif}}}(\mathbf{d}', \mathbf{i}', a'|\mathbf{d}, \mathbf{i}, a) + (1-\varepsilon_n) \, Q_{\boldsymbol{\theta}, \pi_{n-1}^\eta}(\mathbf{d}', \mathbf{i}', a'|\mathbf{d}, \mathbf{i}, a).
        \label{eq:proof-of-convergence-of-kernel-to-stat-dist-1}
    \end{equation}
    We note that for all $(\mathbf{d}, \mathbf{i}, a), (\mathbf{d}', \mathbf{i}', a') \in \mathbb{S}_R \times [K]$,
    \begin{align}
        Q_{\boldsymbol{\theta}, \pi_{n-1}^\eta}(\mathbf{d}', \mathbf{i}', a'|\mathbf{d}, \mathbf{i}, a)
        &= Q_{\boldsymbol{\theta}, R}(\mathbf{d}', \mathbf{i}'|\mathbf{d}, \mathbf{i}, a) \cdot \pi_{n-1}^\eta(a'|\mathbf{d}', \mathbf{i}') \nonumber\\
        &= Q_{\boldsymbol{\theta}, R}(\mathbf{d}', \mathbf{i}'|\mathbf{d}, \mathbf{i}, a) \cdot K\, \pi_{n-1}^\eta(a'|\mathbf{d}', \mathbf{i}') \cdot \frac{1}{K} \nonumber\\ 
        &\geq Q_{\boldsymbol{\theta}, R}(\mathbf{d}', \mathbf{i}'|\mathbf{d}, \mathbf{i}, a) \cdot K\, \min_{(\mathbf{d}, \mathbf{i}, a) \in U_{n-1}} \pi_{n-1}^\eta(a|\mathbf{d}, \mathbf{i}) \cdot \pi^{\text{\rm unif}}(a'|\mathbf{d}', \mathbf{i}') \nonumber\\
        &= K\, \min_{(\mathbf{d}, \mathbf{i}, a) \in U_{n-1}} \pi_{n-1}^\eta(a|\mathbf{d}, \mathbf{i}) \cdot Q_{\boldsymbol{\theta}, \pi^{\text{\rm unif}}}(\mathbf{d}', \mathbf{i}', a'|\mathbf{d}, \mathbf{i}, a),
        \label{eq:proof-of-convergence-of-kernel-to-stat-dist-2}
    \end{align}
    where in writing the inequality above, we make use of the relation
    \begin{equation}
        \pi_{n-1}^\eta(a'|\mathbf{d}', \mathbf{i}') \cdot \frac{1}{K} \geq \min_{(\mathbf{d}, \mathbf{i}, a) \in U_{n-1}} \pi_{n-1}^\eta(a|\mathbf{d}, \mathbf{i}) \cdot \pi^{\text{\rm unif}}(a'|\mathbf{d}', \mathbf{i}') \quad \forall (\mathbf{d}', \mathbf{i}', a').
    \end{equation}
    We then note that for all $(\mathbf{d}, \mathbf{i}, a), (\mathbf{d}', \mathbf{i}', a') \in \mathbb{V}$,
    \begin{align}
        & Q_n^{r_{\boldsymbol{
        \theta}, \text{\rm unif}}}(\mathbf{d}', \mathbf{i}', a'|\mathbf{d}, \mathbf{i}, a) \nonumber\\
        &\geq \varepsilon_n^{r_{\boldsymbol{\theta}}, \text{\rm unif}} \, Q_{\boldsymbol{\theta}, \pi^{\text{\rm unif}}}^{r_{\boldsymbol{\theta}, \text{\rm unif}}}(\mathbf{d}', \mathbf{i}', a'|\mathbf{d}, \mathbf{i}, a) + (1-\varepsilon_n)^{r_{\boldsymbol{\theta}, \text{\rm unif}}} \, Q_{\boldsymbol{\theta}, \pi_{n-1}^\eta}^{r_{\boldsymbol{\theta}, \text{\rm unif}}}(\mathbf{d}', \mathbf{i}', a'|\mathbf{d}, \mathbf{i}, a) \nonumber\\
        &\stackrel{(a)}{\geq} \left(\varepsilon_n^{r_{\boldsymbol{\theta}}, \text{\rm unif}} + \left((1-\varepsilon_n)\, K \, \min_{(\mathbf{d}, \mathbf{i}, a) \in U_{n-1}}\pi_{n-1}^\eta(a|\mathbf{d}, \mathbf{i})\right)^{r_{\boldsymbol{\theta}, \text{\rm unif}}}\, \right)\, Q_{\boldsymbol{\theta}, \pi^{\text{\rm unif}}}^{r_{\boldsymbol{\theta}, \text{\rm unif}}}(\mathbf{d}', \mathbf{i}', a'|\mathbf{d}, \mathbf{i}, a) \nonumber\\
        &\stackrel{(b)}{\geq} \left(\varepsilon_n^{r_{\boldsymbol{\theta}}, \text{\rm unif}} + \left((1-\varepsilon_n)\, K \, \min_{(\mathbf{d}, \mathbf{i}, a) \in U_{n-1}}\pi_{n-1}^\eta(a|\mathbf{d}, \mathbf{i})\right)^{r_{\boldsymbol{\theta}, \text{\rm unif}}}\, \right) \cdot \sigma_{\boldsymbol{\theta}, \text{\rm unif}} \cdot \omega_{\boldsymbol{\theta}, \text{\rm unif}}(\mathbf{d}', \mathbf{i}', a') \nonumber\\
        &= \left(\varepsilon_n^{r_{\boldsymbol{\theta}}, \text{\rm unif}} + \left((1-\varepsilon_n)\, K \, \min_{(\mathbf{d}, \mathbf{i}, a) \in U_{n-1}}\pi_{n-1}^\eta(a|\mathbf{d}, \mathbf{i})\right)^{r_{\boldsymbol{\theta}, \text{\rm unif}}}\, \right) \cdot \sigma_{\boldsymbol{\theta}, \text{\rm unif}} \cdot \frac{\omega_{\boldsymbol{\theta}, \text{\rm unif}}(\mathbf{d}', \mathbf{i}', a')}{\omega_{n}^\star(\mathbf{d}', \mathbf{i}', a')} \cdot \omega_{n}^\star(\mathbf{d}', \mathbf{i}', a') \nonumber\\ 
        &\stackrel{(c)}{\geq} \underbrace{\left(\varepsilon_n^{r_{\boldsymbol{\theta}}, \text{\rm unif}} + \left((1-\varepsilon_n)\, K \, \min_{(\mathbf{d}, \mathbf{i}, a) \in U_{n-1}}\pi_{n-1}^\eta(a|\mathbf{d}, \mathbf{i})\right)^{r_{\boldsymbol{\theta}, \text{\rm unif}}}\, \right) \cdot \sigma_{\boldsymbol{\theta}, \text{\rm unif}} \cdot \left(\min_{(\mathbf{d}, \mathbf{i}, a) \in \mathbb{V}}\frac{\omega_{\boldsymbol{\theta}, \text{\rm unif}}(\mathbf{d}, \mathbf{i}, a)}{\omega_{n}^\star(\mathbf{d}, \mathbf{i}, a)} \right)}_{=\sigma(\varepsilon_n, \pi_{n-1}^\eta, \omega_n^\star)} \cdot \, \omega_{n}^\star(\mathbf{d}', \mathbf{i}', a') \nonumber\\ 
        &= \sigma(\varepsilon_n, \pi_{n-1}^\eta, \omega_n^\star) \cdot  \omega_{n}^\star(\mathbf{d}', \mathbf{i}', a'),
        \label{eq:proof-of-convergence-of-kernel-to-stat-dist-3}
    \end{align}
    where $(a)$ above follows from an application of \eqref{eq:proof-of-convergence-of-kernel-to-stat-dist-2} for a total of $r_{\boldsymbol{\theta}, \text{\rm unif}}$ times in succession, and $(b)$ follows from~\eqref{eq:sigma-unif-definition}. Noting that \eqref{eq:proof-of-convergence-of-kernel-to-stat-dist-3} holds for all $(\mathbf{d}, \mathbf{i}, a), (\mathbf{d}', \mathbf{i}', a') \in \mathbb{V}$, a simple application of Lemma~\ref{lem:rate-of-convergence-levin-and-peres} with $\sigma=\sigma(\varepsilon_n, \pi_{n-1}^\eta, \omega_n^\star)$ and $r_0 = r_{\boldsymbol{\theta}, \text{\rm unif}}$ yields the desired result.
\end{proof}
%-------------------------------------------------------------------------------

\subsection{Concentration of Empirical Transition Kernel}
In this section, we record a result on the concentration of the empirical transition kernel $\widehat{Q}_n$ defined in \eqref{eq:categorical-distribution} under a indefinitely running version of the policy $\pi^{\textsc{Rstl-Dtrack}}$ (i.e., a policy that does not check for the stopping criterion and runs indefinitely by selecting arms according to \eqref{eq:arms-selection-rule}). We begin with some notations. For any $n \geq K$, let $\widehat{\mathcal{M}}_n$ denote the MDP whose state space is $\mathbb{S}_R$, action space is $[K]$, and the transition kernel is specified by $\widehat{Q}_n$.  
For any $\boldsymbol{\theta} \in \Theta^K$, let
\begin{equation}
    \| \widehat{\mathcal{M}}_n - \mathcal{M}_{\boldsymbol{\theta}, R} \|_{\infty} \coloneqq \max_{(\mathbf{d}, \mathbf{i}, a)} \| \widehat{Q}_n(\cdot \mid\mathbf{d}, \mathbf{i}, a) - Q_{\boldsymbol{\theta}, R}(\cdot \mid\mathbf{d}, \mathbf{i}, a) \|_1,
    \label{eq:norm-between-MDPs}
\end{equation}
where $\|\cdot \|_1$ denotes the vector $1$-norm operator.
We then have the following result.
\begin{lemma}
    \label{lem:concentration-of-parameter-estimates}
    Fix $\boldsymbol{\theta} \in \Theta^K$ and $\xi > 0$. For $N \geq K$, let\footnote{We assume, without loss of generality, that $N^{1/4}$ is an integer.}
    \begin{equation}
        C_N^1(\xi) \coloneqq \bigcap_{n=N^5}^{N^6} \bigg\lbrace \|\widehat{\mathcal{M}}_n - \mathcal{M}_{\boldsymbol{\theta}, R} \|_{\infty} \leq \xi \bigg\rbrace.
        \label{eq:C-N-1-of-xi-event}
    \end{equation}
    Then, there exist constants $B$ and $C$ that depend only on $\xi$ and $\boldsymbol{\theta}$ such that under the indefinite version of $\pi^{\textsc{Rstl-Dtrack}}$,
    \begin{equation}
        \P_{\boldsymbol{\theta}}\bigg(\overline{C_N^1(\xi)}\bigg) \leq \frac{1}{N^2} + B\, N^6\, \exp \left( - \frac{C\, N^{1/4}}{\sqrt{K \, S_R}} \right) \quad \forall N \geq K.
        \label{eq:concentration-inequality-parameter-estimates}
    \end{equation}
\end{lemma}
\begin{proof}
    Consider the event
    \begin{equation}
        \mathcal{E}_N \coloneqq \bigg\lbrace \forall(\mathbf{d}, \mathbf{i}, a) \in \mathbb{V}, \quad \forall n \geq K, \quad N(n, \mathbf{d}, \mathbf{i}, a) \geq \left\lceil \frac{n}{\lambda_{\boldsymbol{\theta}}(N)} \right\rceil^{1/4} - 1 \bigg\rbrace,
        \label{eq:event-E-N-definition}
    \end{equation}
    where $\lambda_{\boldsymbol{\theta}}(N) \coloneqq \frac{(1+S_R)^2}{\sigma_{\boldsymbol{\theta}}^2} \log^2(1+KS_R\,N^2)$, and $\mathbb{V}$ is as defined in Lemma~\ref{lem:sufficient-exploration-of-state-actions}. In the definition of $\lambda_{\boldsymbol{\theta}}(N)$, the constant $\sigma_{\boldsymbol{\theta}}$ is as defined in Lemma~\ref{lem:norm-of-product-of-Aj}. From Lemma~\ref{lem:sufficient-exploration-of-state-actions}, we know that $\P_{\boldsymbol{\theta}}(\mathcal{E}_N) \geq 1-1/N^2$. We thus have
    \begin{align}
        \P_{\boldsymbol{\theta}}(\overline{C_N^1(\xi)}) 
        &\leq \P_{\boldsymbol{\theta}}(\overline{\mathcal{E}_N}) + \P_{\boldsymbol{\theta}}(\overline{C_N^1(\xi)} \cap \mathcal{E}_N) \nonumber\\
        &\leq \frac{1}{N^2} + \P_{\boldsymbol{\theta}}(\overline{C_N^1(\xi)} \cap \mathcal{E}_N).
        \label{eq:proof-of-concentration-of-parameter-estimates-1}
    \end{align}
    Using the union bound, we have
    \begin{align}
        & \P_{\boldsymbol{\theta}}(\overline{C_N^1(\xi)} \cap \mathcal{E}_N) \nonumber\\
        &\leq \sum_{n=N^5}^{N^6} \P_{\boldsymbol{\theta}}\bigg(\exists (\mathbf{d}, \mathbf{i}, a) \in \mathbb{V}: \|\widehat{Q}_n(\cdot \mid\mathbf{d}, \mathbf{i}, a) - Q_{\boldsymbol{\theta}, R}(\cdot \mid\mathbf{d}, \mathbf{i}, a) \|_1 > \xi, \ \mathcal{E}_N\bigg) \nonumber\\
        &\leq \sum_{n=N^5}^{N^6} \ \sum_{(\mathbf{d}, \mathbf{i}, a) \in \mathbb{V}} \P_{\boldsymbol{\theta}}\bigg(\exists (\mathbf{d}', \mathbf{i}') \in \mathbb{S}_R: |\widehat{Q}_n(\mathbf{d}', \mathbf{i}'|\mathbf{d}, \mathbf{i}, a) - Q_{\boldsymbol{\theta}, R}(\mathbf{d}', \mathbf{i}'|\mathbf{d}, \mathbf{i}, a)| > \frac{\xi}{S_R}, \ \mathcal{E}_N\bigg) \nonumber\\
        &\leq \sum_{n=N^5}^{N^6} \ \sum_{\substack{(\mathbf{d}, \mathbf{i}, a) \in \mathbb{V}\\(\mathbf{d}', \mathbf{i}') \in \mathbb{S}_R}} \P_{\boldsymbol{\theta}}\bigg(|\widehat{Q}_n(\mathbf{d}', \mathbf{i}'|\mathbf{d}, \mathbf{i}, a) - Q_{\boldsymbol{\theta}, R}(\mathbf{d}', \mathbf{i}'|\mathbf{d}, \mathbf{i}, a)| > \frac{\xi}{S_R}, \ \mathcal{E}_N\bigg) \nonumber\\
        &\leq \sum_{n=N^5}^{N^6} \ \sum_{\substack{(\mathbf{d}, \mathbf{i}, a) \in \mathbb{V}\\(\mathbf{d}', \mathbf{i}') \in \mathbb{S}_R}} \P_{\boldsymbol{\theta}}\bigg(\widehat{Q}_n(\mathbf{d}', \mathbf{i}'|\mathbf{d}, \mathbf{i}, a) - Q_{\boldsymbol{\theta}, R}(\mathbf{d}', \mathbf{i}'|\mathbf{d}, \mathbf{i}, a) > \frac{\xi}{S_R}, \ \mathcal{E}_N\bigg) \nonumber\\
        &\hspace{1cm} + \sum_{n=N^5}^{N^6} \ \sum_{\substack{(\mathbf{d}, \mathbf{i}, a) \in \mathbb{V}\\(\mathbf{d}', \mathbf{i}') \in \mathbb{S}_R}} \P_{\boldsymbol{\theta}}\bigg(\widehat{Q}_n(\mathbf{d}', \mathbf{i}'|\mathbf{d}, \mathbf{i}, a) - Q_{\boldsymbol{\theta}, R}(\mathbf{d}', \mathbf{i}'|\mathbf{d}, \mathbf{i}, a) < -\frac{\xi}{S_R}, \ \mathcal{E}_N\bigg).
        \label{eq:proof-of-concentration-of-parameter-estimates-2}
    \end{align}
    For $x, y \in [0,1]$, let $d(x,y) \coloneqq x\log(x/y) + (1-x) \log((1-x)/(1-y))$. For all $(\mathbf{d}, \mathbf{i}, a) \in \mathbb{V}$, $(\mathbf{d}', \mathbf{i}') \in \mathbb{S}_R$, and $n \geq N^5$, we then have
    \begin{align}
        & \P_{\boldsymbol{\theta}}\bigg(\widehat{Q}_n(\mathbf{d}', \mathbf{i}'|\mathbf{d}, \mathbf{i}, a) - Q_{\boldsymbol{\theta}, R}(\mathbf{d}', \mathbf{i}'|\mathbf{d}, \mathbf{i}, a) > \frac{\xi}{S_R}, \ \mathcal{E}_N\bigg) \nonumber\\
        &\leq \P_{\boldsymbol{\theta}}\bigg(\widehat{Q}_n(\mathbf{d}', \mathbf{i}'|\mathbf{d}, \mathbf{i}, a) - Q_{\boldsymbol{\theta}, R}(\mathbf{d}', \mathbf{i}'|\mathbf{d}, \mathbf{i}, a) > \frac{\xi}{S_R}, \ N(n, \mathbf{d}, \mathbf{i}, a) \geq \left\lceil \frac{n}{\lambda_{\boldsymbol{\theta}}(N)} \right\rceil^{1/4} - 1 \bigg) \nonumber\\
        &= \sum_{u=\left\lceil \frac{n}{\lambda_{\boldsymbol{\theta}}(N)} \right\rceil^{1/4} - 1}^{\infty} \P_{\boldsymbol{\theta}}\bigg(\widehat{Q}_n(\mathbf{d}', \mathbf{i}'|\mathbf{d}, \mathbf{i}, a) - Q_{\boldsymbol{\theta}, R}(\mathbf{d}', \mathbf{i}'|\mathbf{d}, \mathbf{i}, a) > \frac{\xi}{S_R}, \ N(n, \mathbf{d}, \mathbf{i}, a) = u \bigg) \nonumber\\
        &\stackrel{(a)}{\leq} \sum_{u=\left\lceil \frac{n}{\lambda_{\boldsymbol{\theta}}(N)} \right\rceil^{1/4} - 1}^{\infty} \exp\left(-u \cdot d\left(Q_{\boldsymbol{\theta}, R}(\mathbf{d}', \mathbf{i}'|\mathbf{d}, \mathbf{i}, a) + \frac{\xi}{S_R}, Q_{\boldsymbol{\theta}, R}(\mathbf{d}', \mathbf{i}'|\mathbf{d}, \mathbf{i}, a)\right)\right) \nonumber\\
        &\leq \dfrac{\exp\left(-\left(\left\lceil \frac{n}{\lambda_{\boldsymbol{\theta}}(N)} \right\rceil^{1/4} - 1\right) \cdot d\left(Q_{\boldsymbol{\theta}, R}(\mathbf{d}', \mathbf{i}'|\mathbf{d}, \mathbf{i}, a) + \frac{\xi}{S_R}, Q_{\boldsymbol{\theta}, R}(\mathbf{d}', \mathbf{i}'|\mathbf{d}, \mathbf{i}, a)\right)\right)}{1-\exp\left(-d\left(Q_{\boldsymbol{\theta}, R}(\mathbf{d}', \mathbf{i}'|\mathbf{d}, \mathbf{i}, a) + \frac{\xi}{S_R}, Q_{\boldsymbol{\theta}, R}(\mathbf{d}', \mathbf{i}'|\mathbf{d}, \mathbf{i}, a)\right)\right)} \nonumber\\
        &\leq \dfrac{\exp\left(-\left( \frac{N^{5/4}}{\lambda_{\boldsymbol{\theta}}(N)^{1/4}} - 1\right) \cdot d\left(Q_{\boldsymbol{\theta}, R}(\mathbf{d}', \mathbf{i}'|\mathbf{d}, \mathbf{i}, a) + \frac{\xi}{S_R}, Q_{\boldsymbol{\theta}, R}(\mathbf{d}', \mathbf{i}'|\mathbf{d}, \mathbf{i}, a)\right)\right)}{1-\exp\left(-d\left(Q_{\boldsymbol{\theta}, R}(\mathbf{d}', \mathbf{i}',|\mathbf{d}, \mathbf{i}, a) + \frac{\xi}{S_R}, Q_{\boldsymbol{\theta}, R}(\mathbf{d}', \mathbf{i}'|\mathbf{d}, \mathbf{i}, a)\right)\right)},
        \label{eq:proof-of-concentration-of-parameter-estimates-3}
    \end{align}
    where $(a)$ above follows from the Chernoff--Hoeffding bound. Along similar lines as above, we have
    \begin{align}
        & \P_{\boldsymbol{\theta}}\bigg(\widehat{Q}_n(\mathbf{d}', \mathbf{i}'|\mathbf{d}, \mathbf{i}, a) - Q_{\boldsymbol{\theta}, R}(\mathbf{d}', \mathbf{i}'|\mathbf{d}, \mathbf{i}, a) < -\frac{\xi}{S_R}, \ \mathcal{E}_N\bigg) \nonumber\\
        &\leq \dfrac{\exp\left(-\left( \frac{N^{5/4}}{\lambda_{\boldsymbol{\theta}}(N)^{1/4}} - 1\right) \cdot d\left(Q_{\boldsymbol{\theta}, R}(\mathbf{d}', \mathbf{i}'|\mathbf{d}, \mathbf{i}, a) - \frac{\xi}{S_R}, Q_{\boldsymbol{\theta}, R}(\mathbf{d}', \mathbf{i}'|\mathbf{d}, \mathbf{i}, a)\right)\right)}{1-\exp\left(-d\left(Q_{\boldsymbol{\theta}, R}(\mathbf{d}', \mathbf{i}'|\mathbf{d}, \mathbf{i}, a) - \frac{\xi}{S_R}, Q_{\boldsymbol{\theta}, R}(\mathbf{d}', \mathbf{i}',|\mathbf{d}, \mathbf{i}, a)\right)\right)}.
        \label{eq:proof-of-concentration-of-parameter-estimates-4}
    \end{align}
    Plugging \eqref{eq:proof-of-concentration-of-parameter-estimates-3} and \eqref{eq:proof-of-concentration-of-parameter-estimates-4} into \eqref{eq:proof-of-concentration-of-parameter-estimates-2}, setting
\begin{align}
    C &\coloneqq \sqrt{\frac{\sigma_{\boldsymbol{\theta}}}{1+S_R}} \min_{\substack{(\mathbf{d}, \mathbf{i}, a) \in \mathbb{V}\\(\mathbf{d}', \mathbf{i}') \in \mathbb{S}_R}} \bigg\lbrace d\left(Q_{\boldsymbol{\theta}, R}(\mathbf{d}', \mathbf{i}'|\mathbf{d}, \mathbf{i}, a) - \frac{\xi}{S_R}, ~~ Q_{\boldsymbol{\theta}, R}(\mathbf{d}', \mathbf{i}'|\mathbf{d}, \mathbf{i}, a)\right), \nonumber\\
    &\hspace{5cm} d\left(Q_{\boldsymbol{\theta}, R}(\mathbf{d}', \mathbf{i}'|\mathbf{d}, \mathbf{i}, a) + \frac{\xi}{S_R}, ~~ Q_{\boldsymbol{\theta}, R}(\mathbf{d}', \mathbf{i}'|\mathbf{d}, \mathbf{i}, a) \right) \bigg\rbrace, \label{eq:C-definition} \\
    B &\coloneqq \sum_{\substack{(\mathbf{d}, \mathbf{i}, a) \in \mathbb{V}\\(\mathbf{d}', \mathbf{i}') \in \mathbb{S}_R}}  \bigg[\dfrac{\exp\left(d\left(Q_{\boldsymbol{\theta}, R}(\mathbf{d}', \mathbf{i}'|\mathbf{d}, \mathbf{i}, a) + \frac{\xi}{S_R}, ~~ Q_{\boldsymbol{\theta}, R}(\mathbf{d}', \mathbf{i}'|\mathbf{d}, \mathbf{i}, a)\right)\right)}{1-\exp\left(-d\left(Q_{\boldsymbol{\theta}, R}(\mathbf{d}', \mathbf{i}'|\mathbf{d}, \mathbf{i}, a) + \frac{\xi}{S_R}, ~~ Q_{\boldsymbol{\theta}, R}(\mathbf{d}', \mathbf{i}'|\mathbf{d}, \mathbf{i}, a)\right)\right)} \nonumber\\
    &\hspace{3cm} + \dfrac{\exp\left(d\left(Q_{\boldsymbol{\theta}, R}(\mathbf{d}', \mathbf{i}'|\mathbf{d}, \mathbf{i}, a) - \frac{\xi}{S_R}, ~~ Q_{\boldsymbol{\theta}, R}(\mathbf{d}', \mathbf{i}'|\mathbf{d}, \mathbf{i}, a)\right)\right)}{1-\exp\left(-d\left(Q_{\boldsymbol{\theta}, R}(\mathbf{d}', \mathbf{i}'|\mathbf{d}, \mathbf{i}, a) - \frac{\xi}{S_R}, ~~ Q_{\boldsymbol{\theta}, R}(\mathbf{d}', \mathbf{i}'|\mathbf{d}, \mathbf{i}, a)\right)\right)}\bigg], \label{eq:B-definition}
\end{align}
and 
noting that 
\begin{align*}
    \sum_{n=N^5}^{N^6} \exp\left(-\frac{C\, N^{5/4}}{\sqrt{\log(1+KS_R\,N^2)}} \right) 
    &\leq N^6 \exp\left(-\, \frac{C\, N^{5/4}}{\sqrt{\log(1+KS_R\,N^2)}} \right) \nonumber\\
    &\leq N^6 \exp\left(-\, \frac{C\, N^{5/4}}{\sqrt{K\, S_R\, N^2}} \right) \nonumber\\
    &= N^6 \exp\left(-\, \frac{C\, N^{1/4}}{\sqrt{K\, S_R}} \right),
\end{align*}
we arrive at \eqref{eq:concentration-inequality-parameter-estimates}. 
\end{proof}

\end{document}